\documentclass{article}

\usepackage[preprint]{neurips_2021}

\usepackage[utf8]{inputenc} 
\usepackage[T1]{fontenc}    
\usepackage{hyperref}       
\usepackage{url}            
\usepackage{booktabs}       
\usepackage{amsfonts}       
\usepackage{nicefrac}       
\usepackage{microtype}      
\usepackage{xcolor}         
\usepackage{amsmath}
\usepackage{comment}
\usepackage{csquotes}
\usepackage{subcaption}
\usepackage{graphicx}
\usepackage{wrapfig}
\usepackage{algorithm} 
\usepackage{algorithmic} 

\DeclareMathOperator*{\argmin}{arg\,min}
\usepackage{thmtools}
\newcommand{\R}{\mathbb{R}}
\newcommand{\gN}{\mathcal{N}}

\newtheorem{proof}{Proof}

\usepackage{thmtools}
\usepackage{thm-restate}
\usepackage{hyperref}
\usepackage{cleveref}

\title{Imbalanced Adversarial Training with Reweighting}

\author{%
  Wentao Wang\thanks{Equal Contribution}\\
  Michigan State University\\
  \texttt{wangw116@msu.edu} \\
  \And 
  Han Xu$^*$\\
Michigan State University \\
  \texttt{xuhan1@msu.edu} \\
  \And
Xiaorui Liu\\
  Michigan State University\\
  \texttt{xiaorui@msu.edu} \\
  \AND 
Yaxin Li\\
  Michigan State University\\
  \texttt{liyaxin1@msu.edu} \\
  \And 
Bhavani Thuraisingham\\
    The University of Texas at Dallas\\
    \texttt{bxt043000@utdallas.edu} \\
    \And
Jiliang Tang\\
Michigan State University\\
  \texttt{tangjili@msu.edu} \\
}

\hypersetup{draft}

\begin{document}

\maketitle

\begin{abstract}
Adversarial training has been empirically proven to be one of the most effective and reliable defense methods against adversarial attacks. However, almost all existing studies about adversarial training are focused on balanced datasets, where each class has an equal amount of training examples. Research on adversarial training with imbalanced training datasets is rather limited. As the initial effort to investigate this problem, we reveal the facts that adversarially trained models present two distinguished behaviors from naturally trained models in imbalanced datasets: (1) Compared to natural training, adversarially trained models can suffer much worse performance on under-represented classes, when the training dataset is extremely imbalanced.
(2) Traditional \textit{reweighting} strategies may lose efficacy to deal with the imbalance issue for adversarial training. For example, upweighting the under-represented classes will drastically hurt the model's performance on well-represented classes, and as a result, finding an optimal reweighting value can be tremendously challenging. In this paper, to further understand our observations, we theoretically show that the poor data separability is one key reason causing this strong tension between under-represented and well-represented classes. Motivated by this finding, we propose \textit{Separable Reweighted Adversarial Training} (SRAT) to facilitate adversarial training under imbalanced scenarios, by learning more separable features for different classes. Extensive experiments on various datasets verify the effectiveness of the proposed framework.
\end{abstract}

\section{Introduction}

The existence of adversarial samples~\cite{szegedy2013intriguing,goodfellow2014explaining} has risen huge concerns on applying deep neural network (DNN) models into security-critical applications, such as autonomous driving~\cite{chen2015deepdriving} and video surveillance systems~\cite{kurakin2016adversarial}. As countermeasures against adversarial attacks, adversarial training~\cite{madry2017towards, zhang2019theoretically} has been empirically proven to be one of the most effective and reliable defense methods. In general, adversarial training can be formulated to minimize the model's average error on adversarially perturbed input examples~\cite{madry2017towards, zhang2019theoretically, rice2020overfitting}. Although promising to improve the model’s robustness, most existing adversarial training methods~\cite{zhang2019theoretically,wang2019improving} assume that the number of training examples from each class is equally distributed. However, datasets collected from real-world applications typically have imbalanced distribution~\cite{everingham2010pascal,lin2014microsoft,van2017devil}. Hence, it is natural to ask: \textit{what is the behavior of adversarial training under imbalanced scenarios? Can we directly apply existing imbalanced learning strategies in natural training to tackle the imbalance issue for adversarial training?} Recent studies find that adversarial training usually presents distinct properties from natural training~\cite{schmidt2018adversarially, xu2020robust}. For example, compared to natural training, adversarially trained models suffer more from the overfitting issue~\cite{schmidt2018adversarially}.
Moreover, it is evident from a recent study~\cite{xu2020robust} that the adversarially trained models tend to present strong class-wise performance disparities, even if the training examples are uniformly distributed over different classes. Imagine that if the training data distribution is highly imbalanced, these properties of adversarial training can be greatly exaggerated and make it extremely difficult to be applied in practice. Therefore, it is important but challenging to answers aforementioned questions.

As the initial effort to study the imbalanced problem in adversarial training, in this work, we first investigate the performance of existing adversarial training under imbalanced settings. As a preliminary study shown in Section~\ref{subsec:pre1}, we apply both natural training and PGD adversarial training~\cite{madry2017towards} on multiple imbalanced image datasets constructed from the CIFAR10 dataset~\cite{krizhevsky2009learning} using the ResNet18 architecture~\cite{he2016deep} and evaluate trained models' performance on class-balanced test datasets. From the preliminary results, we observe that, compared to naturally trained models, adversarially trained models always present very low standard accuracy and robust accuracy\footnote{In this work, we denote \emph{standard accuracy} as model’s accuracy on the input samples without perturbations and \emph{robust accuracy} as model’s accuracy on the input samples which are adversarially perturbed. Without clear clarification, we consider the perturbation is constrained by $l_\infty$-norm 8/255.} on under-represented classes. For example, a naturally trained model can achieve around 40\% and 60\% standard accuracy on under-represented classes ``frog'' and ``truck'' separately, while an adversarially trained model gets both 0\% standard \& robust accuracy on these two classes. This observation suggests that adversarial training is more sensitive to imbalanced data distribution than natural training. Thus, when applying adversarial training in practice, imbalance learning strategies should always be considered for help.

As a result, we explore the potential solutions which can handle the imbalance issues for adversarial training. In this work, we focus on studying the behavior of the \textit{reweighting} strategy~\cite{he2013imbalanced} and leave other strategies such as resampling~\cite{estabrooks2004multiple} for one future work.
In Section~\ref{subsec:pre2}, we apply the reweighting strategy to existing adversarial training with varied weights assigning to one under-represented class and evaluate trained models' performance. From the results, we observe that, in adversarial training, increasing weights for an under-represented class can substantially improve the standard \& robust accuracy on this class, but drastically hurt the model's performance on the well-represented class. For example, the robust accuracy of the adversarially trained model on the under-represented class ``horse'' can be greatly improved when setting a relatively large weight, like 200, to its examples, but the model's robust accuracy on the well-represented class ``cat'' is dropped to even lower than the class ``horse'' and, hence, the overall robust performance of the model is also decreased. These facts indicate that the performance of adversarially trained models is very sensitive to the reweighting manipulations and it could be very hard to figure out an eligible reweighting strategy which is optimal for all classes.

It is also worth noting that this phenomenon is absent in natural training under the same settings. In natural training, from the results in Section~\ref{subsec:pre2}, we find that upweighting the under-represented class increases model's standard accuracy on this class but only slightly hurts the accuracy on other classes, even when adopting a large weight for under-represent class. 
To further investigate the possible reasons leading to different behaviors of the reweighing strategy in natural and adversarial training, we visualize their learned features via t-SNE~\cite{van2008visualizing}. As shown in Figure~\ref{fig:pre_binary_tsne}, we observe that features learned by the adversarially trained model of different classes tend to mix together while they are well separated for the naturally trained model. This observation motivates us to theoretically show that when the given data distribution has poor data separability, upweighting under-represented classes will hurt the model's performance on well-represented classes.  Motivated by this theoretical understanding, we propose a novel algorithm \emph{Separable Reweighted Adversarial Training (SRAT)} to facilitate the reweighting strategy in imbalanced adversarial training by enhancing the separability of learned features. Through extensive experiments, we validate the effectiveness of SRAT.

\section{Preliminary Study}\label{sec:preliminary}

\subsection{The Behavior of Adversarial Training}\label{subsec:pre1}

In this subsection, we conduct preliminary studies to examine the performance of PGD adversarial training~\cite{madry2017towards}, under an imbalanced training dataset which is resampled from CIFAR10 dataset~\cite{krizhevsky2009learning}. Following previous imbalanced learning works~\cite{cui2019class,cao2019learning}, we consider to construct an imbalanced training dataset where each of the first 5 classes (well-represented classes) has 5,000 training examples, and each of the last 5 classes (under-represented classes) has 50 training examples. 
Figure~\ref{fig:pre_step_100} shows the performance of naturally and adversarially trained models under a ResNet18~\cite{he2016deep} architecture. 
From the figure, we can observe that, comparing with natural training, PGD adversarial training will result in a larger performance gap between well-represented classes and under-represented classes. For example, in natural training, the ratio between the average standard accuracy of well-represented classes (brown) and under-represented classes (violet) is about 2:1, while in adversarial training, this ratio expands to 16:1. 
Moreover, for adversarial training, although it can achieve good standard \& robust accuracy on well-represented classes, it has extremely poor performance on under-represented classes. 
There are 3 out of the 5 under-represented classes with 0\% standard \& robust accuracy. 
As a conclusion, the performance of adversarial training is easier to be affected by imbalanced distribution than natural training and suffers more on under-represented classes.
In Appendix~\ref{app_sec:pre_com1}, we provide more implementation details of this experiment, as well as additional results of the same experiment under other imbalanced settings. The results in Appendix~\ref{app_sec:pre_com1} further support our findings.

\begin{figure}[t]
\centering
\begin{subfigure}[b]{0.3\textwidth}
\centering
\includegraphics[width=1.72in]{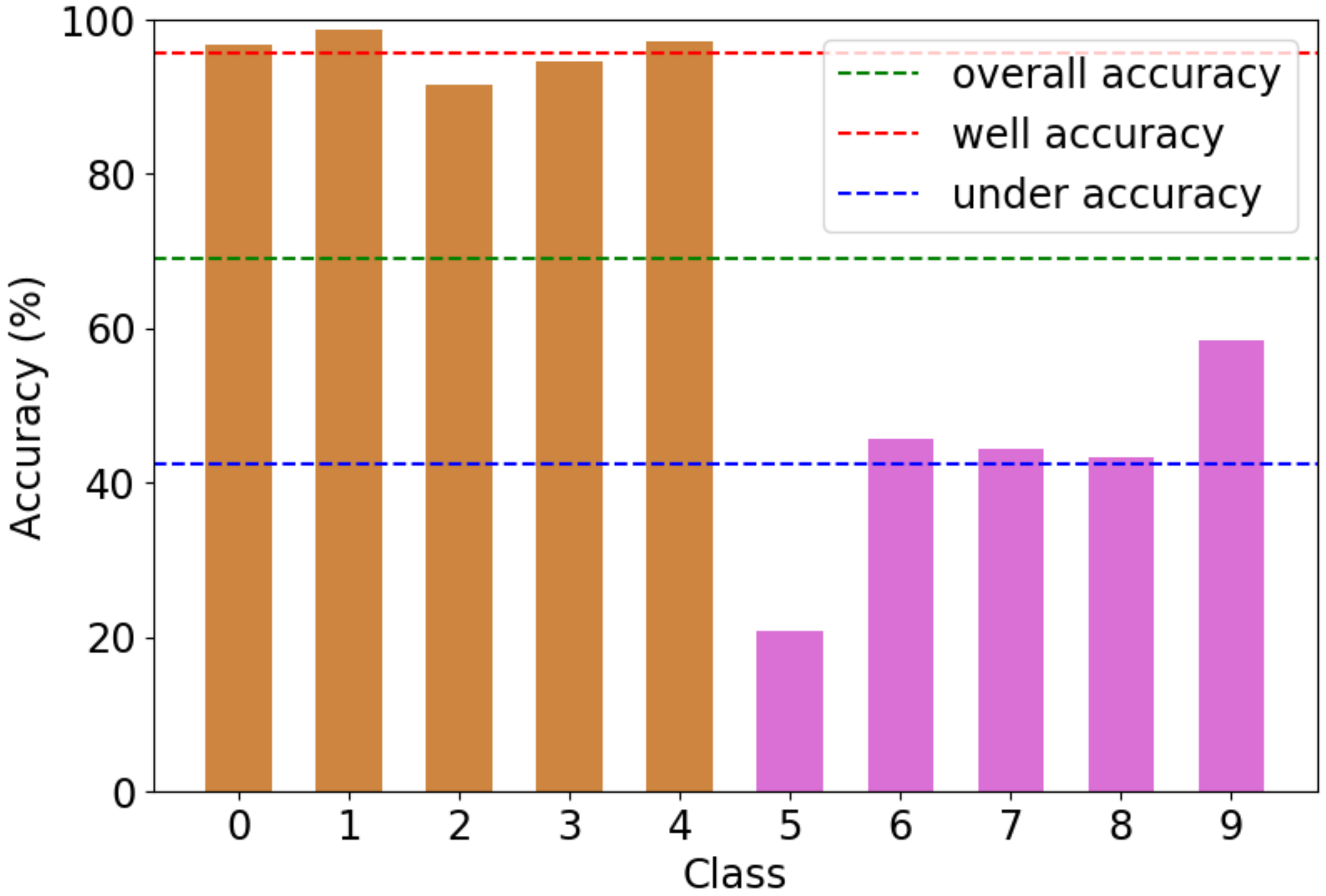}
\caption{Natural Training Standard Acc}
\label{fig:pre_step_100_nature_standard}
\end{subfigure}
\hspace{0.12in}
\begin{subfigure}[b]{0.3\textwidth}
\centering
\includegraphics[width=1.72in]{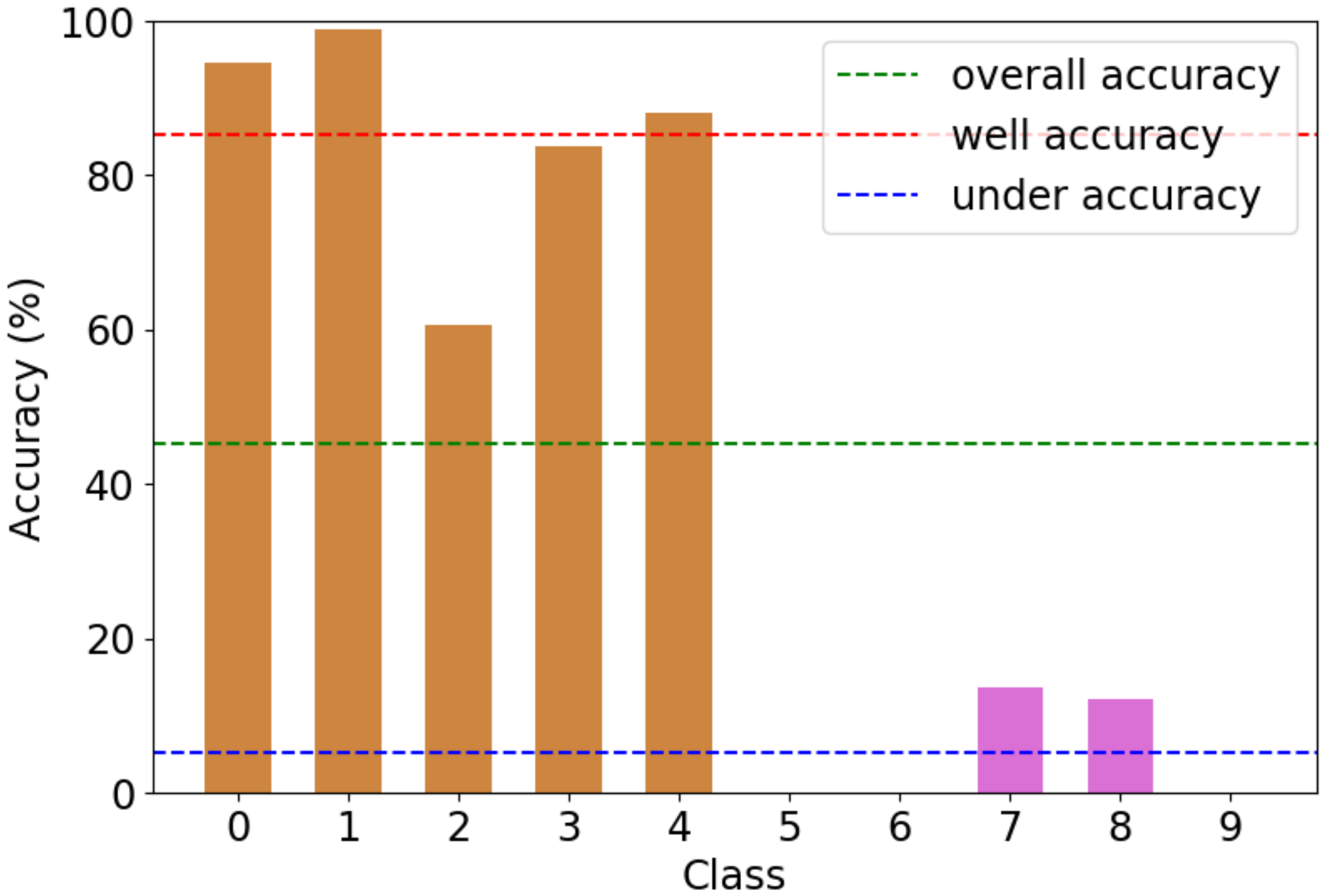}
\caption{Adv. Training Standard Acc.}
\label{fig:pre_step_100_adv_standard}
\end{subfigure}
\hspace{0.12in}
\begin{subfigure}[b]{0.3\textwidth}
\centering
\includegraphics[width=1.72in]{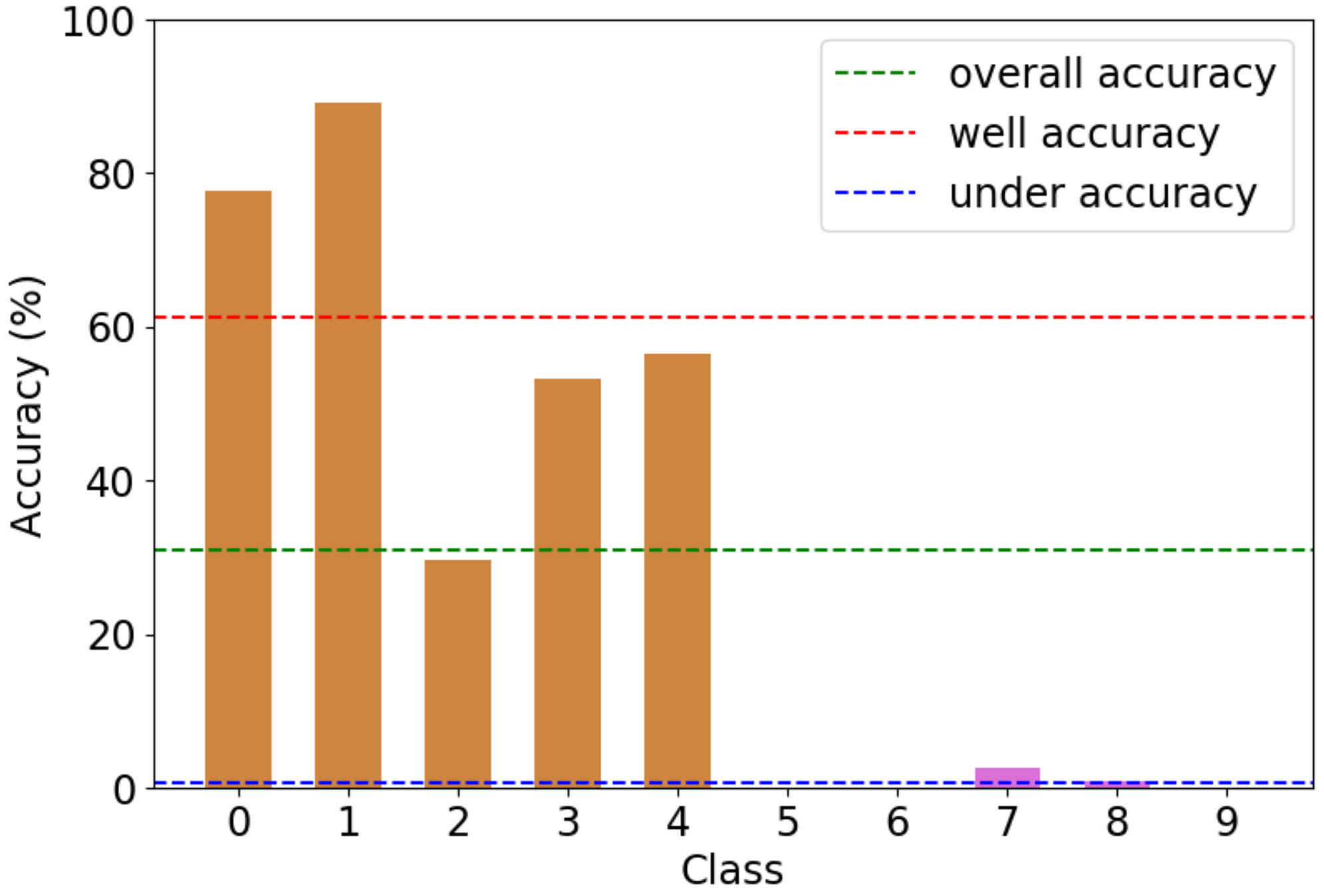}
\caption{Adv. Training Robust Acc.}
\label{fig:pre_step_100_adv_robust}
\end{subfigure}
\caption{Class-wise performance of natural \& adversarial training under an imbalanced CIFAR10.}
\label{fig:pre_step_100}
\end{figure}

\subsection{The Reweighting Strategy in Natural Training v.s. in Adversarial Training}\label{subsec:pre2}

The preliminary study in Section~\ref{subsec:pre1} demonstrates that it is highly demanding to adjust the original adversarial training methods to accommodate class-imbalanced data. Therefore, in this subsection, we investigate the effectiveness of existing imbalanced learning strategies in natural training when adopted in adversarial training. In this paper, we focus on the reweighting strategy~\cite{he2013imbalanced} as the initial effort to study this problem and leave other methods such as resampling~\cite{chawla2002smote} for future investigation.
In this subsection, we conduct experiments under a binary classification problem, where the training dataset contains two classes that are randomly selected from CIFAR10 dataset, with each class having 5,000 and 50 training examples respectively. Under this training dataset, we arrange multiple trails of (reweighted) natural training and (reweighted) adversarial training, with the weight ratio between the under-represented class and well-represented class ranging from 1:1 to 200:1.

\begin{figure}[h]
\begin{subfigure}[b]{0.31\textwidth}
\centering
\includegraphics[width=1.72in]{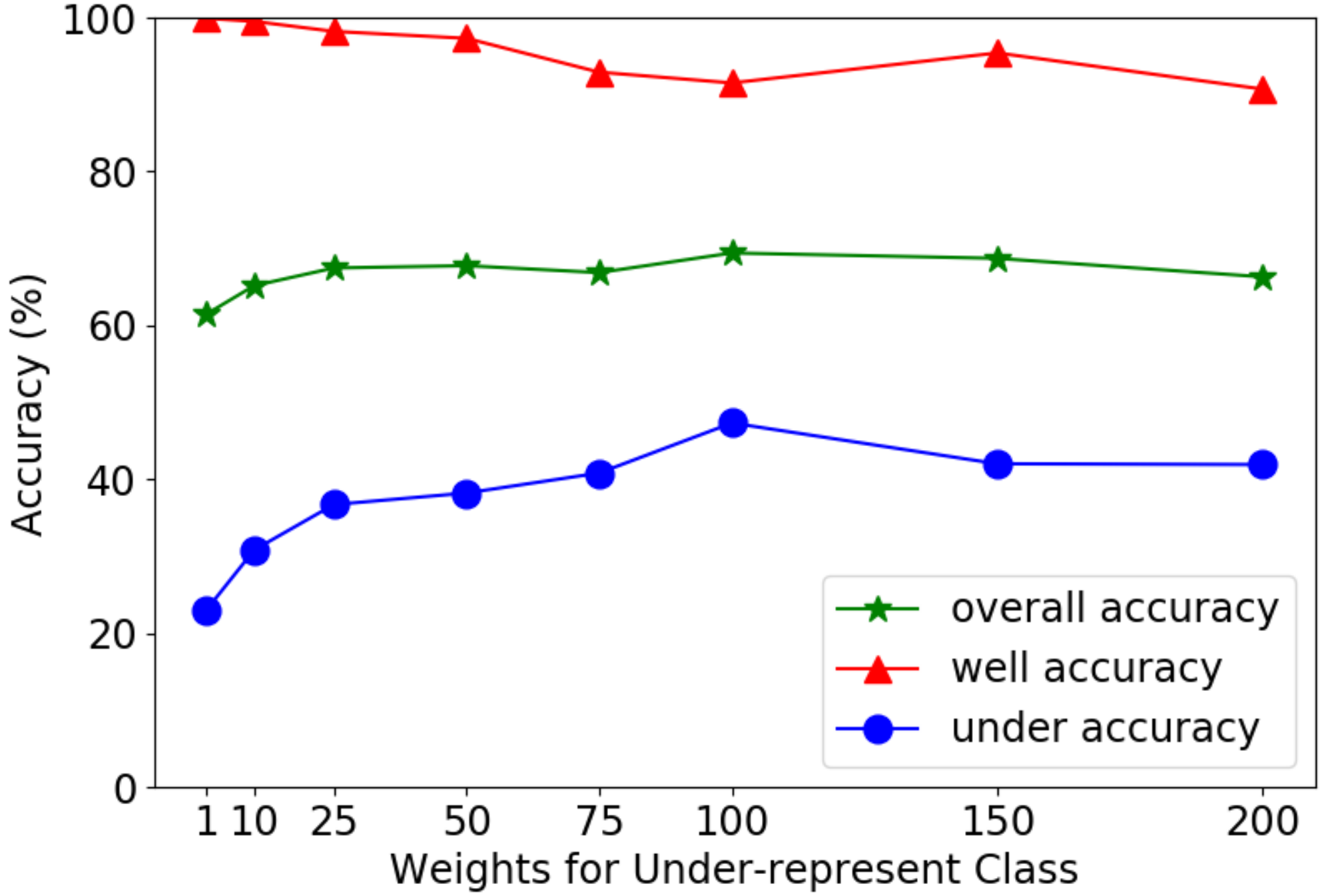}
\caption{Natural Training Standard Acc.}
\label{fig:pre_binary_3_7_nature_standard}
\end{subfigure}
\hspace{0.12in}
\begin{subfigure}[b]{0.31\textwidth}
\centering
\includegraphics[width=1.72in]{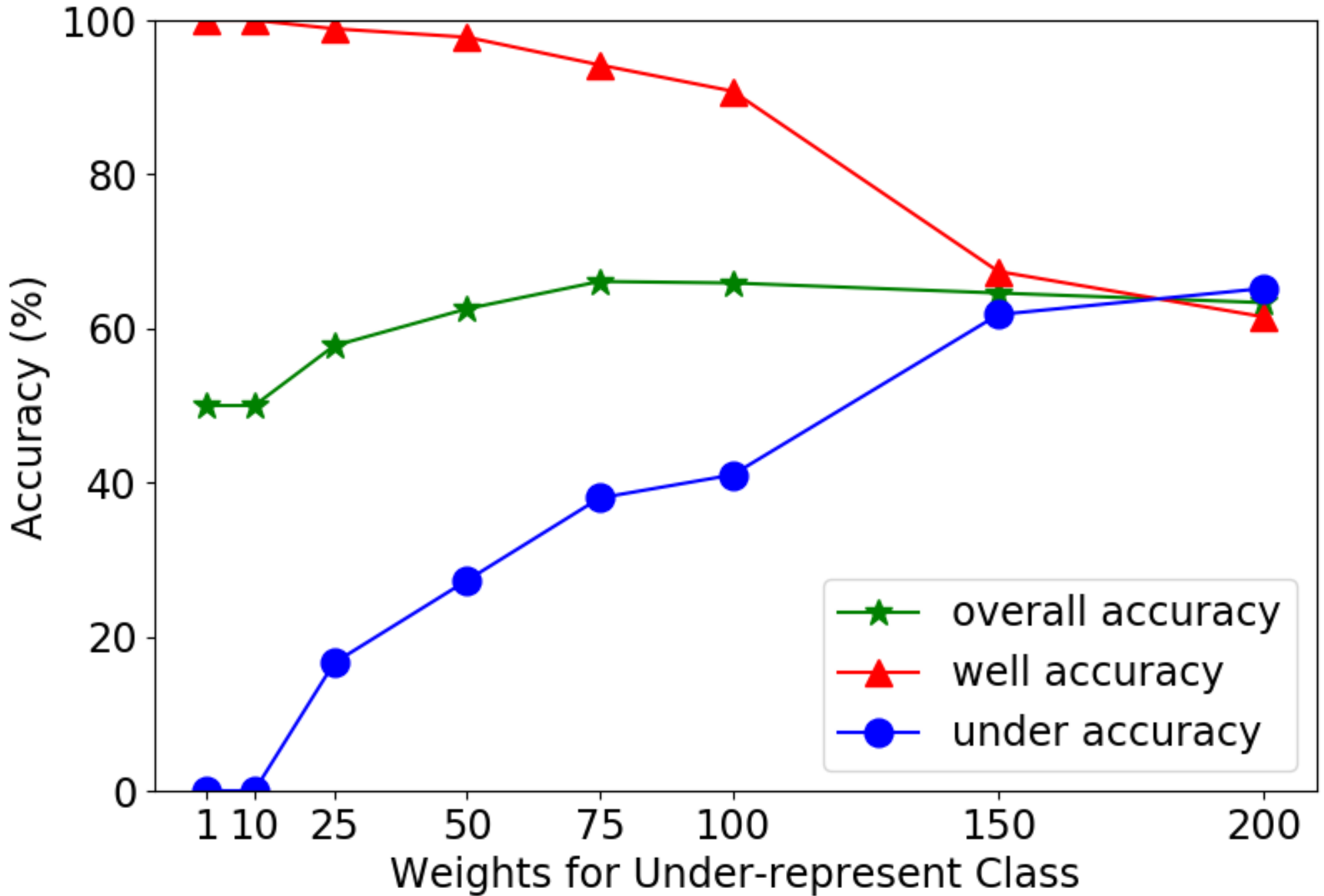}
\caption{Adv. Training Standard Acc.}
\label{fig:pre_binary_3_7_adv_standard}
\end{subfigure}
\hspace{0.12in}
\begin{subfigure}[b]{0.31\textwidth}
\centering
\includegraphics[width=1.72in]{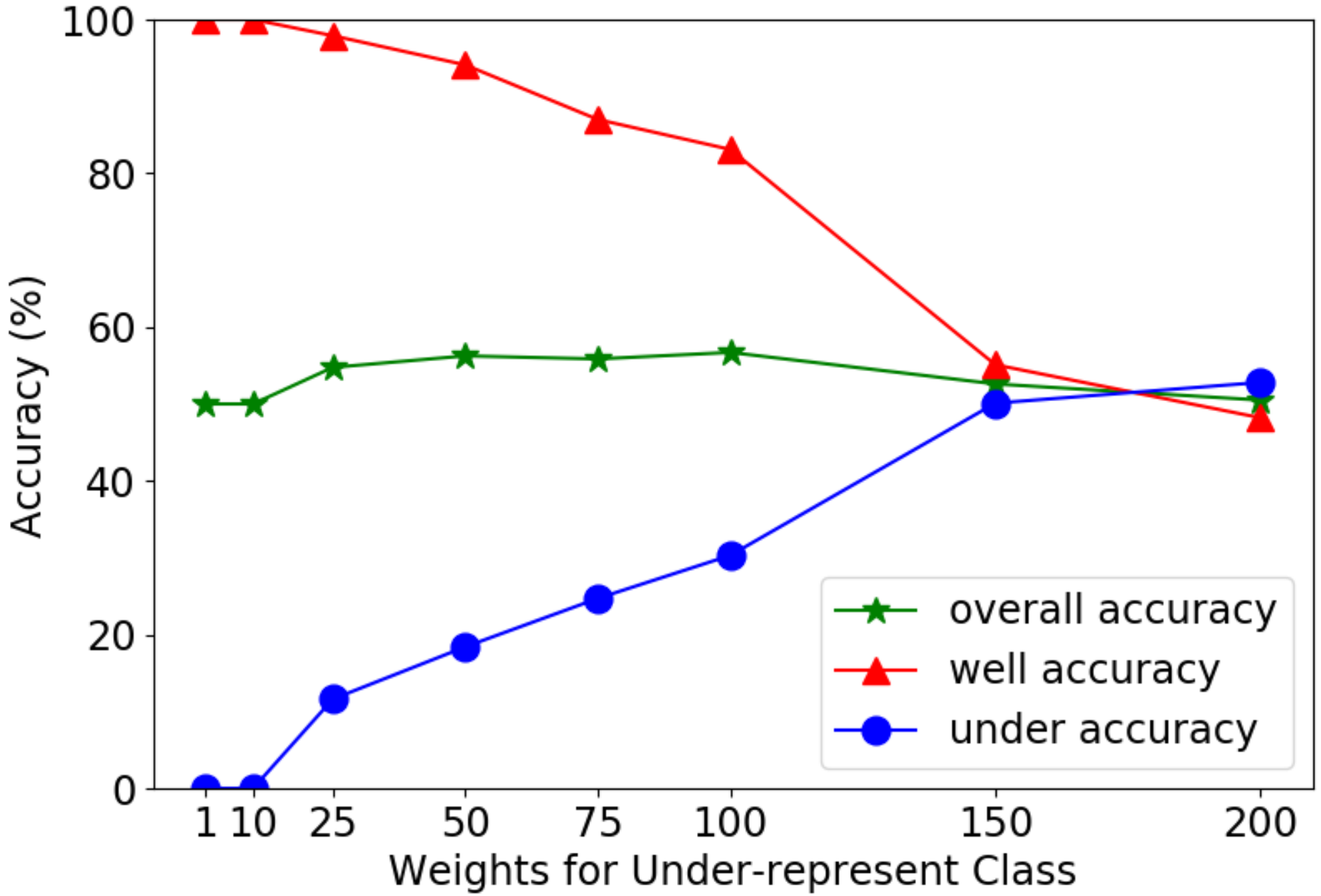}
\caption{Adv. Training Robust Acc.}
\label{fig:pre_binary_3_7_adv_robust}
\end{subfigure}
\caption{Class-wise performance of reweighted natural \& adversarial training in binary classification.}
\label{fig:pre_binary_3_7}
\end{figure}

Figure~\ref{fig:pre_binary_3_7} shows the experimental results with data from the classes ``horse'' and ``cat''. As demonstrated in Figure~\ref{fig:pre_binary_3_7}, increasing the weight of the under-represented class will drastically increase the model's performance of the under-represented class, while also immensely decreasing the performance of the well-represented class. For example, when increasing the weight ratio between two classes from 1:1 to 150:1, the under-represented class's standard accuracy can be improved from 0\% to $\sim 60\%$ and its robust accuracy from $0\%$ to $\sim50\%$. However, the standard \& robust accuracy of the well-represented class is also drastically decreasing. For instance, the well-represented class's standard accuracy drops from 100\% to 60\%, and its robust accuracy drops from 100\% to 50\%. These results illustrate that adversarial training's performance can be significantly affected by the reweighting strategy. As a result, the reweighting strategy in this setting can hardly help improve the overall performance no matter which weight ratio is chosen, because the model's performance always presents a strong tension between these two classes.
As a comparison, for the naturally trained models (Figure~\ref{fig:pre_binary_3_7_nature_standard}), increasing the weights for the under-represented examples will only slightly decrease the performance on the well-represented class. More experiments using different binary imbalanced datasets are reported in Appendix~\ref{app_sec:pre_com2}, where we have similar observations.

\section{Theoretical Analysis}\label{sec:theory}

\begin{figure}[t]
\centering
\begin{subfigure}[b]{0.45\textwidth}
\centering
\includegraphics[width=1.6in]{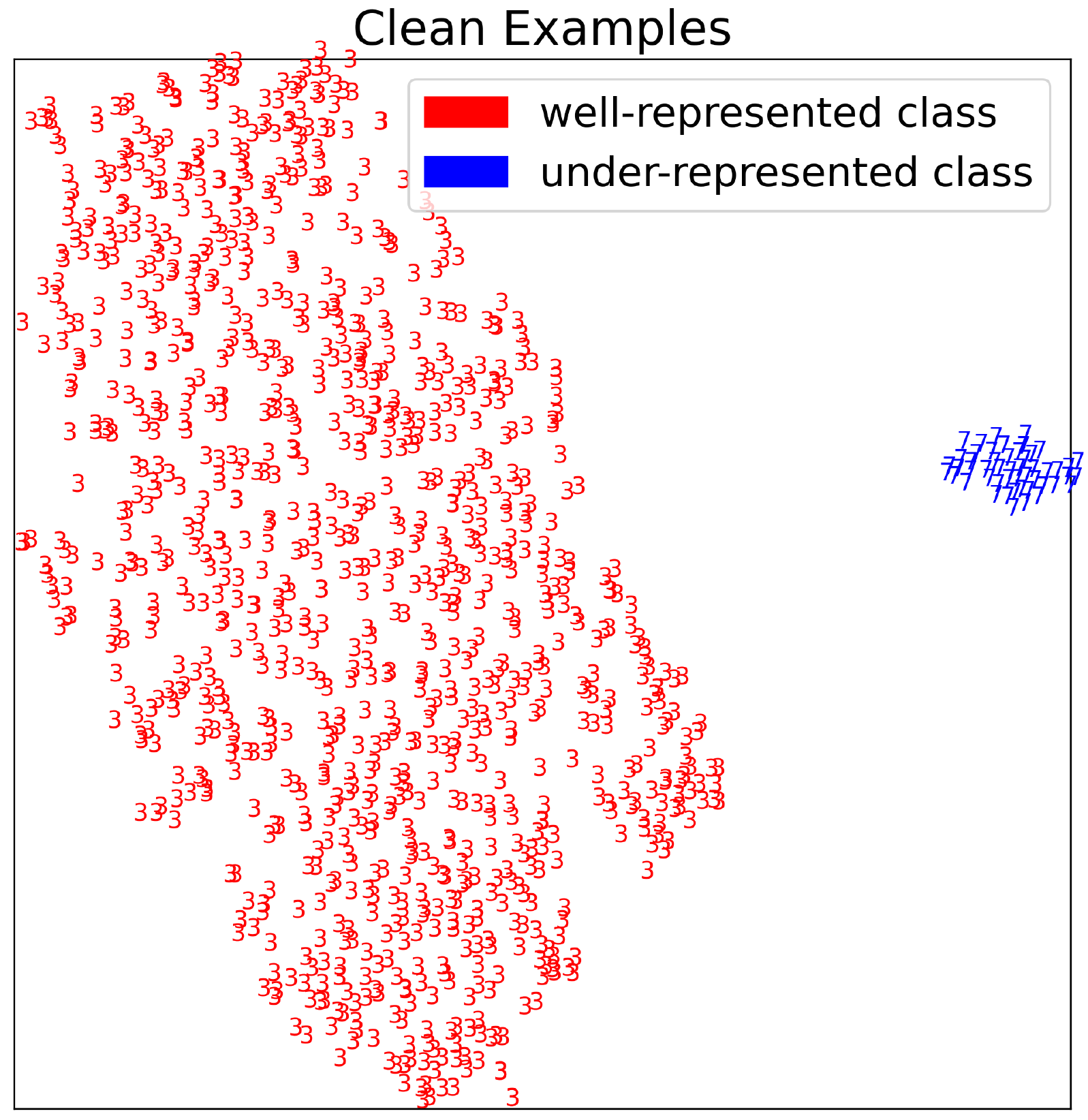}
\caption{Natural Training.}
\label{fig:preliminary_a}
\end{subfigure}
\begin{subfigure}[b]{0.45\textwidth}
\centering
\includegraphics[width=1.6in]{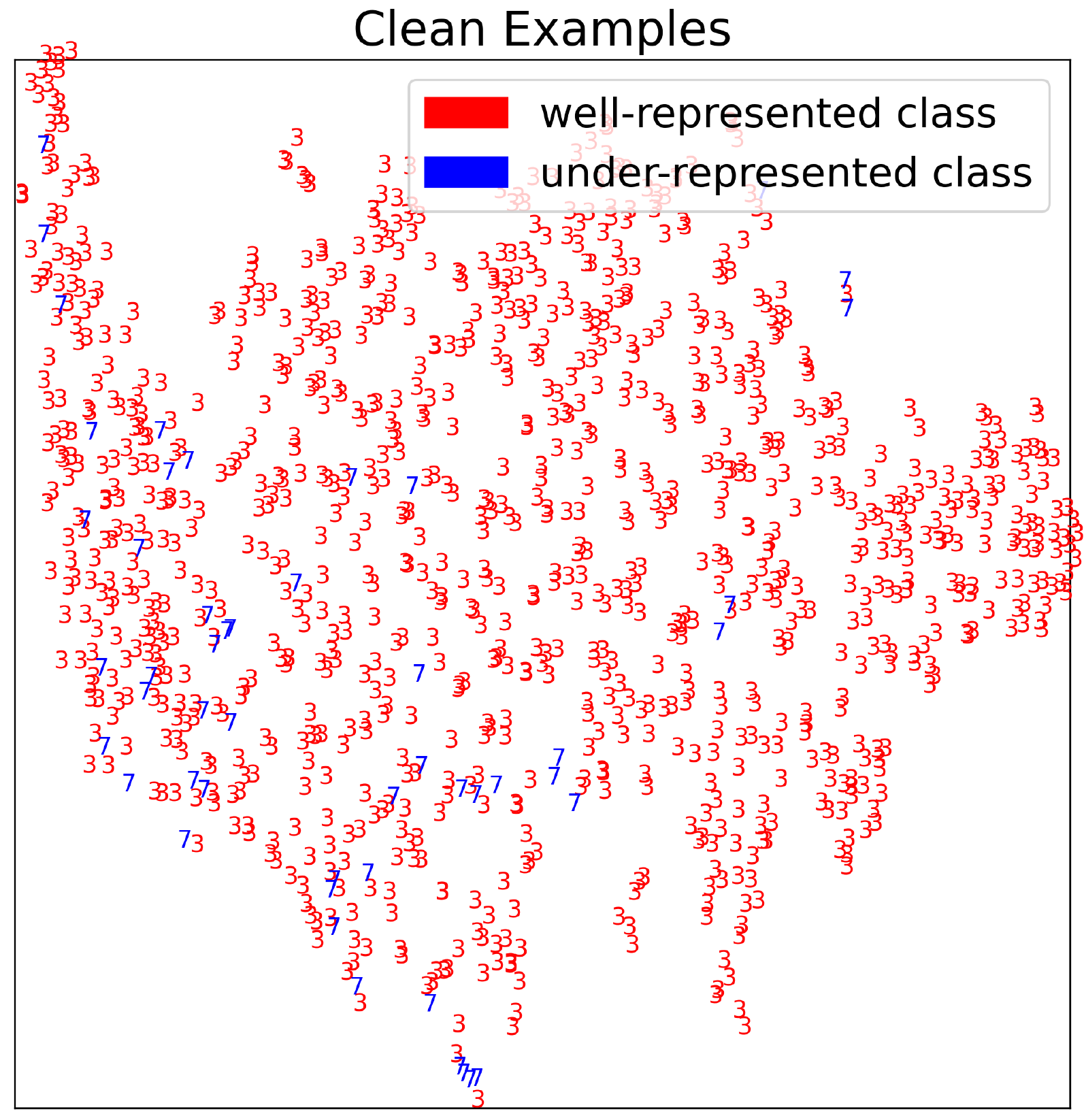}
\caption{Adversarial Training.}
\label{fig:preliminary_b}
\end{subfigure}
\caption{t-SNE visualization of penultimate layer features.}
\label{fig:pre_binary_tsne}
\end{figure}

In Section~\ref{subsec:pre2}, we observe that in natural training, the reweighting strategy can only make a small impact on the two classes' performance. This phenomenon has been extensively studied by recent works~\cite{byrd2019effect, xu2021understanding}, which investigate the decision boundaries of perfectly fitted DNNs. In particular, they consider the case where the data is linearly (or nonlinearly) separable and study the behavior of linear (or nonlinear) models optimized by reweighted SGD algorithms. Interestingly, they conclude that over the process of training, these models' decision boundaries will eventually converge to weight-agnostic solutions. For example, a linear classifier optimized by SGD on a linearly separable data will converge to the solution of the \textit{hard-margin support vector machine}~\cite{noble2006support}. In other words, as long as the data can be well separated, reweighting will not make huge influence on the finally trained models, which is consistent with what we observed above. 

Although these studies only focus on natural training, their interpretations and conclusions motivate our hypothesis in adversarial training. For adversarial training, we conjecture that it is because the models separate the data poorly, thus, their performance is highly sensitive to the reweighting strategy. As a direct validation of this hypothesis, in Figure~\ref{fig:pre_binary_tsne}, we visualize the learned (penultimate layer) features of the imbalanced training examples used in the binary classification problem in Section~\ref{subsec:pre2}. We find that adversarially trained models do present obviously poorer separability on the learned features. This suggests that, compared to naturally trained models, adversarially trained models have a weaker ability to separate training data and could potentially make themselves sensitive to reweighting. Next, we will theoretically analyze the impact of reweighting on linear models which are optimized under poorly separable data. 
Since our empirical study shows that adversarially trained models usually poorly separate the data (see Figure~\ref{fig:pre_binary_tsne}), the analysis can hopefully shed light on the behavior of reweighting in adversarial trained models in practice.

\textbf{Binary Classification Problem.} 
To construct the theoretical study, we focus on a binary classification problem, with a Gaussian mixture distribution $\mathcal{D}$ which is defined as:
\begin{align}
\label{eq:data_dist}
\begin{split}
    & y \sim \{-1, +1\},~~~
     x \sim  
    \begin{cases}
      \mathcal{N}(\mu, \sigma^2I), & \text{if $y= + 1$}\\
      \mathcal{N}(-\mu, \sigma^2I), & \text{if $y= - 1$}\\
    \end{cases} 
    \text{ \;  and } \mu = ( \overbrace{\eta,...,\eta}^\text{dim = d}),
\end{split}
\end{align}
where the two classes' centers $(\pm\mu \in \R ^d)$ with each dimension has mean value $\pm\eta$ ($\eta > 0$), variance $\sigma^2$. Formally, we define the data \textbf{\textit{separability}} as $S = \eta / \sigma^2$. Intuitively, if the separability term $S$ is larger, it suggests that two classes have farther distance or data examples of each class are more concentrated, so these two classes can be well separated. Previous works~\cite{byrd2019effect} also closely studied this term to describe data separability. 
Besides, we particularly define the imbalanced training dataset satisfying the condition $\text{Pr.}(y=+1) = K \cdot \text{Pr.}(y=-1)$ and $K>1$ which indicates the imbalance ratio between the two classes. During test, we assume that two classes have the equal probability to appear. Under data distribution $\mathcal{D}$, we will discuss the performance of linear classifiers $f(x) = \text{sign}(w^Tx -b)$ where $w$ and $b$ are the weight and bias term of model $f$. If a reweighting strategy is involved, we define the model will upweight the under-represented class ``-1'' by $\rho$. In the following lemma, we first derive the solution of the optimized linear classifier $f$ training on this imbalanced dataset. Then we will extend the result of Lemma~\ref{lemma1} to analyze the impact of data separability on the performance of model $f$. 

\begin{restatable}[]{lem}{lemm}
\label{lemma1}
Under the data distribution $\mathcal{D}$ as defined in Eq.~(\ref{eq:data_dist}), with an imbalanced ratio $K$ and a reweight ratio $\rho$, the optimal classifier which minimizes the (reweighted) empirical risk:
\begin{align}
\begin{split}
f^* = \argmin_f \bigg( &\text{Pr.}(f(x)\neq y | y = -1)\cdot \text{Pr.}(y = -1) \cdot \rho \\  
+ \; &\text{Pr.}(f(x)\neq y | y = +1)\cdot \text{Pr.}(y = +1) \bigg) 
\end{split}
\end{align}
has the solution: $w = \textbf{1}$ and $b = \frac{1}{2} \log (\frac{\rho}{K}) \frac{d\sigma^2}{\eta} = \frac{1}{2} \log (\frac{\rho}{K}) \frac{d}{S}$. 
\end{restatable}

The proof of Lemma~\ref{lemma1} can be found at Appendix~\ref{app_sec:lemma1}. Note that the final optimized classifier has a weight vector equal to $\bf 1$ and its bias term $b$ only depends on $K$, $\rho$ and the data separability $S$. In the following, our first theorem is focused on one special setting when $\rho = 1$, which is the original ERM model without reweighting. Specifically, Theorem~\ref{theorem1} calculates and compares the model's performance under data distributions: $\mathcal{D}_1$ (with a higher separability $S_1$) and $\mathcal{D}_2$ (with a lower separability $S_2$). From Theorem~\ref{theorem1}, we aim to compare the behavior of linear models when they can poorly separate data (like adversarial trained models) or they can well separate data (like naturally trained models).

\begin{restatable}[]{thm}{theoryA}
\label{theorem1}
Under two data distributions $(x^{(1)}, y^{(1)}) \in \mathcal{D}_1$ and $(x^{(2)}, y^{(2)}) \in \mathcal{D}_2$ with the separability $S_1 > S_2$, let $f_1^*$ and $f_2^*$ be the optimal non-reweighted classifiers ($\rho=1$) under $\mathcal{D}_1$ and $\mathcal{D}_2$, respectively. Given the imbalance ratio $K$ is large enough, we have:
\begin{align}
\begin{split}
&\text{Pr.}(f_1^*(x^{(1)})\neq y^{(1)} | y^{(1)} = -1) - 
\text{Pr.}(f_1^*(x^{(1)})\neq y^{(1)} | y^{(1)} = +1) \\ <~ 
&\text{Pr.}(f_2^*(x^{(2)})\neq y^{(2)} | y^{(2)} = -1) - 
\text{Pr.}(f_2^*(x^{(2)})\neq y^{(2)} | y^{(2)} = +1).
\end{split}
\end{align}
\end{restatable}

The proof of Theorem~\ref{theorem1} is provided at Appendix~\ref{app_sec:theorem1}. Intuitively, Theorem~\ref{theorem1} suggests that when the data separability $S$ is low (such as $\mathcal{D}_2$), the optimized classifier (without reweighting) can intrinsically have a larger error difference between the under-represented class ``-1'' and the well-represented class ``+1''. Similar to the observation in Section~\ref{subsec:pre1} and Figure~\ref{fig:pre_binary_tsne}, adversarially trained models also present a weak ability to separate data, and it also presents a strong performance gap between the well-represented class and under-represented class. Conclusively, Theorem~\ref{theorem1} indicates that the poor ability to separate the training data can be one important reason which leads to the strong performance gap of adversarially trained models.

Next, we consider the case when the reweighting strategy is applied. Similar to Theorem~\ref{theorem1}, we also calculate the models' classwise error under $\mathcal{D}_1$ and $\mathcal{D}_2$ with different levels of separability. In particular, Theorem~\ref{theorem2} focuses on the well-represented class ``+1'' and calculates its error increase when upweighting the under-represented class ``-1'' by $\rho$. Through the analysis in Theorem~\ref{theorem2}, we compare the impact of upweighting the under-represented class on the performance of well-represented class.

\begin{restatable}[]{thm}{theoryB}
\label{theorem2}
Under two data distributions $(x^{(1)}, y^{(1)}) \in \mathcal{D}_1$ and $(x^{(2)}, y^{(2)}) \in \mathcal{D}_2$ with different separability $S_1 > S_2$, let $f_1^*$ and $f_2^*$ be the optimal non-reweighted classifiers ($\rho=1$) under $\mathcal{D}_1$ and $\mathcal{D}_2$ respectively, and let ${f'_1}^*$ and ${f'_2}^*$ be the optimal reweighted classifiers under $\mathcal{D}_1$ and $\mathcal{D}_2$ given the optimal reweighting ratio ($\rho = K$).  Given the imbalance ratio $K$ is large enough, we have:
\begin{align}
\begin{split}
&\text{Pr.}({f'_1}^*(x^{(1)})\neq y^{(1)} | y^{(1)} = +1) - 
\text{Pr.}(f_1^*(x^{(1)})\neq y^{(1)} | y^{(1)} = +1) \\ <~ 
&\text{Pr.}({f'_2}^*(x^{(2)})\neq y^{(2)} | y^{(2)} = +1) - 
\text{Pr.}(f_2^*(x^{(2)})\neq y^{(2)} | y^{(2)} = +1).
\end{split}
\end{align}
\end{restatable}

The detail the proof of Theorem~\ref{theorem2} at Appendix~\ref{app_sec:theorem2}. The theorem shows that, when the data distribution has poorer data separability, such as $\mathcal{D}_2$, upweighting the under-represented class can cause greater hurt on the performance of the well-represented class. It is also consistent with our empirical findings about adversarial training models. Since the adversarially trained models poorly separate the data (Figure~\ref{fig:pre_binary_tsne}), upweighting the under-represented class always drastically decreases the performance of well-represented class (Section~\ref{subsec:pre2}). Through the discussions in both Theorem~\ref{theorem1} and Theorem~\ref{theorem2}, we can conclude that the poor separability can be one important reason which makes adversarial training and its reweighted variants extremely difficult to achieve good performance under imbalance data distribution. Therefore, in the next section, we explore potential solutions which can facilitate the reweighting strategy in adversarial training.

\section{Separable Reweighted Adversarial Training (SRAT)}\label{sec:methodology}

The observations from both preliminary studies and theoretical understandings indicate that more separable data will advance the reweighting strategy in adversarial training under imbalanced scenarios. Thus, in this section, we present a framework, Separable Reweighted Adversarial Training (SRAT), that enables the effectiveness of the reweighting strategy in adversarial training under imbalanced scenarios by increasing the separability in the learned latent feature space. 

\subsection{Reweighted Adversarial Training}

Given an input example $(\mathbf{x}, y)$, adversarial training~\cite{madry2017towards} aims to obtain a robust model $f_\theta$ that can make the same prediction $y$ for an adversarial example $\mathbf{x}'$, generated by applying an adversarially perturbation on $\mathbf{x}$. The adversarially perturbations are typically bounded by a small value $\epsilon$ under $L_p$-norm, i.e., $\Vert \mathbf{x}' - \mathbf{x} \Vert_p \leq \epsilon$. More formally, adversarial training can be formulated as solving a min-max optimization problem, where a DNN model is trained on minimizing the prediction error on adversarial examples generated by iteratively maximizing some loss function.

As indicated in Section~\ref{subsec:pre1}, adversarial training cannot be applied in imbalanced scenarios directly, as it presents very low performance on under-represented classes. To tackle this problem, a natural idea is to integrate existing imbalanced learning strategies proposed in natural training, such as reweighting, into adversarial training to improve the trained model's performance on those under-represented classes. Hence, the reweighted adversarial training can be defined as
\begin{align}
    \min_{\theta} \frac{1}{n} \sum_{i = 1}^n \max_{\Vert \mathbf{x}_i' - \mathbf{x}_i \Vert_p \leq \epsilon} w_i \mathcal{L}(f_{\theta}(\mathbf{x}'_i), y_i),
    \label{eq:reweight_adv_training}
\end{align}
where $w_i$ is a reweighting value assigned for each input sample $(\mathbf{x}_i, y_i)$ based on the example size of the class $(\mathbf{x}_i, y_i)$ belongs to or some properties of $(\mathbf{x}_i, y_i)$. 
In most existing adversarial training methods~\cite{madry2017towards,zhang2019theoretically,wang2019improving}, the cross entropy (CE) loss is adopted as the loss function $\mathcal{L}(\cdot, \cdot)$. However, the CE loss could be suboptimal in imbalanced settings and some new loss functions designed for imbalanced settings specifically, such as Focal loss~\cite{lin2017focal} and LDAM loss~\cite{cao2019learning}, have been prove superiority in natural training. Hence, besides CE loss, Focal loss and LDAM loss can also be adopted as the loss function $\mathcal{L}(\cdot, \cdot)$ in Eq.~(\ref{eq:reweight_adv_training}).

\subsection{Increasing Feature Separability}

Our preliminary study indicates that only reweighted adversarial training cannot work well under imbalanced scenarios. Moreover, the reweighting strategy behaves very differently between natural training and adversarial training. Meanwhile, our theoretical analysis suggests that the poor separability of the feature space produced by the adversarially trained model can be one reason to understand these observations. Hence, in order to facilitate the reweighting strategy in adversarial training under imbalanced scenarios, we equip a feature separation loss with our SRAT method. We aim to enforce the learned feature space as separable as possible. More specifically, the goal of the feature separation loss is to make (1) the learned features of examples from the same class well clustered, and (2) the features of examples from different classes well separated. By achieving this goal, the model is able to learn more discriminative features for each class. Correspondingly adjusting the decision boundary via the reweighting strategy to fit under-represented classes' examples more will not hurt well-represented classes drastically. The feature separation loss is formally defined as:
\begin{align}
\begin{split}
    \mathcal{L}_{sep}(\mathbf{x}'_i) = -\frac{1}{\vert P(i) \vert} \sum_{p \in P(i)} \log \frac{\exp (\mathbf{z}'_i \cdot \mathbf{z}'_p / \tau)}{\sum_{a \in A(i)} \exp (\mathbf{z}'_i \cdot \mathbf{z}'_a / \tau)},
\label{eq:separation_loss}
\end{split}
\end{align}
where $\mathbf{z}'_i$ is the feature representation of the adversarial example $\mathbf{x}'_i$ of $\mathbf{x}_i$, $\tau \in \mathcal{R}^+$ is a scalar temperature parameter, $P(i)$ denotes the set of input examples belonging to the same class with $\mathbf{x}_i$ and $A(i)$ indicates the set of all input examples excepts $\mathbf{x}'_i$. When minimizing the feature separation loss during training, the learned features of examples from the same class will tend to aggregate together in the latent feature space, and, hence, result in a more separable latent feature space. Our proposed feature separation loss $\mathcal{L}_{sep}(\cdot)$ is inspired by the supervised contrastive loss proposed in~\cite{khosla2020supervised}. The main difference is, instead of applying data augmentation techniques to generate two different views of each data example and feeding the model with augmented data examples, our feature separation loss directly takes the adversarial example $\mathbf{x}'_i$ of each data example $\mathbf{x}_i$ as input.

\subsection{Training Schedule}

By combining the feature separation loss with the reweighted adversarial training, the final object function for Separable reweighted Adversarial Training (SRAT) can be defined as:
\begin{align}
\label{eq:objective}
\min_{\theta} \frac{1}{n} \sum_{i = 1}^n \max_{\Vert \mathbf{x}_i' - \mathbf{x}_i \Vert_p \leq \epsilon} w_i \mathcal{L}(f_{\theta}(\mathbf{x}'_i), y_i) + \lambda \mathcal{L}_{sep}(\mathbf{x}'_i),
\end{align}
where we use a hyper-parameter $\lambda$ to balance the contributions from the reweighted adversarial training and the feature separation loss.

In practice, in order to better take advantage of the reweighting strategy in our SRAT method, we adopt a deferred reweighting training schedule~\cite{cao2019learning}. Specifically, before annealing the learning rate, our SRAT method first trains a model guided by Eq.~(\ref{eq:objective}) without introducing the reweighting strategy, i.e., setting $w_i = 1$ for every input example $\mathbf{x}'_i$, and then applies reweighting into model training process with a smaller learning rate. Our SRAT method enables to learn more separable feature space, thus comparing with applying the reweighting strategy from the beginning of training, this deferred re-balancing training schedule enables the reweighting strategy to obtain more benefits from our SRAT method, and as a result, it can boost the performance of our SRAT method with the help of the reweighting strategy. The detailed training algorithm for SRAT is shown in Appendix~\ref{app_sec:algorithm}.

\section{Experiment}\label{sec:experiment}

In this section, we perform comprehensive experiments to validate the effectiveness of our proposed SRAT method. We first compare our method with several representative imbalanced learning methods in adversarial training under various imbalanced scenarios and then conduct ablation study to understand our method more deeply. 

\subsection{Experimental Settings}\label{subsec:datasets}

\textbf{Datasets.} We conduct experiments on multiple imbalanced training datasets artificially created from two benchmark image datasets CIFAR10~\cite{krizhevsky2009learning} and SVHN~\cite{netzer2011reading} with diverse imbalanced distributions. Specifically, we consider two types of imbalance types: Exponential (Exp) imbalance~\cite{cui2019class} and Step imbalance~\cite{buda2018systematic}. For Exp imbalance, the number of training examples of each class will be reduced according to an exponential function $n=n_i\tau^i$, where $i$ is the class index, $n_i$ is the number of training examples in the original CIFAR10/SVHN training dataset for class $i$ and $\tau \in (0,1)$. We categorize five most frequent classes in the constructed imbalanced training dataset as well-represented classes and the remaining five classes as under-represented classes. For Step imbalance, we follow the same process adopted in Section~\ref{subsec:pre1} to construct imbalanced training datasets based on CIFAR10 and SVHN, separately. Moreover, in both imbalanced types, we denote \emph{imbalance ratio} $K$ as the ratio between training example sizes of the most frequent and least frequent class. In our experiments, we construct four different imbalanced datasets, named as ``Step-100", ``Step-10", ``Exp-100" and ``Exp-10", by adopting different imbalanced types (Step or Exp) with different imbalanced ratios ($K=100$ or $K=10$) to train models, and evaluate model's performance on the original uniformly distributed test datasets of CIFAR10 and SVHN correspondingly. More detailed information about imbalanced training sets used in our experiments can be found in Appendix~\ref{app_sec:dataset}.

\textbf{Baseline methods.} We implement several representative and state-of-the-art imbalanced learning methods (or their combinations) into adversarial training as baseline methods. These methods include: (1) Focal loss (Focal); (2) LDAM loss (LDAM); (3) Class-balanced reweighting (CB-Reweight)~\cite{cui2019class}, where each example is reweighted proportionally by the inverse of the effective number\footnote{The effective number is defined as the volume of examples and can be calculated by $(1-\beta^{n_i})/(1-\beta)$, where $\beta \in [0, 1)$ is a hyperparameter and $n_i$ denotes the number of examples of class $i$.} of its class; (4) Class-balanced Focal loss (CB-Focal)~\cite{cui2019class}, a combination of Class-balanced method~\cite{cui2019class} and Focal loss~\cite{lin2017focal}, where well-classified examples will be down-weighted while hard-classified examples will be up-weighted controlled by their corresponding effective numbers; (5) deferred reweighted CE loss (DRCB-CE), where a deferred reweighting training schedule is applied based on the CE loss; (6) deferred reweighted Class-balanced Focal loss (DRCB-Focal), where a deferred reweighting training schedule is applied based on the CB-Focal loss; (7) deferred reweighted Class-balanced LDAM loss (DRCB-LDAM)~\cite{cao2019learning}, where a deferred reweighting training schedule is applied based on the CB-LDAM loss. In addition, we also include the original PGD adversarial training method using cross entropy loss (CE) in our experiments.

\textbf{Our proposed methods.} We evaluate three variants of our proposed SRAT method with different implementations of the prediction loss $\mathcal{L}(\cdot, \cdot)$ in Eq.~(\ref{eq:reweight_adv_training}), i.e., CE loss, Focal loss and LDAM loss. The variant utilizing CE loss is denoted as SRAT-CE, and, similarly, other two variants are denoted as SRAT-Focal and SRAT-LDAM, respectively. For all these three variants, Class-balanced method~\cite{cui2019class} is adopted to set reweighting values within the deferred reweighting training schedule.  

\textbf{Implementation details.} We implement all methods used in our experiments based on a Pytorch library DeepRobust~\cite{li2020deeprobust}. For CIFAR10 based imbalanced datasets, the adversarial examples used in training are calculated by PGD-10, with a perturbation budget $\epsilon=8/255$, and step size $\gamma=2/255$. For robustness evaluation, we report robust accuracy under $l_{\infty}$-norm $8/255$ attacks generated by PGD-20 on Resnet-18~\cite{he2016deep} models. For SVHN based imbalanced datasets, the setting is similar with CIFAR10 based datasets, excepts we set step size $\gamma$ to $1/255$ in both training and test phases, as suggested in~\cite{wu2020adversarial}. For the deferred reweighting training schedule used in our methods and some baseline methods, we set the number of the training epochs to 200 and the initial learning rate to 0.1, and then decay the learning rate at epoch 160 and 180 with the ratio 0.01. The reweighting strategy will be applied starting from epoch 160.

\subsection{Performance Comparison}

Table~\ref{tab:step_main_result} and~\ref{tab:longtail_main_result} show the performance comparison on various imbalanced CIFAR10 datasets with different imbalance types and imbalance ratios. In these two tables, we use bold values to denote the highest accuracy among all methods and use the underline values to indicate our SRAT variants which achieve the highest accuracy among their corresponding baseline methods utilizing the same loss function for making predictions. Due to the limited space, we report the performance comparison on SVHN based imbalanced datasets in Appendix~\ref{app_sec:svhn_results}. 

From Table~\ref{tab:step_main_result} and Table~\ref{tab:longtail_main_result}, we can make the following observations. First, compared to baseline methods, our SRAT methods can obtain improved performance in terms of both overall standard \& robust accuracy under almost all imbalanced settings. More importantly, our SRT methods make significantly improvement on those under-represented classes, especially under the extremely imbalanced setting. For example, on the Step imbalanced dataset with imbalance ratio $K=100$, our SRAT-Focal method improves the standard accuracy on under-represented classes from 21.81\% achieved by the best baseline method utilizing Focal loss to 51.83\% and robust accuracy from 3.24\% to 15.89\%. These results demonstrate that our proposed SRAT method is able to obtain more robustness under imbalanced settings. Second, the performance gap among three variants SRAT-CE, SRAT-Focal and SRAT-LDAM are mainly caused by the gap between the loss functions in these methods. As shown in Table~\ref{tab:step_main_result} and~\ref{tab:longtail_main_result}, DRCB-LDAM typically performs better than DRCE-CE and DRCB-Focal, and similarly, SRAT-LDAM outperforms SRAT-CE and SRAT-Focal under corresponding imbalanced settings. 

\begin{table}[t]
\small
\centering
\caption{Performance comparison on imbalanced CIFAR10 datasets (Imbalanced Type: Step)}
\label{tab:step_main_result}
\begin{tabular}{c|c|c|c|c|c|c|c|c}
\hline
Imbalance Ratio & \multicolumn{4}{c}{10} & \multicolumn{4}{|c}{100} \\
\hline
Imbalance Ratio & \multicolumn{2}{c}{Standard Accuracy} & \multicolumn{2}{|c}{Robust Accuracy} & \multicolumn{2}{|c}{Standard Accuracy} & \multicolumn{2}{|c}{Robust Accuracy} \\
\hline
Method & Overall & Under & Overall & Under & Overall & Under & Overall & Under \\
\hline
CE  & 63.26 & 40.62 & 36.96 & 14.23 & 47.29 & 9.03 & 30.39 & 1.62 \\
Focal & 63.57 & 41.17 & 36.89 & 14.25 & 47.36 & 9.03 & 30.12 & 1.45 \\
LDAM & 57.08 & 31.09 & 37.18 & 12.44 & 42.49 & 0.85 & 30.80 & 0.05 \\
CB-Reweight & 73.30 & 74.80 & 41.34 & 42.15 & 37.68 & 19.64 & 25.58 & 10.33 \\
CB-Focal & 73.47 & 73.69 & 41.19 & 41.02 & 15.44 & 0.00 & 14.46 & 0.00 \\
DRCB-CE & 75.89 & 70.55 & 39.93 & 33.33 & 53.40 & 22.86 & 28.31 & 3.35 \\
DRCB-Focal & 74.61 & 67.06 & 37.91 & 29.50 & 52.75 & 21.81 & 27.78 & 3.24 \\
DRCB-LDAM & 72.95 & 75.42 & 45.23 & 44.98 & 61.60 & 50.69 & 31.37 & 16.25 \\
\hline
\hline
SRAT-CE & \underline{\textbf{76.32}} & 73.20 & \underline{41.71} & 37.86 & \underline{59.10} & \underline{40.24} & 30.02 & \underline{11.72} \\
SRAT-Focal & \underline{75.41} & \underline{74.91} & \underline{42.05} & \underline{41.28} & \underline{62.93} & \underline{51.83} & 28.38 & \underline{15.89} \\
SRAT-LDAM & \underline{73.99} & \underline{\textbf{76.63}} & \underline{\textbf{45.60}} & \underline{\textbf{45.96}} & \underline{\textbf{63.13}} & \underline{\textbf{52.73}} & \underline{\textbf{33.51}} & \underline{\textbf{18.89}} \\
\hline
\end{tabular}
\end{table}

\begin{table}[t]
\small
\centering
\caption{Performance comparison on imbalanced CIFAR10 datasets (Imbalanced Type: Exp). }
\label{tab:longtail_main_result}
\begin{tabular}{c|c|c|c|c|c|c|c|c}
\hline
Imbalance Ratio & \multicolumn{4}{c}{10} & \multicolumn{4}{|c}{100} \\
\hline
Metric & \multicolumn{2}{c}{Standard Accuracy} & \multicolumn{2}{|c}{Robust Accuracy} & \multicolumn{2}{|c}{Standard Accuracy} & \multicolumn{2}{|c}{Robust Accuracy} \\
\hline
Method & Overall & Under & Overall & Under & Overall & Under & Overall & Under \\
\hline
CE & 71.95 & 64.09 & 37.94 & 26.79 & 48.40 & 23.04 & 26.94 & 6.17 \\
Focal & 72.06 & 63.99 & 37.62 & 26.27 & 49.16 & 23.69 & 26.84 & 5.88 \\
LDAM & 67.39 & 58.01 & 41.35 & 28.65 & 48.39 & 25.69 & 29.51 & 8.95 \\
CB-Reweight & 75.17 & 76.87 & 41.02 & 41.67 & 57.49 & 56.47 & 29.01 & 26.53 \\
CB-Focal & 74.73 & 76.67 & 38.86 & 42.41 & 50.35 & 60.05 & 27.15 & \textbf{33.56} \\
DRCB-CE & 76.25 & 75.83 & 40.02 & 37.93 & 57.30 & 37.90 & 26.97 & 10.57 \\
DRCB-Focal & 75.36 & 72.72 & 37.76 & 33.83 & 54.76 & 31.79 & 25.24 & 7.81 \\
DRCB-LDAM & 73.92 & 78.53 & 46.29 & 48.81 & 62.65 & 57.19 & 31.66 & 22.11 \\
\hline
\hline
SRAT-CE & \underline{\textbf{76.94}} & \underline{79.50} & \underline{41.50} & \underline{43.08} & \underline{\textbf{64.93}} & \underline{64.34} & \underline{29.68} & 25.42 \\
SRAT-Focal & 75.26 & \underline{\textbf{80.52}} & \underline{42.37} & \underline{47.22} & \underline{62.57} & \underline{64.88} & \underline{30.34} & 28.66 \\
SRAT-LDAM & \underline{74.63} & \underline{79.82} & \underline{\textbf{46.72}} & \underline{\textbf{50.38}} & \underline{63.11} & \underline{\textbf{65.60}} & \underline{\textbf{34.22}} & \underline{32.55} \\
\hline
\end{tabular}
\end{table}

\subsection{Ablation Study}

In this subsection, we provide ablation study to understand our SRAT method more comprehensively.

\begin{figure}[b]
\begin{subfigure}[b]{0.32\textwidth}
\centering
\includegraphics[width=1.4in]{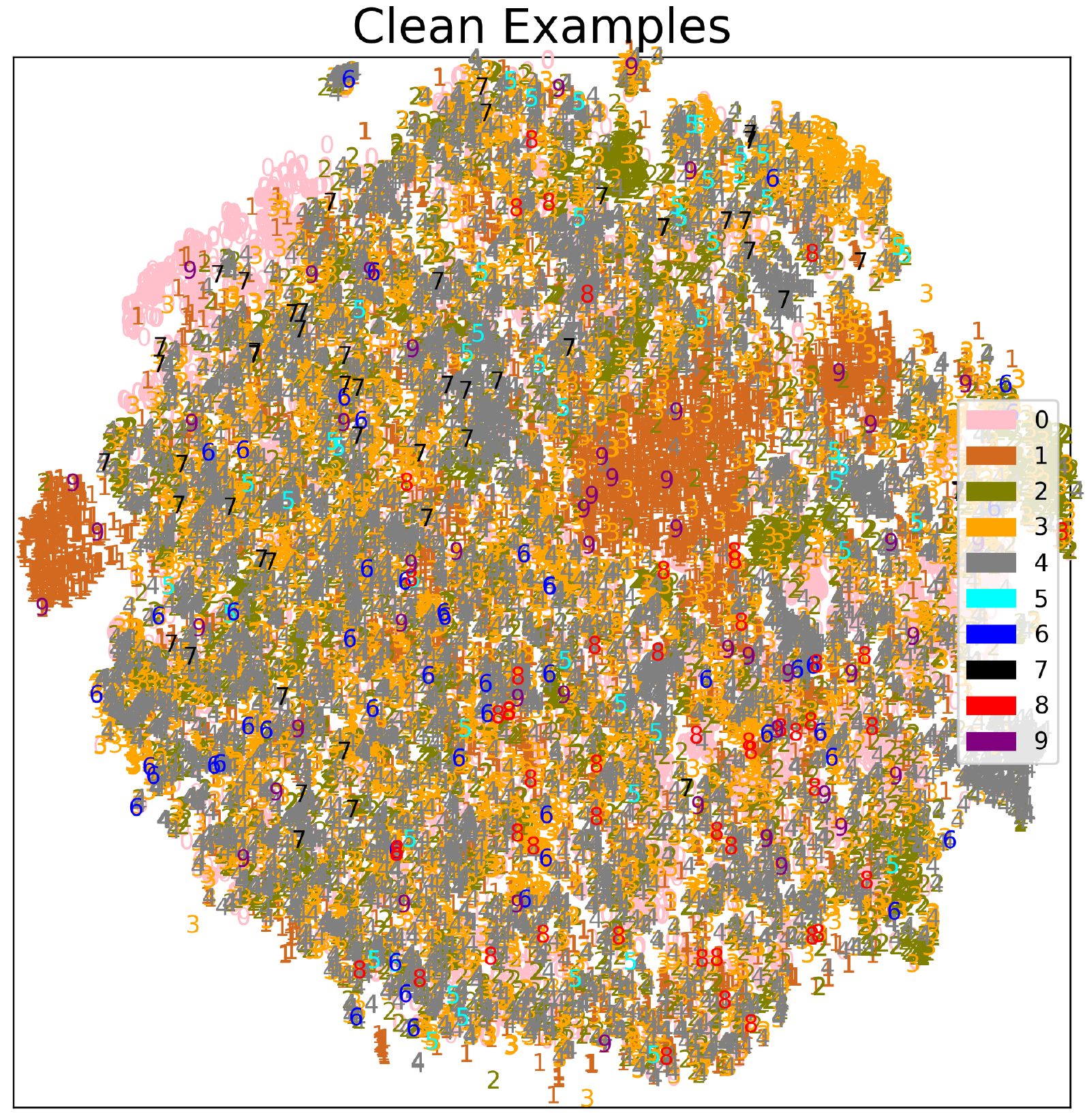}
\caption{CE.}
\label{fig:tsne_erm}
\end{subfigure}
\hspace{0.12in}
\begin{subfigure}[b]{0.3\textwidth}
\centering
\includegraphics[width=1.4in]{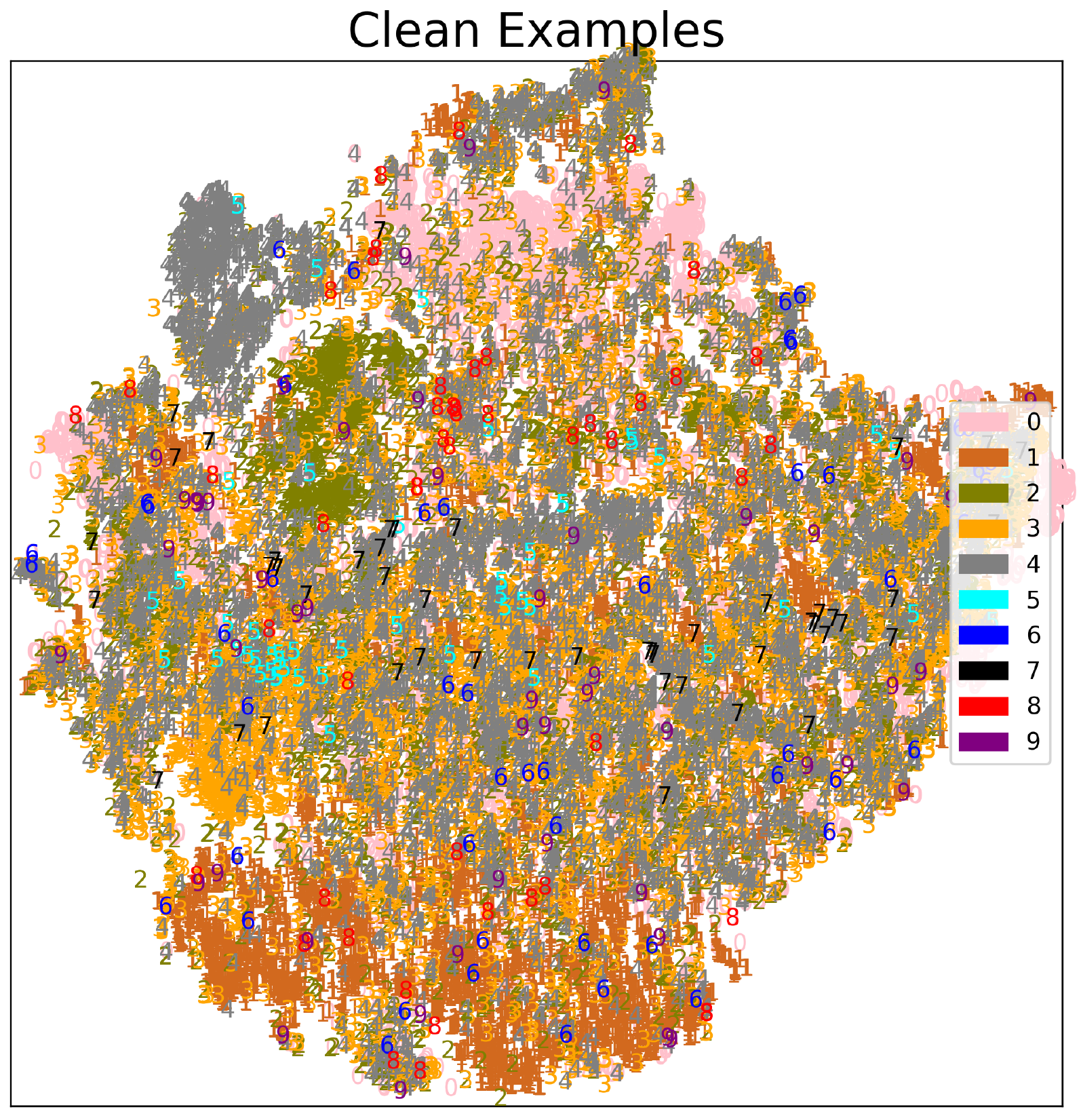}
\caption{DRCB-LDAM.}
\label{fig:tsne_ldam}
\end{subfigure}
\hspace{0.12in}
\begin{subfigure}[b]{0.3\textwidth}
\centering
\includegraphics[width=1.4in]{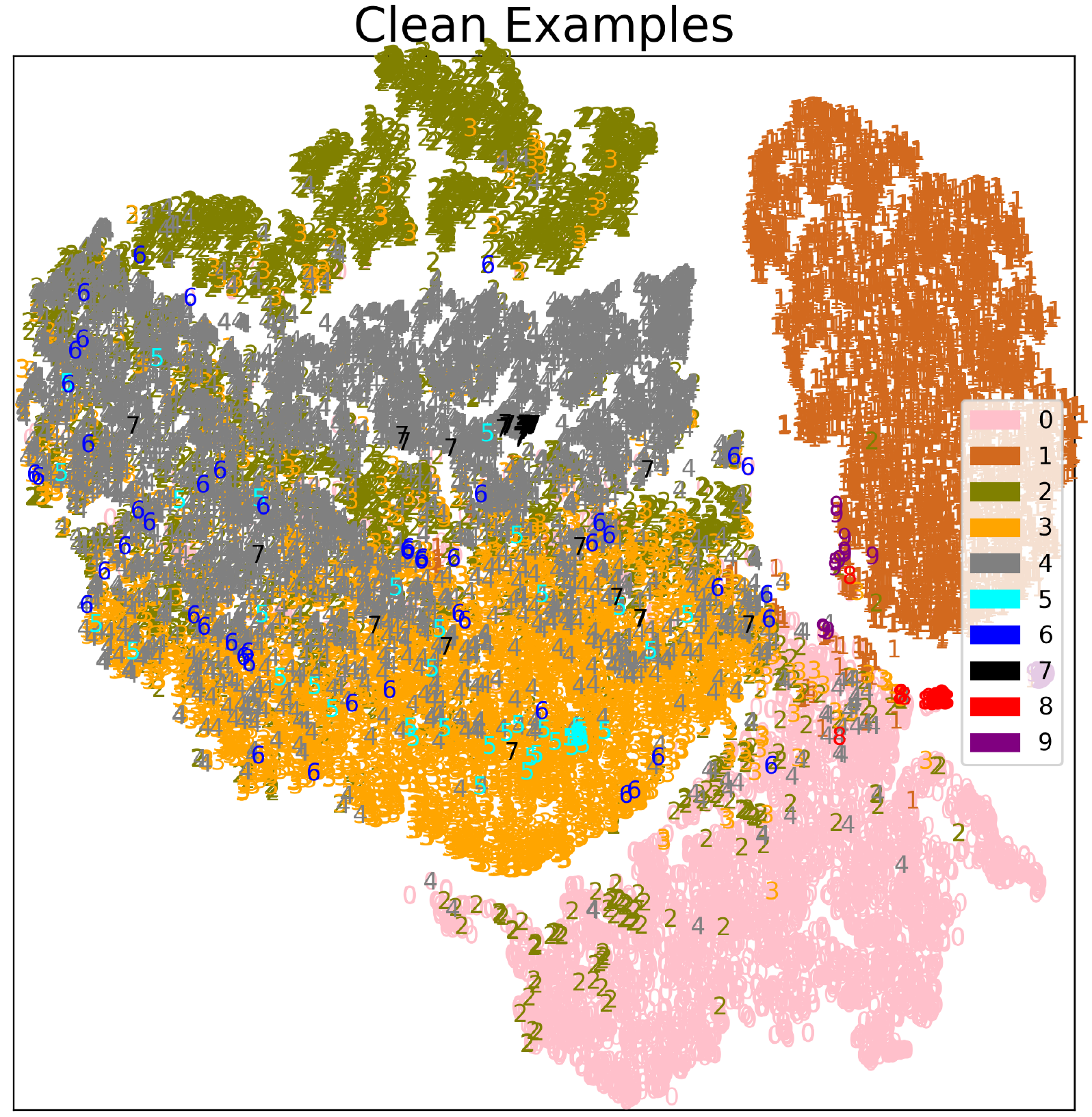}
\caption{SRAT-LDAM.}
\label{fig:tsne_pm}
\end{subfigure}
\caption{t-SNE feature visualization of training examples learned by SRAT and two baseline methods using imbalanced training datasets ``Step-100".}
\label{fig:tsne}
\end{figure}

\textbf{Feature space visualization.} In order to facilitate the reweighting strategy in adversarial training under the imbalanced setting, we present a feature separation loss in our SRAT method. The main goal of the feature separation loss is to enforce the learned feature space as much separated as possible. For checking whether the feature separation loss can work as expected, we apply t-SNE~\cite{van2008visualizing} to visualize the latent feature space learned by our SRAT-LDAM method in Figure~\ref{fig:tsne}. As a comparison, we also provide the visualization of feature space learned by the original PGD adversarial training method (CE) and DRCB-LDAM method.

As shown in Figure~\ref{fig:tsne}, the feature space learned by our SRAT-LDAM method is more separable than two baseline methods. This observation demonstrates that, with our proposed feature separation loss, the adversarially trained model is able to learn much better features and thus our SRAT method can achieve superiority performance.

\textbf{Impact of reweighting values.} As in all SRAT variants, we adopt the Class-balanced method~\cite{cui2019class} to assign different weights to different classes based on their effective number. To explore how the assigned weights impact the performance of our proposed SRAT method, we conduct experiments on a Step-imbalanced CIFAR10 dataset with imbalance ratio $K=100$ to see the change of model's performance using different reweighting values. In our experiments, we assign five well-represented classes with weight 1 and change the weight for remaining five under-represented classes from 10 to 200. The experimental results are shown in Figure~\ref{fig:diff_weight}. Here, we use an approximation integer 78 to denote the weight calculated by the Class-balanced method when the imbalance ratio equals 100. 

From Figure~\ref{fig:diff_weight}, we can obverse that, for all SRAT variants, the model's standard accuracy is increased with the increase of the weights assigning to under-represented classes. However, the robust accuracy for these three methods do not synchronize with the change of their standard accuracy. When increasing the weights for under-represented classes, robust accuracy of SRAT-LDAM is almost unchanged and robust accuracy of SRAT-CE and SRAT-Focal even has slight decrease. As a trade-off, using a relative large weights, such as 78 or 100, in our SRAT method can obtain satisfactory performance on both standard \& robust accuracy, where the former is calculated by the Class-balanced method and the latter equals the imbalance ratio $K$. 

\begin{figure}[t]
\centering
\includegraphics[width=2.2in]{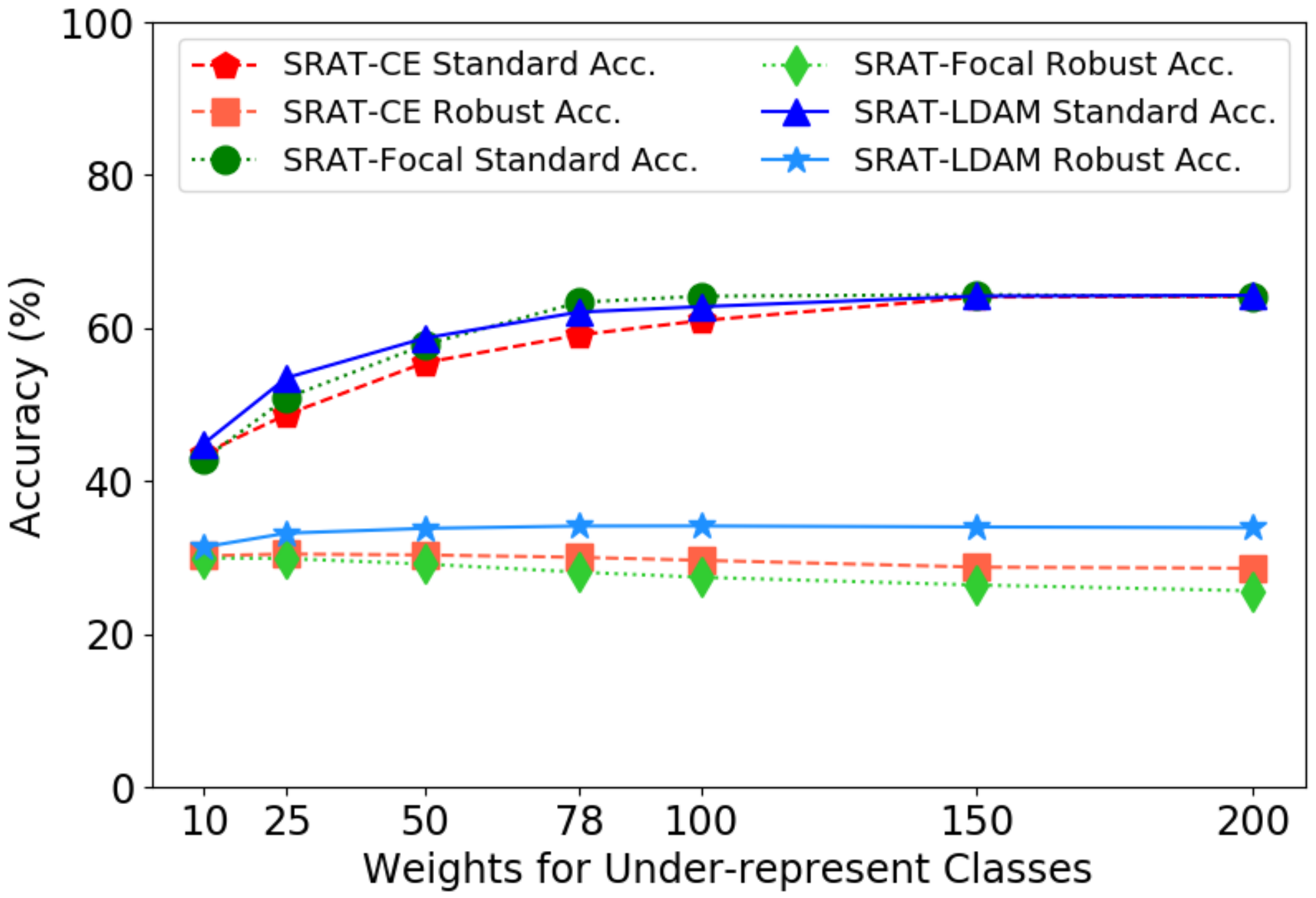}
\caption{The impact of reweighting values using an imbalanced training dataset ``Step-100".}
\label{fig:diff_weight}
\end{figure}

\begin{figure}[t]
\centering
\begin{subfigure}[b]{0.45\textwidth}
\centering
\includegraphics[width=2.2in]{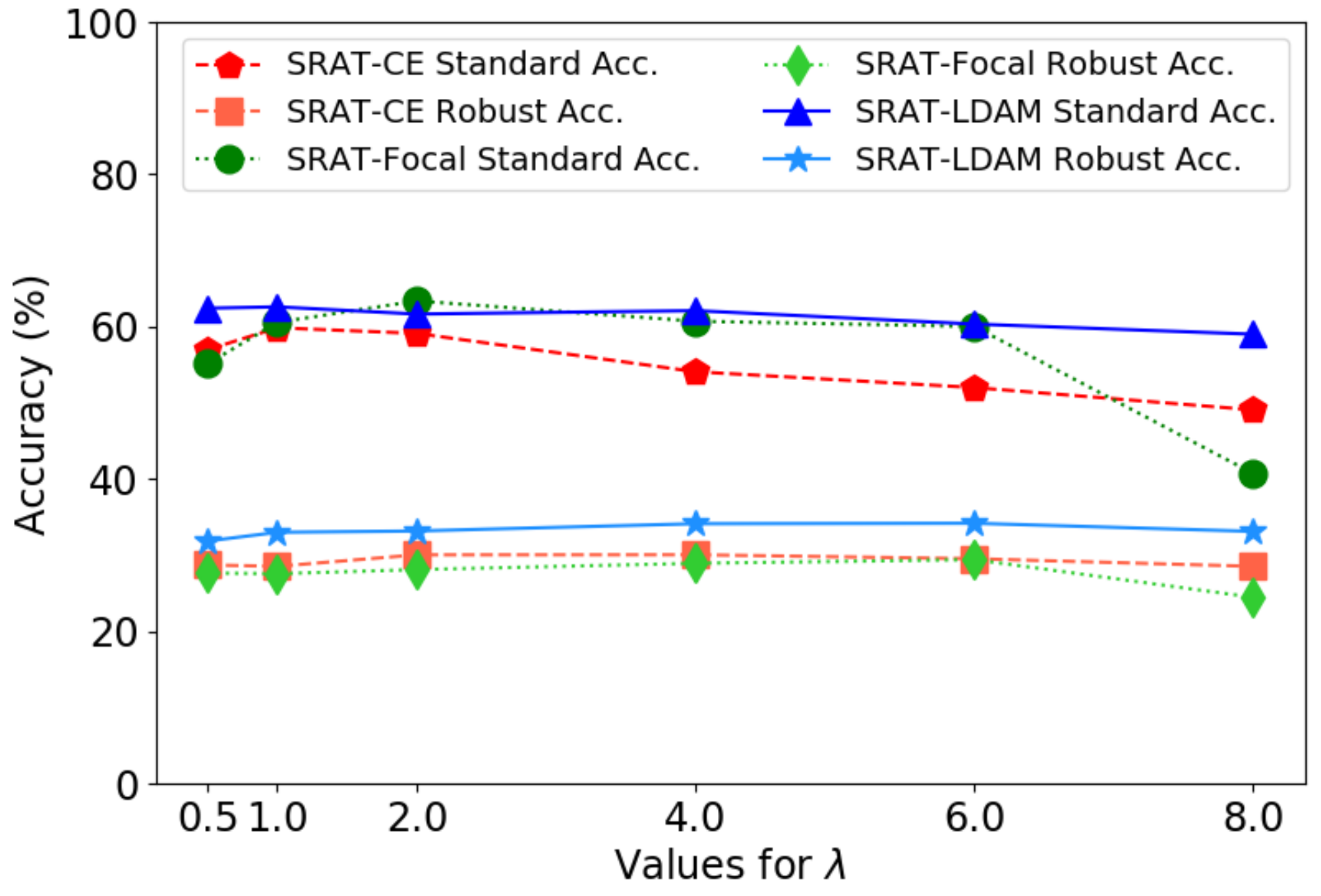}
\caption{Step-100.}
\label{fig:diff_lambda_a}
\end{subfigure}
\hspace{0.2in}
\begin{subfigure}[b]{0.45\textwidth}
\centering
\includegraphics[width=2.2in]{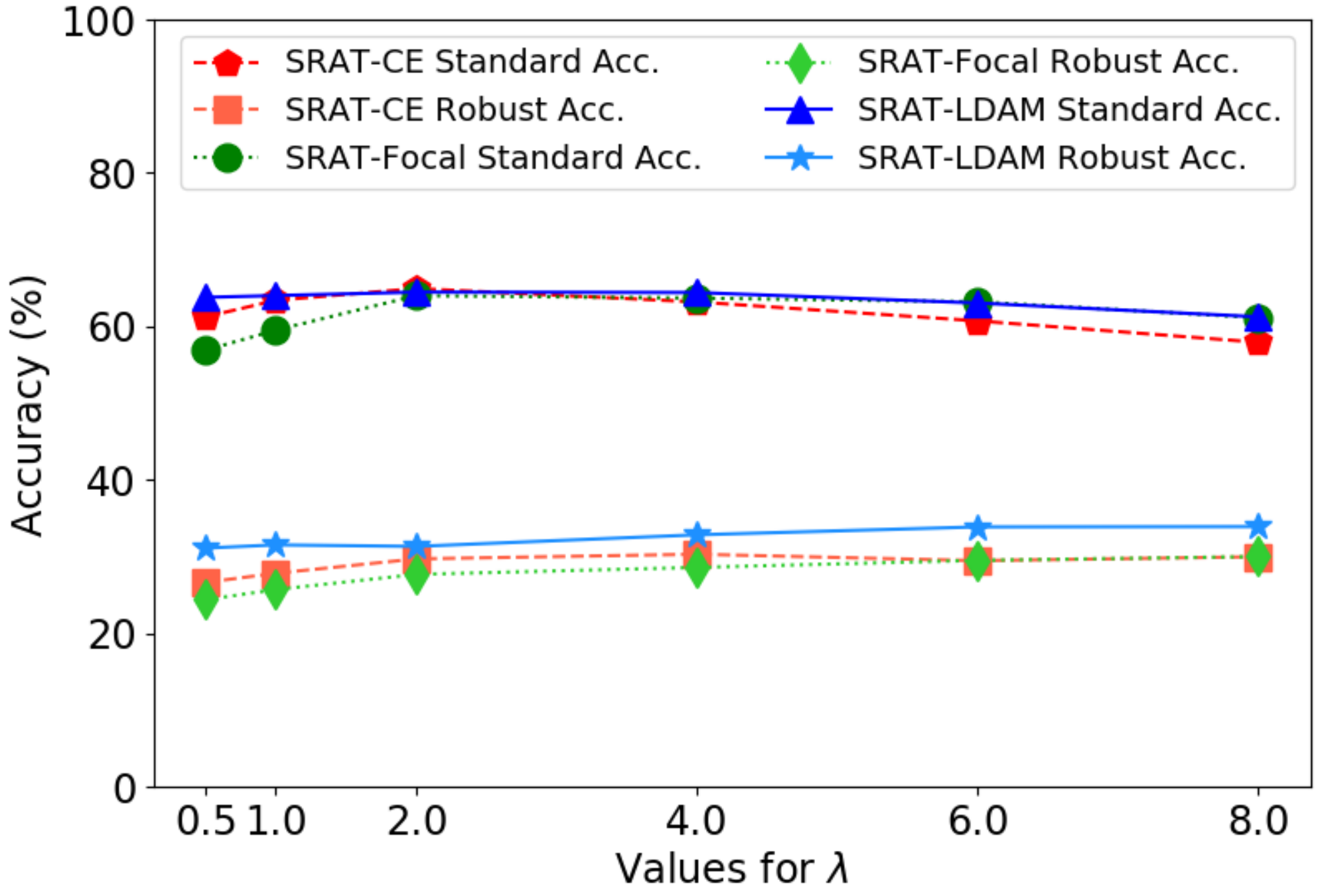}
\caption{Exp-100.}
\label{fig:diff_lambda_b}
\end{subfigure}
\caption{The impact of the hyper-parameter $\lambda$ using imbalanced training datasets ``Step-100" and ``Exp-100".}
\label{fig:diff_lambda}
\end{figure}

\textbf{Impact of hyper-parameter $\lambda$.} In our proposed SRAT method, the contributions of  feature separation loss and prediction loss are controlled by a hyper-parameter $\gamma$. In this part, we study how this hyper-parameter affects the performance of our SRAT method. In our experiments, we evaluate the models' performance of all SRAT variants with different values of $\lambda$ used in training process on both Step-imbalanced CIFAR10 dataset and Exp-imbalanced CIFAR10 dataset with imbalance ratio $K=100$.

As shown in Figure~\ref{fig:diff_lambda}, the performance of all SRAT variants are not very sensitive with the choice of $\lambda$. However, a large value of $\lambda$, such as 8, may hurt the model's performance.

\section{Related Work}\label{sec:related_work}

\textbf{Adversarial Robustness.} The vulnerability of DNN models to adversarial examples has been verified by many existing successful attack methods~\cite{goodfellow2014explaining,carlini2017towards,madry2017towards}. To improve model robustness against adversarial attacks, various defense methods have been proposed~\cite{goodfellow2014explaining,madry2017towards,raghunathan2018certified,cohen2019certified}. Among them, adversarial training has been proven to be one of the most effective defense methods~\cite{athalye2018obfuscated}. Adversarial training can be formulated as solving a min-max optimization problem where the outer minimization process enforces the model to be robust to adversarial examples, generated by the inner maximization process via some existing attacking methods like PGD~\cite{madry2017towards}. Based on adversarial training, several variants, such as TRADES~\cite{zhang2019theoretically}, MART~\cite{wang2019improving} and FAT~\cite{zhang2020attacks}, have been presented to improve the model's performance further. More details about adversarial robustness can be found in recent surveys~\cite{chakraborty2018adversarial,xu2020adversarial}. Since almost all studies of adversarial training are focused on balanced datasets, it's worthwhile to investigate the performance of adversarial training methods on imbalanced training datasets. 

\textbf{Imbalanced Learning.} Most existing works of imbalanced training can be roughly classified into two categories, i.e., re-sampling and reweighting. \emph{Re-sampling} methods aim to reduce the level of imbalance through either over-sampling data examples from under-represented classes~\cite{buda2018systematic,byrd2019effect} or under-sampling data examples from well-represented classes~\cite{japkowicz2002class,drummond2003c4,he2009learning,yen2009cluster}. \emph{reweighting} methods allocate different weights for different classes or even different data examples. For example, Focal loss~\cite{lin2017focal} enlarges the weights of wrongly-classified examples while reducing the weights of well-classified examples in the standard cross entropy loss; and LDAM loss~\cite{cao2019learning} regularizes the under-represented classes more strongly than the over-represented classes to attain good generalization performance on under-represented classes. More information about imbalanced learning can be found in recent surveys~\cite{he2013imbalanced,johnson2019survey}. The majority of existing methods focused on the nature training scenario and their trained models will be crashed when facing adversarial attacks~\cite{szegedy2013intriguing,goodfellow2014explaining}. Hence, in this paper, we develop a novel method that can defend adversarial attacks and achieve well-pleasing performance under the imbalance setting.

\section{Conclusion}\label{sec:conclusion}

In this work, we first empirically investigate the behavior of adversarial training under imbalanced settings and explore the potential solutions to assist adversarial training in tackling the imbalanced issues. As neither adversarial training method itself nor adversarial training with reweighting strategy can work well under imbalanced scenarios, we further  theoretically verify that the poor data separability is one key reason causing the failure of adversarial training based methods under imbalanced scenarios. Based on our findings, we propose the Separable Reweighted Adversarial Training (SRAT) framework to facilitate the reweighting strategy in imbalanced adversarial training by enhancing the separability of learned features. Through extensive experiments, we validate the effectiveness of SRAT. In the future, we plan to examine how other types of defense methods perform under imbalanced scenarios and how other types of balanced learning strategies in natural training behavior under adversarial training. 

\bibliographystyle{plain}
\bibliography{ref}

\appendix

\newpage

\section{Appendix}

\subsection{The Behavior of Adversarial Training}\label{app_sec:pre_com1}

In order to examine the performance of PGD adversarial training under imbalanced scenarios, we adversarially train ResNet18~\cite{he2016deep} models on multiple imbalanced training datasets based on CIFAR10 dataset~\cite{krizhevsky2009learning}. Similar with observations we discussed in Section~\ref{subsec:pre1}, as shown in Figure~\ref{fig:pre_step_10}, Figure~\ref{fig:pre_exp_100} and Figure~\ref{fig:pre_exp_10}, adversarial training produces larger performance gap between well-represented classes and under-represented classes than natural training. Especially, in all imbalanced scenarios, adversarially trained models obtain very low robust accuracy on under-represented classes, which proves again that adversarial training cannot be applied in practical imbalanced scenarios directly.  

\begin{figure}[h]
\begin{subfigure}[b]{0.31\textwidth}
\centering
\includegraphics[width=1.72in]{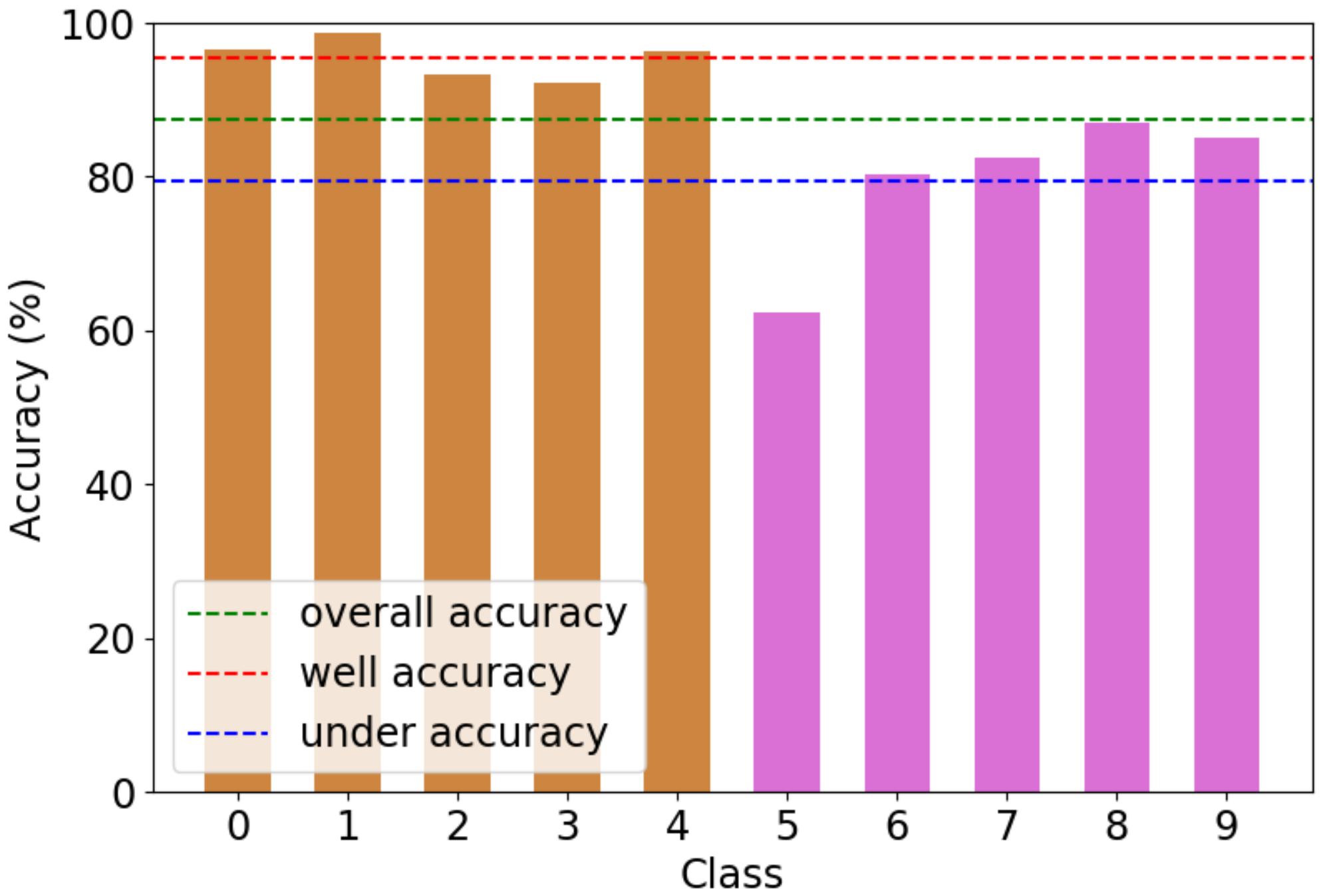}
\caption{Natural Training Standard Acc.}
\label{fig:pre_step_10_nature_standard}
\end{subfigure}
\hspace{0.12in}
\begin{subfigure}[b]{0.31\textwidth}
\centering
\includegraphics[width=1.72in]{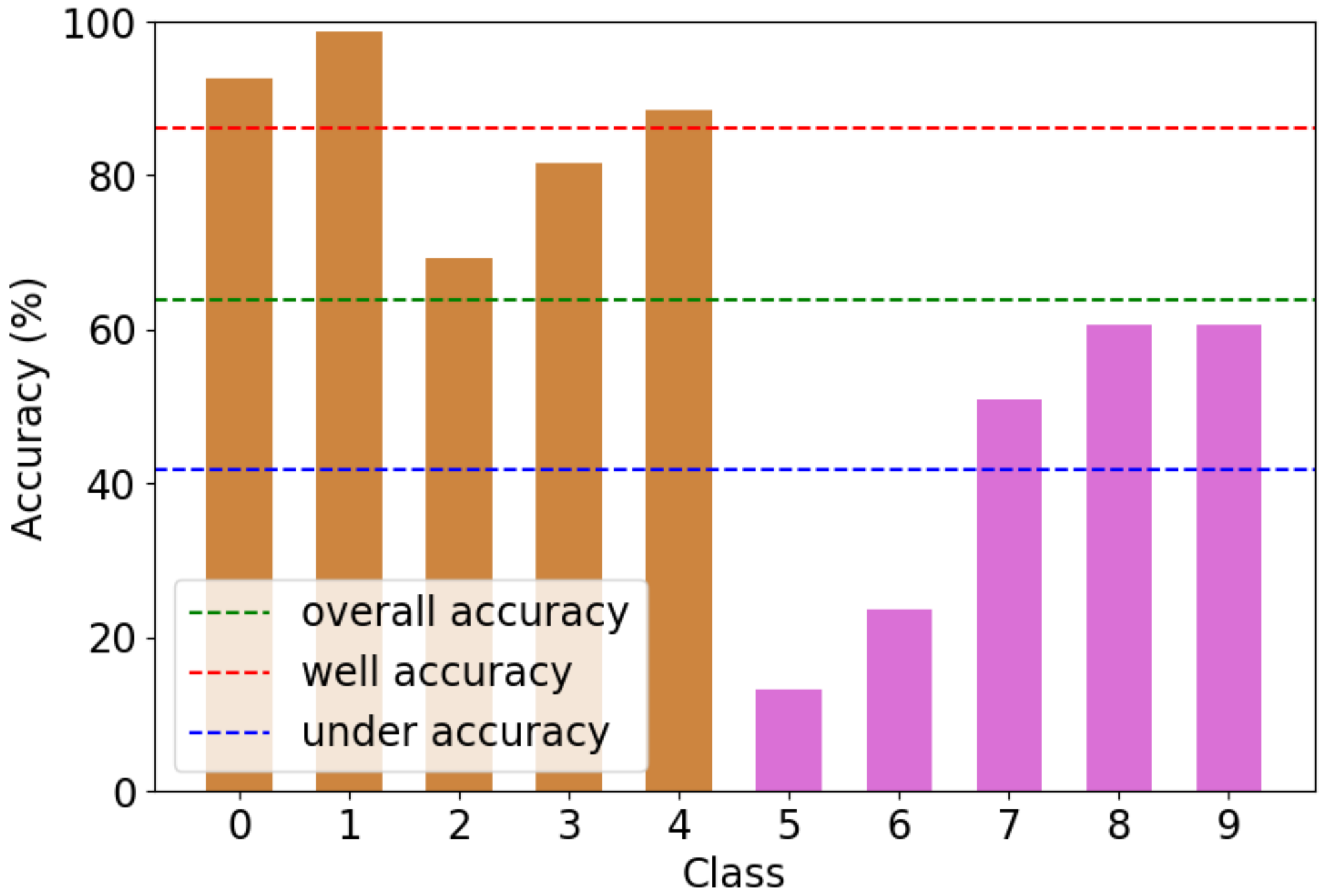}
\caption{Adv. Training Standard Acc.}
\label{fig:pre_step_10_adv_standard}
\end{subfigure}
\hspace{0.12in}
\begin{subfigure}[b]{0.31\textwidth}
\centering
\includegraphics[width=1.72in]{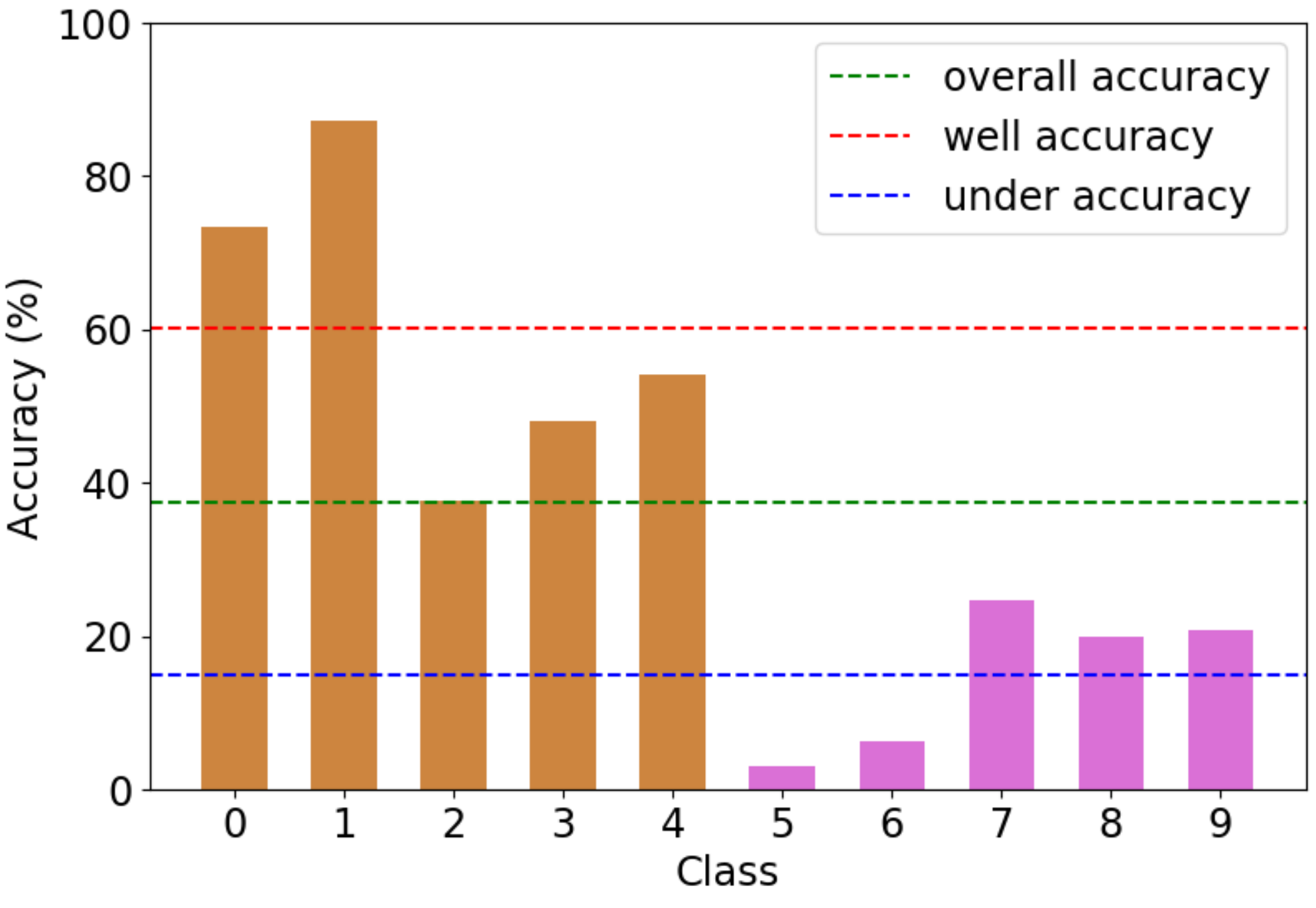}
\caption{Adv. Training Robust Acc.}
\label{fig:pre_step_10_adv_robust}
\end{subfigure}
\caption{Class-wise performance of natural \& adversarial training under an imbalanced CIFAR10 dataset ``Step-10".}
\label{fig:pre_step_10}
\end{figure}

\begin{figure}[h]
\begin{subfigure}[b]{0.31\textwidth}
\centering
\includegraphics[width=1.72in]{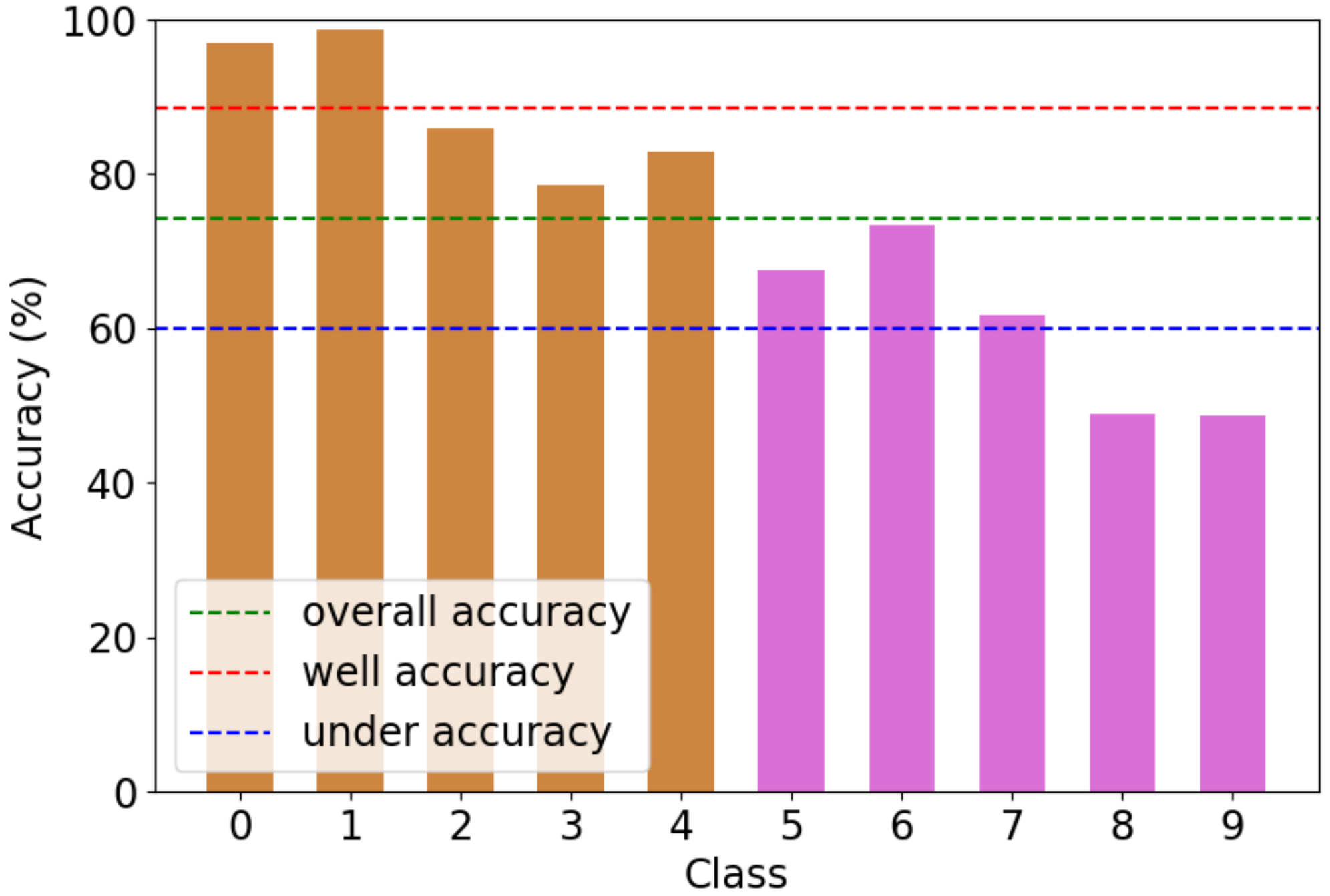}
\caption{Natural Training Standard Acc.}
\label{fig:pre_exp_100_nature_standard}
\end{subfigure}
\hspace{0.12in}
\begin{subfigure}[b]{0.31\textwidth}
\centering
\includegraphics[width=1.72in]{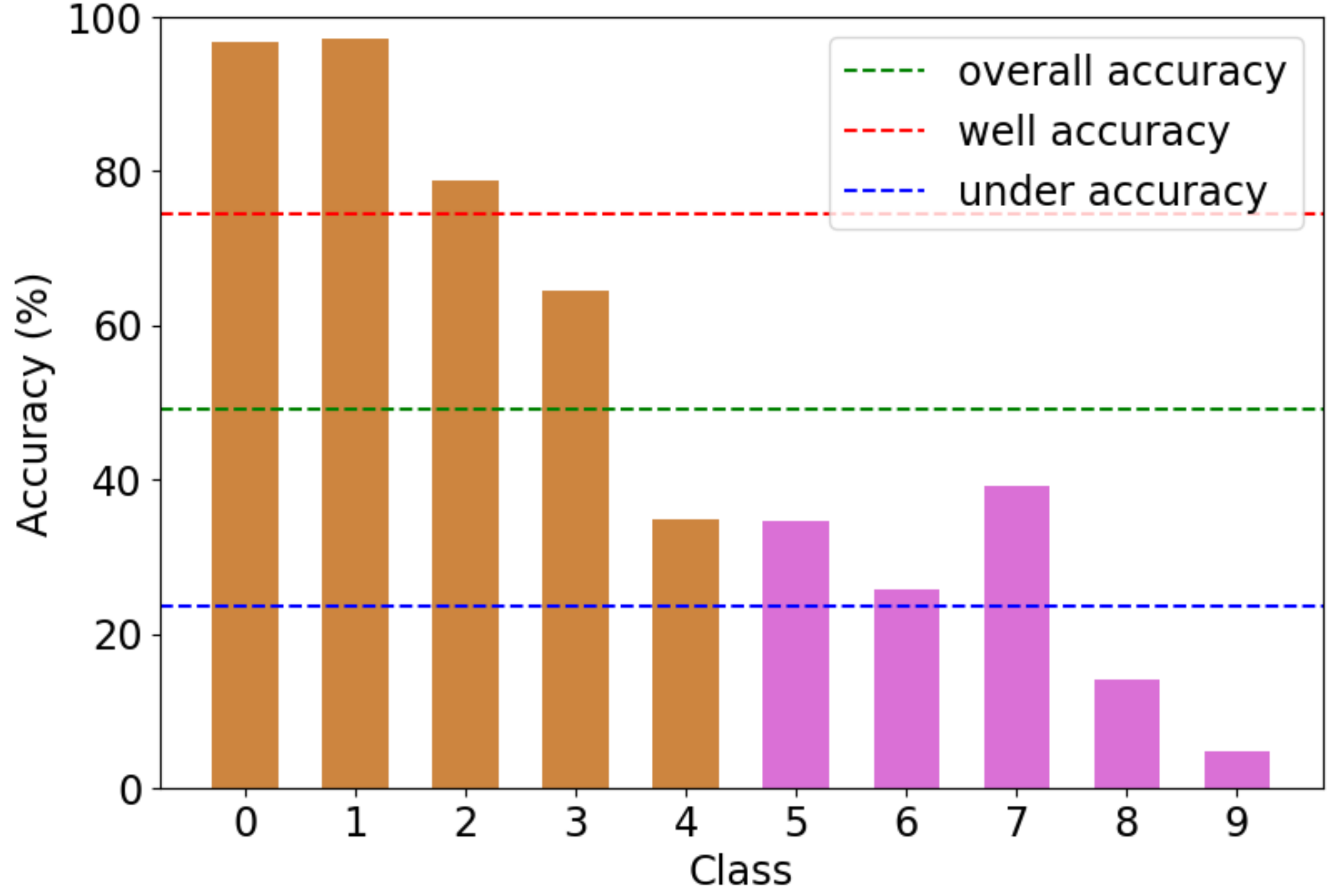}
\caption{Adv. Training Standard Acc.}
\label{fig:pre_exp_100_adv_standard}
\end{subfigure}
\hspace{0.12in}
\begin{subfigure}[b]{0.31\textwidth}
\centering
\includegraphics[width=1.72in]{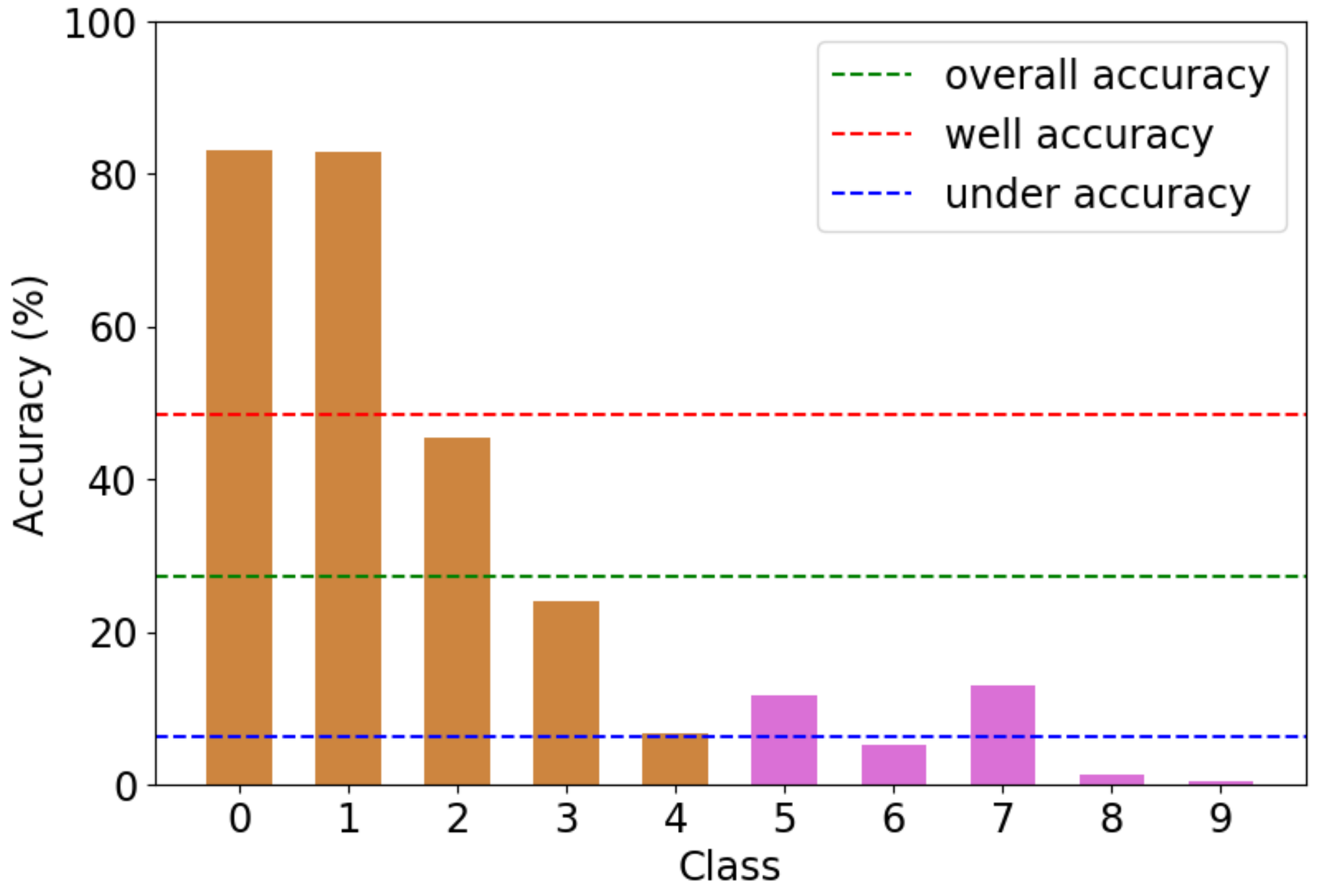}
\caption{Adv. Training Robust Acc.}
\label{fig:pre_exp_100_adv_robust}
\end{subfigure}
\caption{Class-wise performance of natural \& adversarial training under an imbalanced CIFAR10 dataset ``Exp-100".}
\label{fig:pre_exp_100}
\end{figure}

\begin{figure}[h]
\begin{subfigure}[b]{0.31\textwidth}
\centering
\includegraphics[width=1.72in]{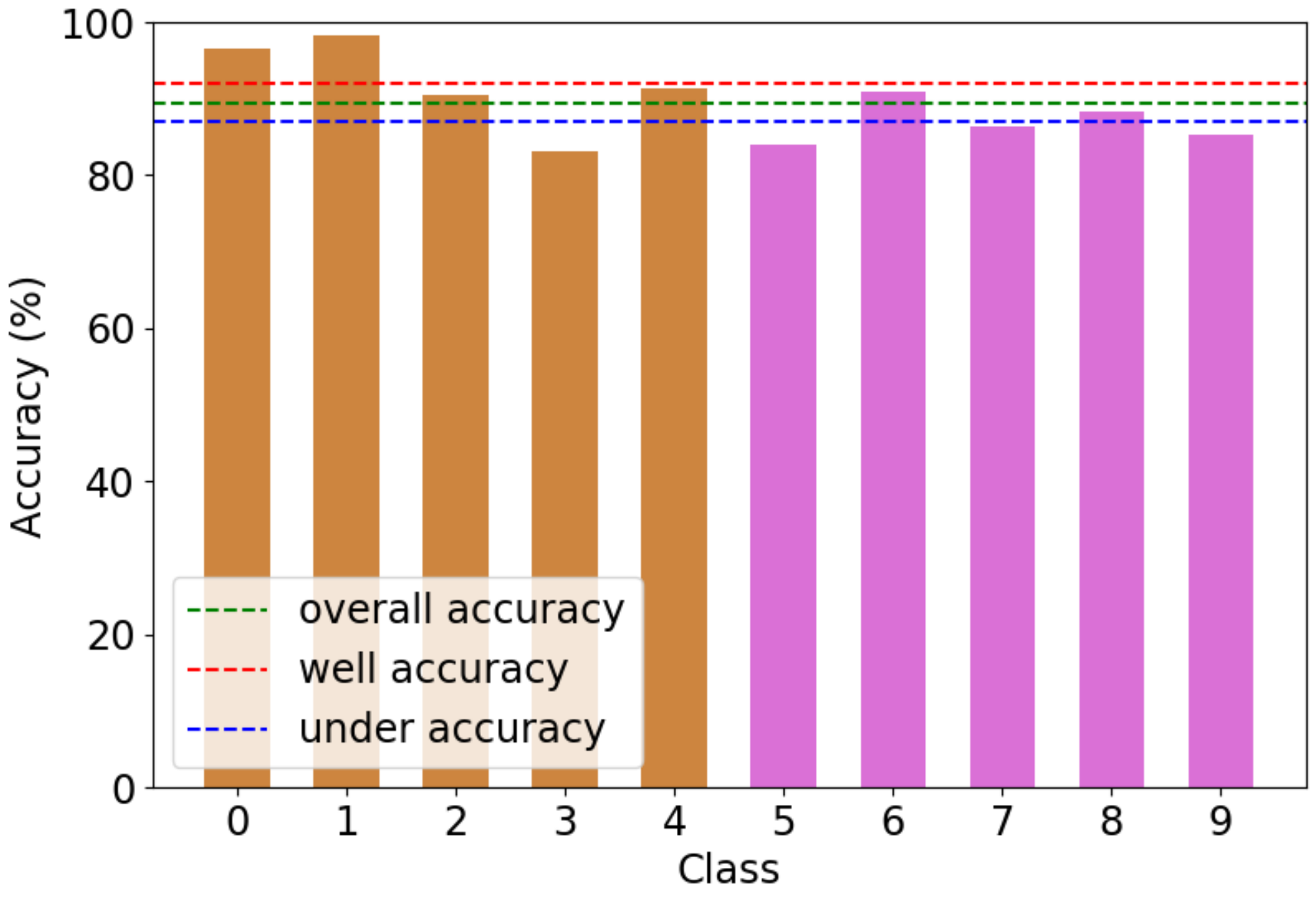}
\caption{Natural Training Standard Acc.}
\label{fig:pre_exp_10_nature_standard}
\end{subfigure}
\hspace{0.12in}
\begin{subfigure}[b]{0.31\textwidth}
\centering
\includegraphics[width=1.72in]{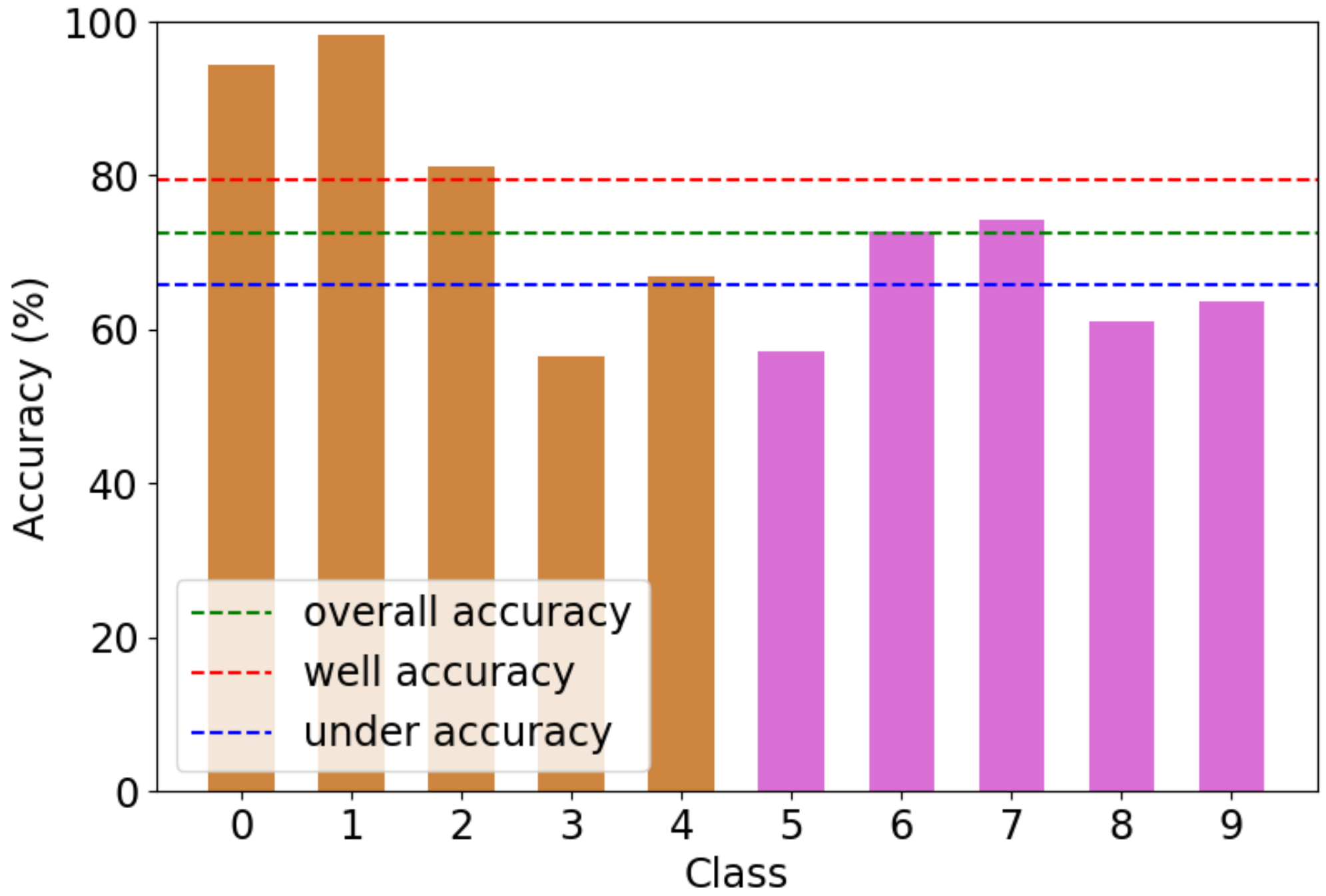}
\caption{Adv. Training Standard Acc.}
\label{fig:pre_exp_10_adv_standard}
\end{subfigure}
\hspace{0.12in}
\begin{subfigure}[b]{0.31\textwidth}
\centering
\includegraphics[width=1.72in]{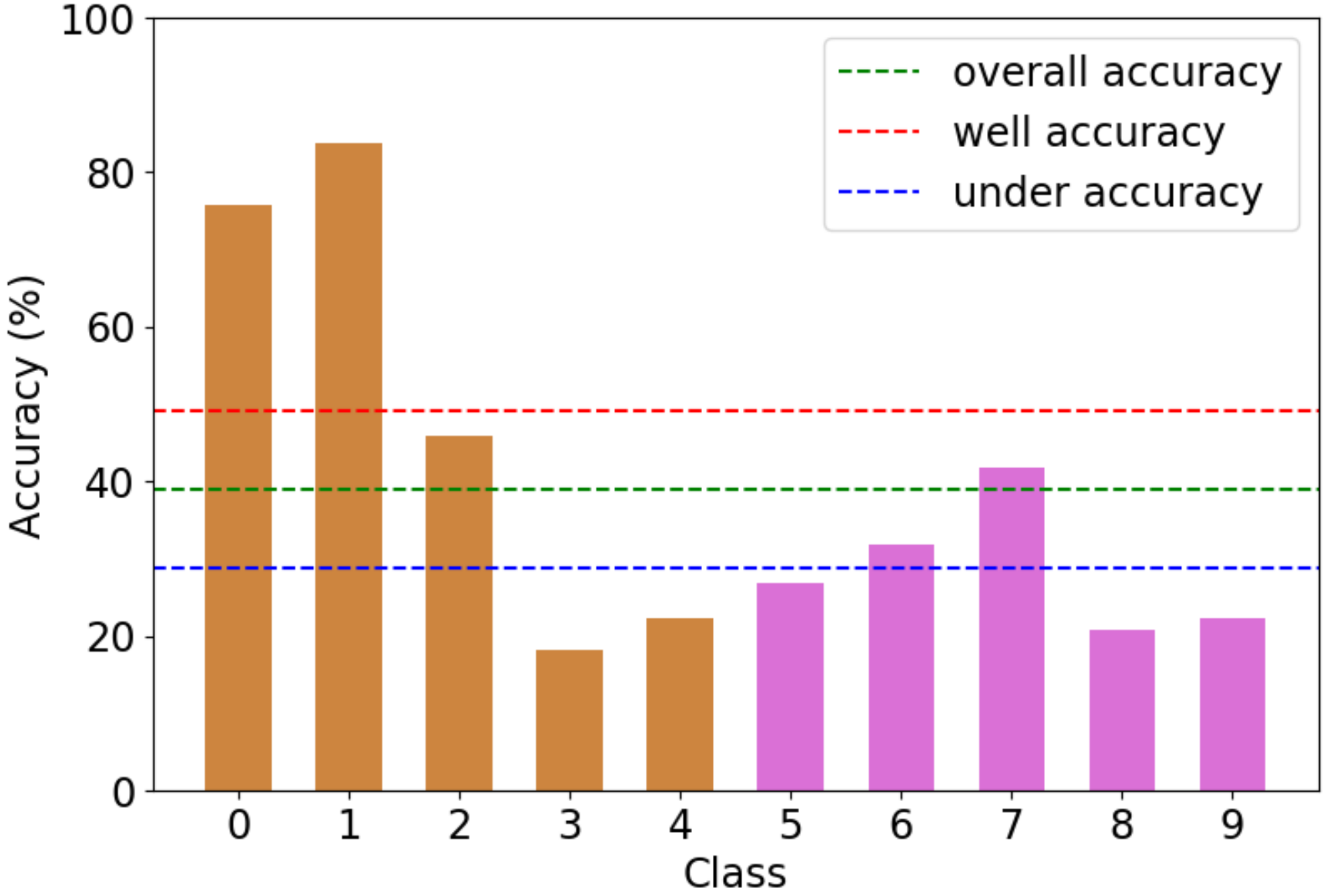}
\caption{Adv. Training Robust Acc.}
\label{fig:pre_exp_10_adv_robust}
\end{subfigure}
\caption{Class-wise performance of natural \& adversarial training under an imbalanced CIFAR10 dataset ``Exp-10".}
\label{fig:pre_exp_10}
\end{figure}

\subsection{Reweighting Strategy in Natural Training v.s. in Adversarial Training}\label{app_sec:pre_com2}

For exploring whether the reweighting strategy can help adversarial training deal with imbalanced issues, we evaluate performance of adversarial trained models using diverse binary imbalanced training datasets with different weights assigning to under-represented class. As shown in Figure~\ref{fig:pre_binary_1_9}, Figure~\ref{fig:pre_binary_2_6}, Figure~\ref{fig:pre_binary_5_4}, for adversarially trained models, increasing the weights assigning to under-represented class will improve models' performance on under-represented class. However, as the same time, the models' performance on well-represented class will be drastically decreased. As a comparison, adopting larger weights in naturally trained models will also improve models' performance on under-represented class but only result in slight drop in performance on well-represented class. In other words, the reweighting strategy proposed in natural training to handle imbalanced problem may only provide limited help in adversarial training, and, hence, new techniques are needed for adversarial training under imbalanced scenarios.

\begin{figure}[h]
\begin{subfigure}[b]{0.31\textwidth}
\centering
\includegraphics[width=1.72in]{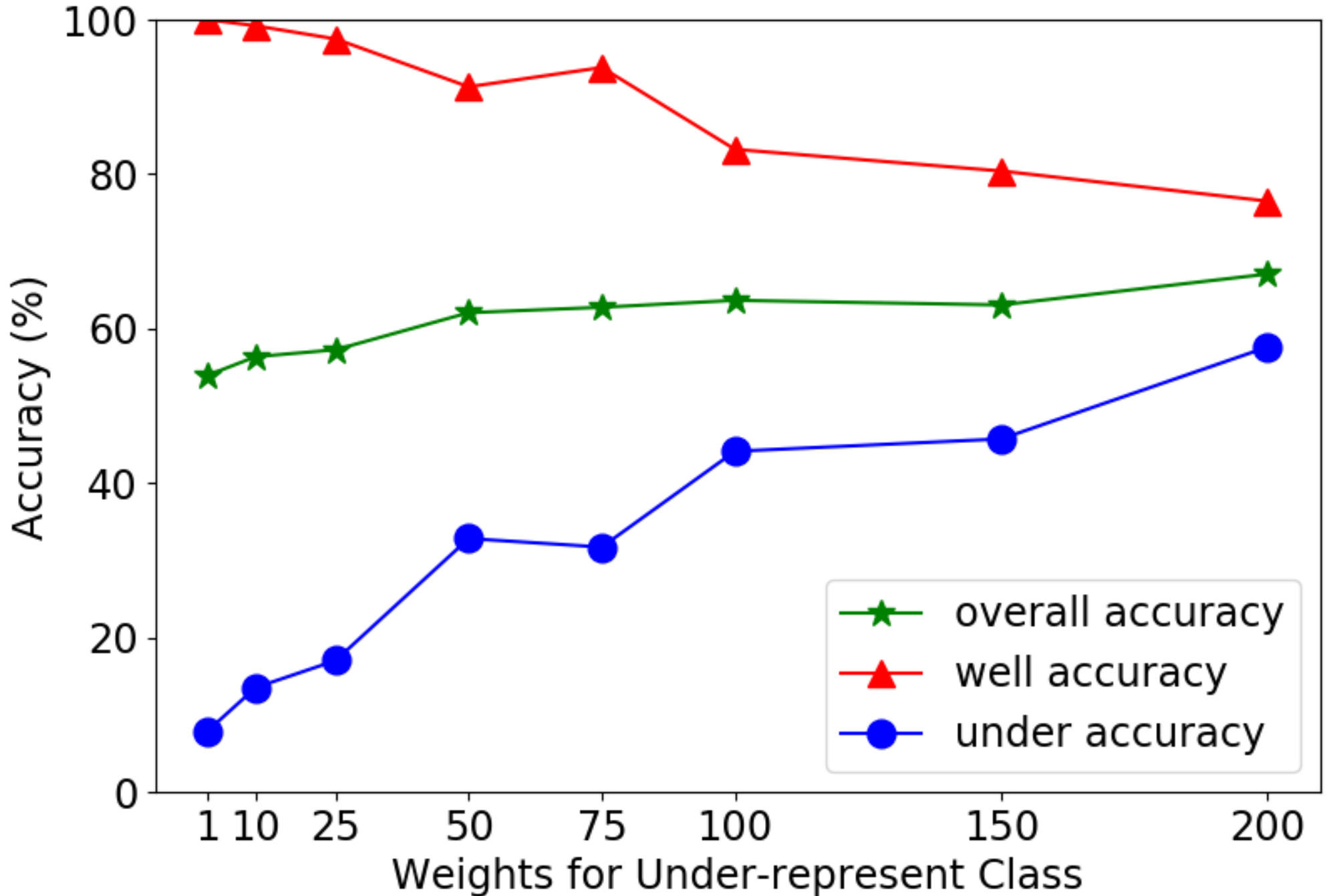}
\caption{Natural Training Standard Acc.}
\label{fig:pre_binary_1_9_nature_standard}
\end{subfigure}
\hspace{0.12in}
\begin{subfigure}[b]{0.3\textwidth}
\centering
\includegraphics[width=1.72in]{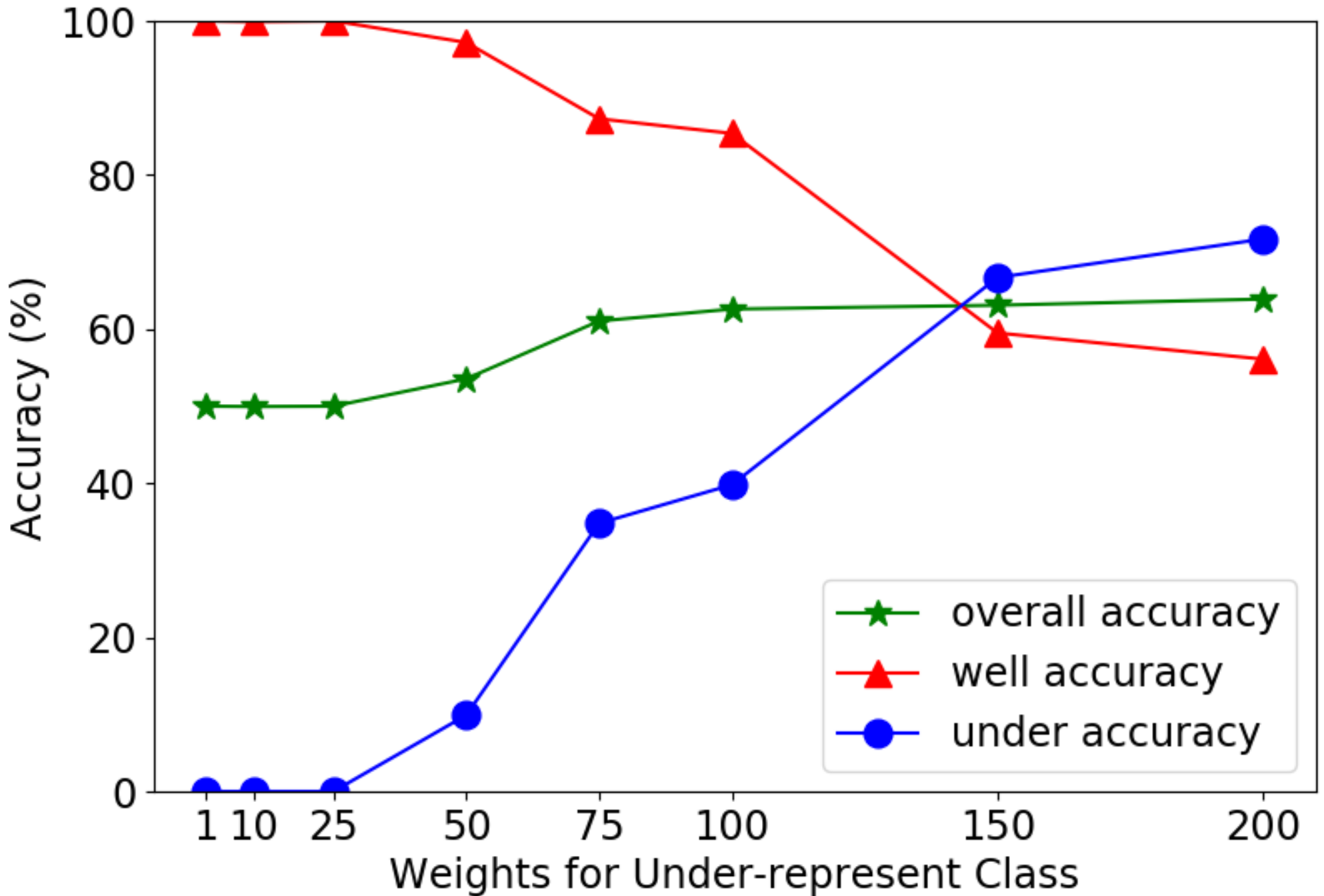}
\caption{Adv. Training Standard Acc.}
\label{fig:pre_binary_1_9_adv_standard}
\end{subfigure}
\hspace{0.12in}
\begin{subfigure}[b]{0.3\textwidth}
\centering
\includegraphics[width=1.72in]{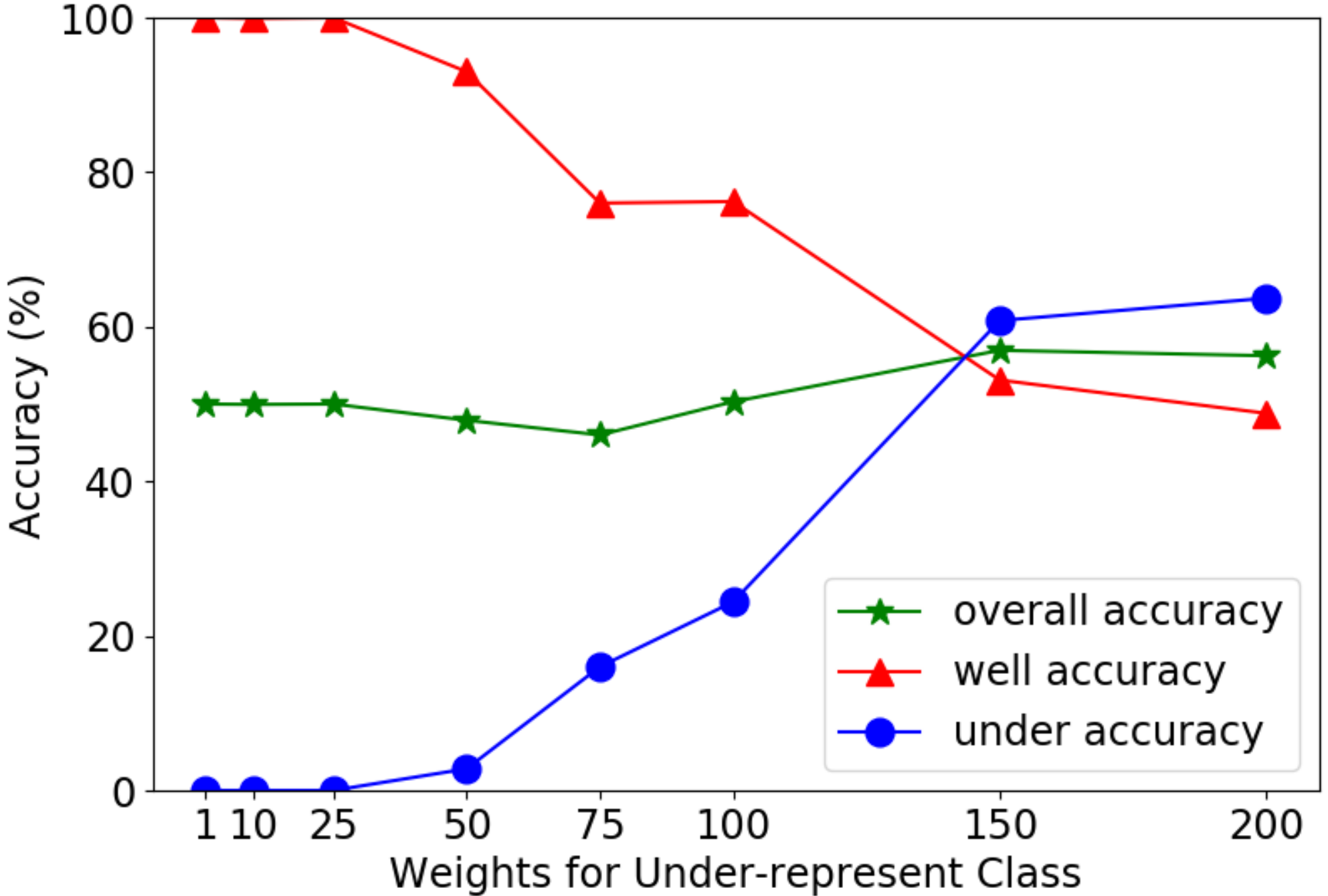}
\caption{Adv. Training Robust Acc.}
\label{fig:pre_binary_1_9_adv_robust}
\end{subfigure}
\caption{Class-wise performance of reweighted natural \& adversarial training in binary classification. (``auto'' as well-represented class and ``truck'' as under-represented class).}
\label{fig:pre_binary_1_9}
\end{figure}

\begin{figure}[h]
\begin{subfigure}[b]{0.31\textwidth}
\centering
\includegraphics[width=1.72in]{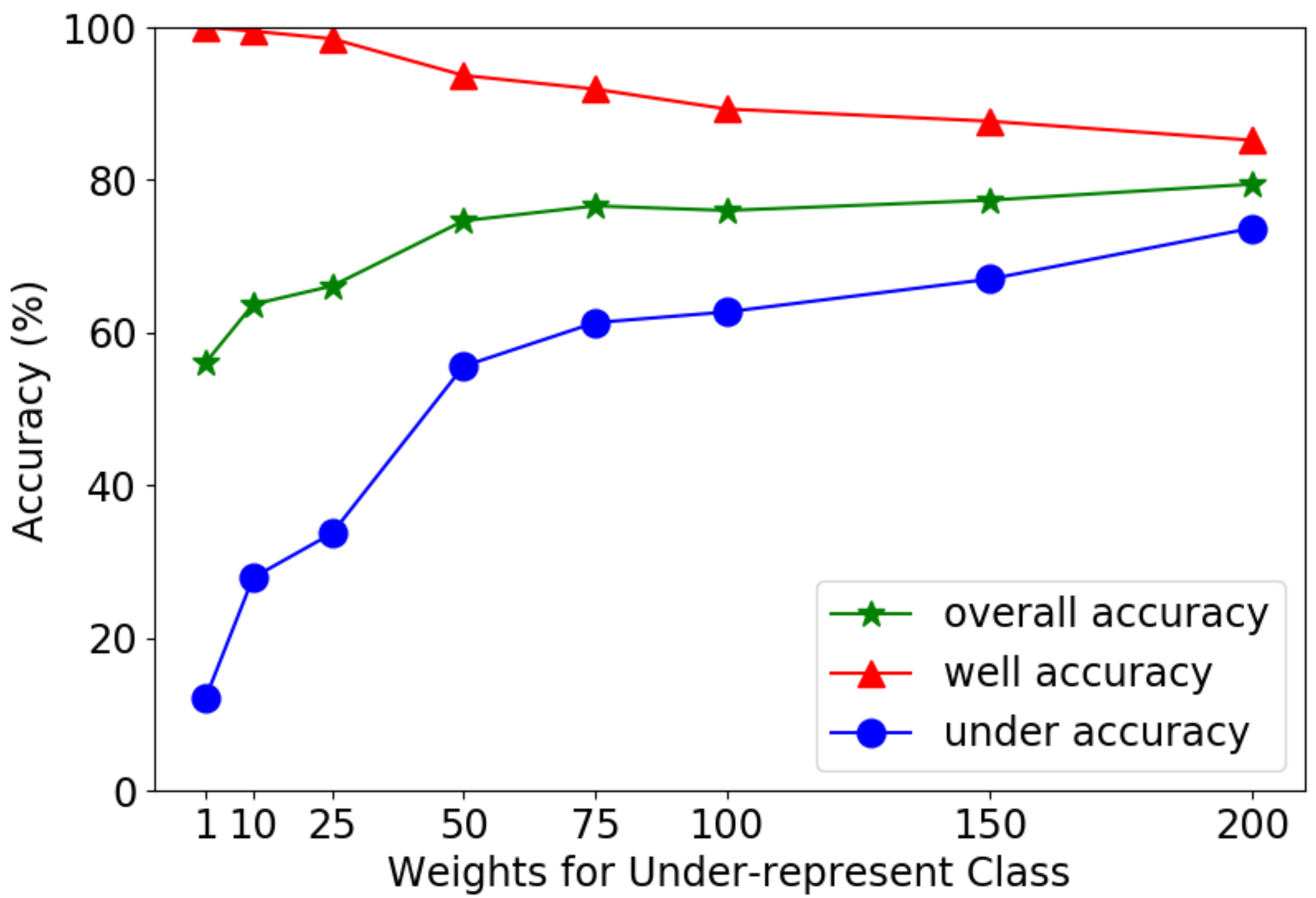}
\caption{Natural Training Standard Acc.}
\label{fig:pre_binary_2_6_nature_standard}
\end{subfigure}
\hspace{0.12in}
\begin{subfigure}[b]{0.3\textwidth}
\centering
\includegraphics[width=1.72in]{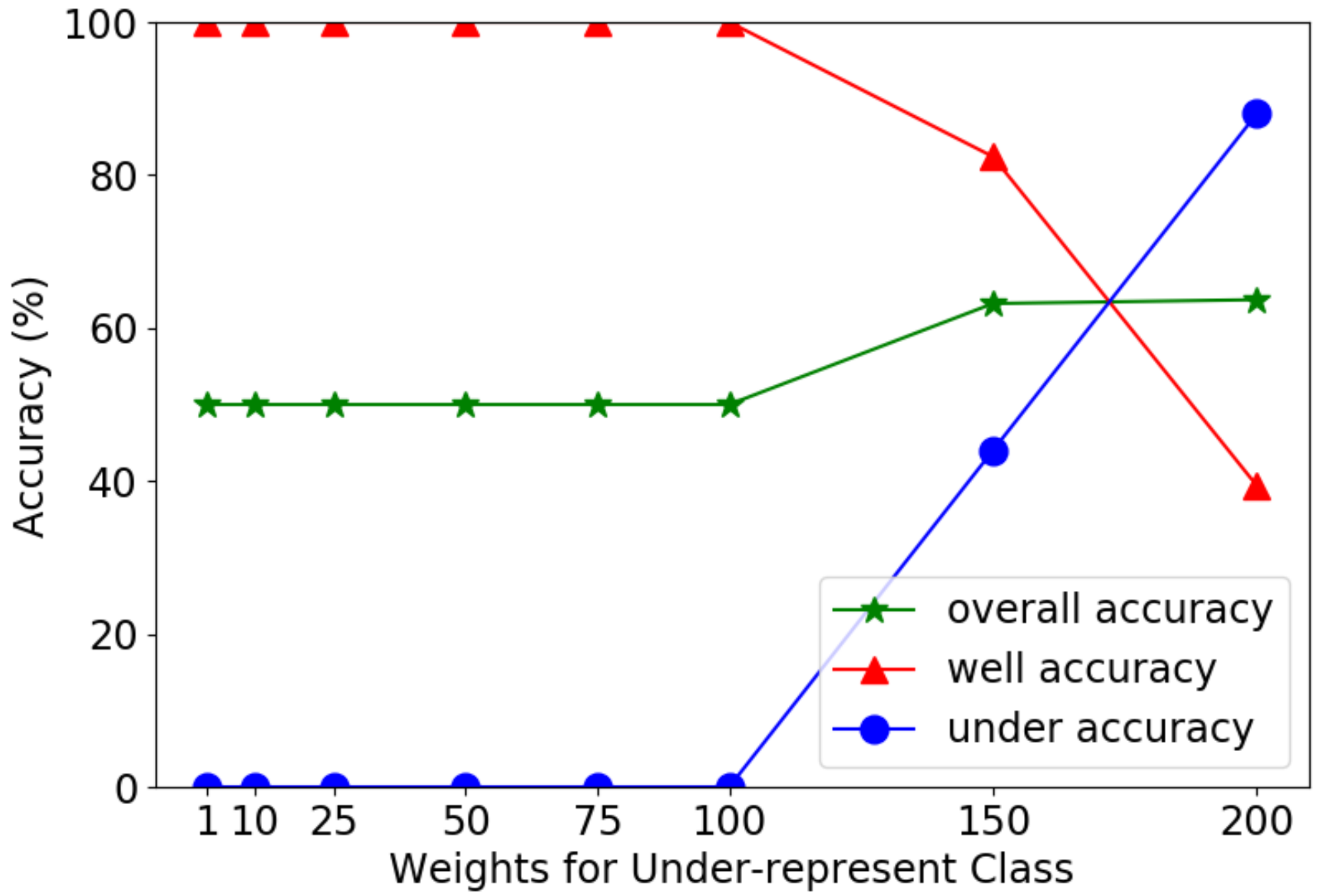}
\caption{Adv. Training Standard Acc.}
\label{fig:pre_binary_2_6_adv_standard}
\end{subfigure}
\hspace{0.12in}
\begin{subfigure}[b]{0.3\textwidth}
\centering
\includegraphics[width=1.72in]{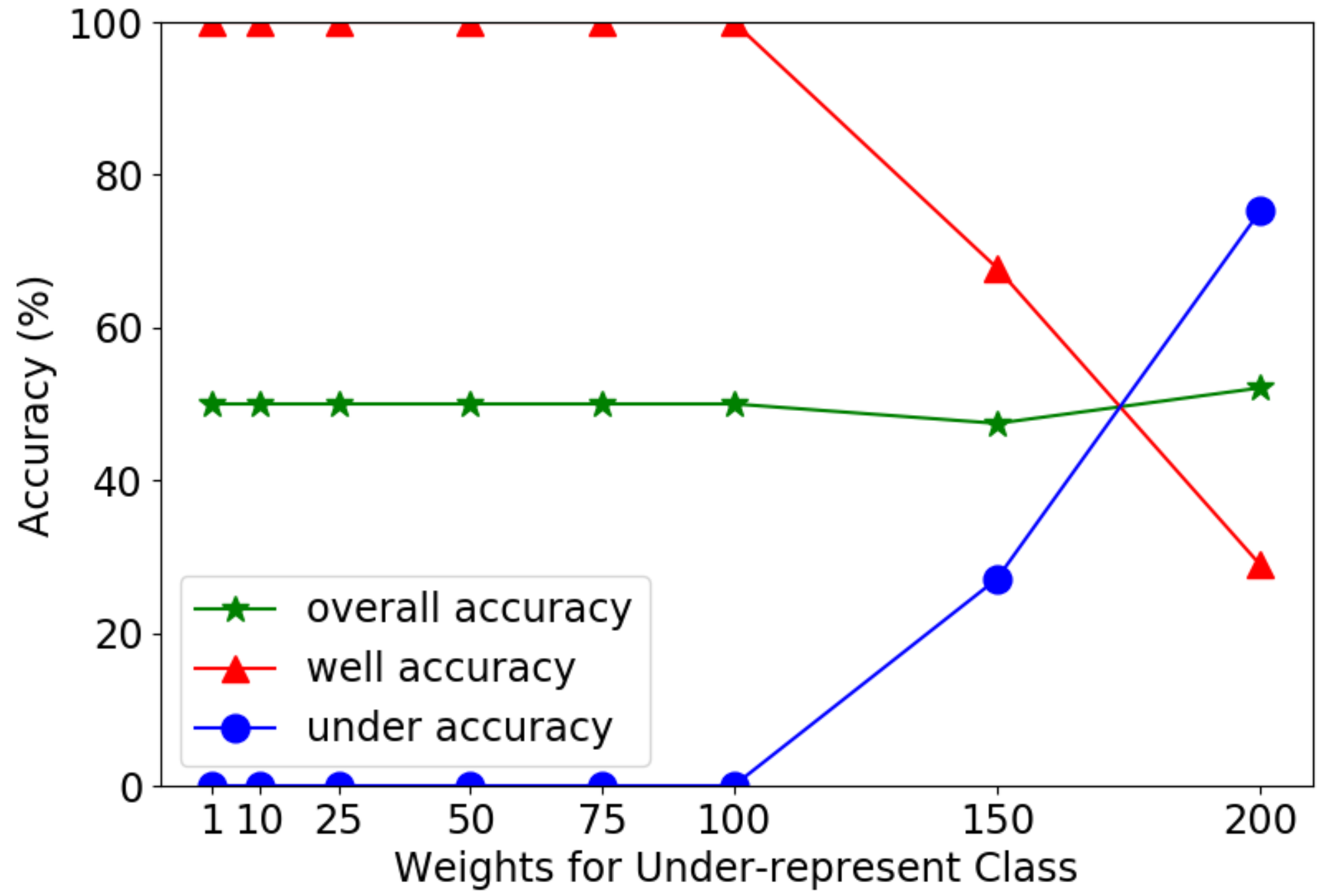}
\caption{Adv. Training Robust Acc.}
\label{fig:pre_binary_2_6_adv_robust}
\end{subfigure}
\caption{Class-wise performance of reweighted natural \& adversarial training in binary classification. (``bird'' as well-represented class and ``frog'' as under-represented class).}
\label{fig:pre_binary_2_6}
\end{figure}

\begin{figure}[h]
\begin{subfigure}[b]{0.31\textwidth}
\centering
\includegraphics[width=1.72in]{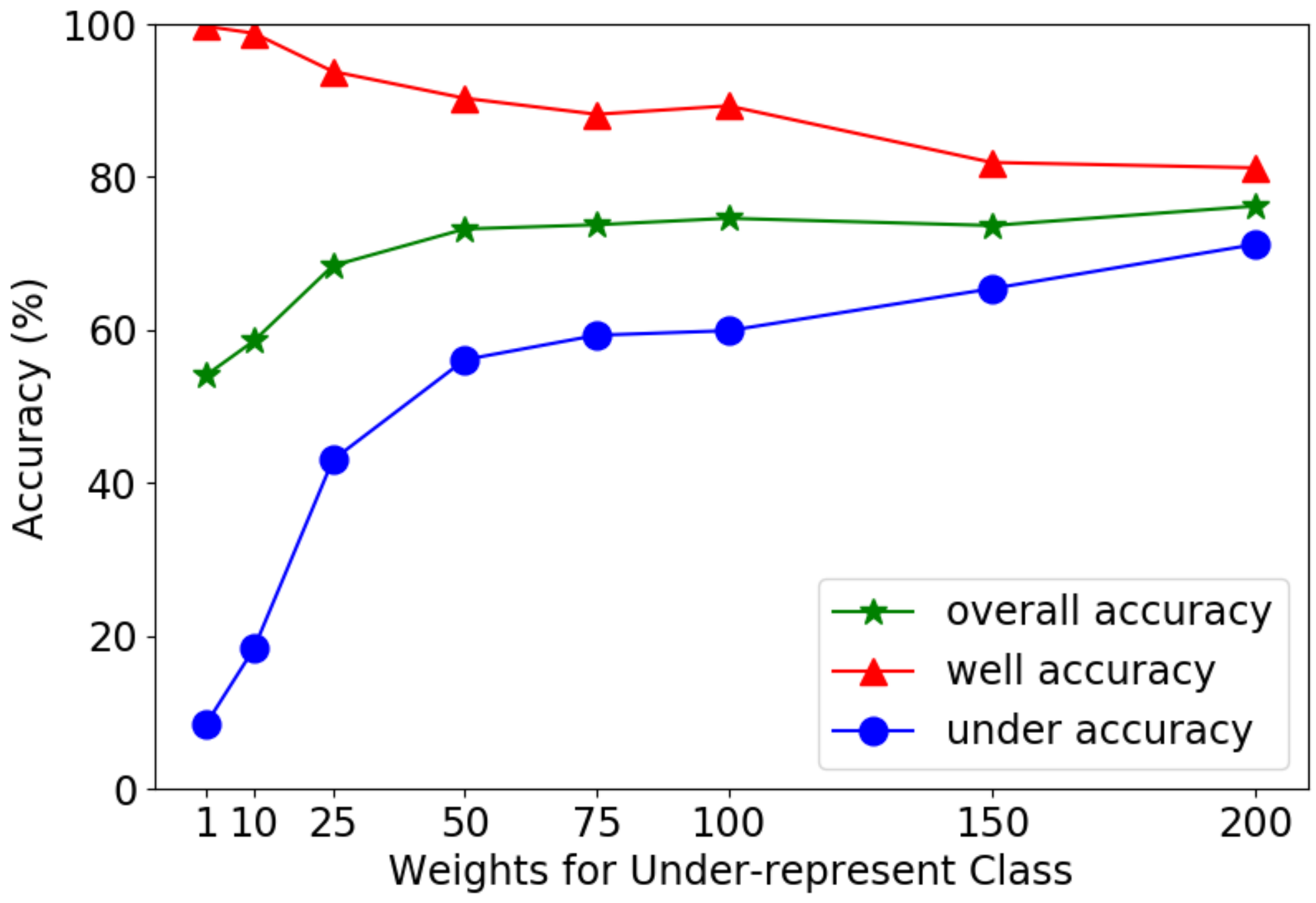}
\caption{Natural Training Standard Acc.}
\label{fig:pre_binary_5_4_nature_standard}
\end{subfigure}
\hspace{0.12in}
\begin{subfigure}[b]{0.3\textwidth}
\centering
\includegraphics[width=1.72in]{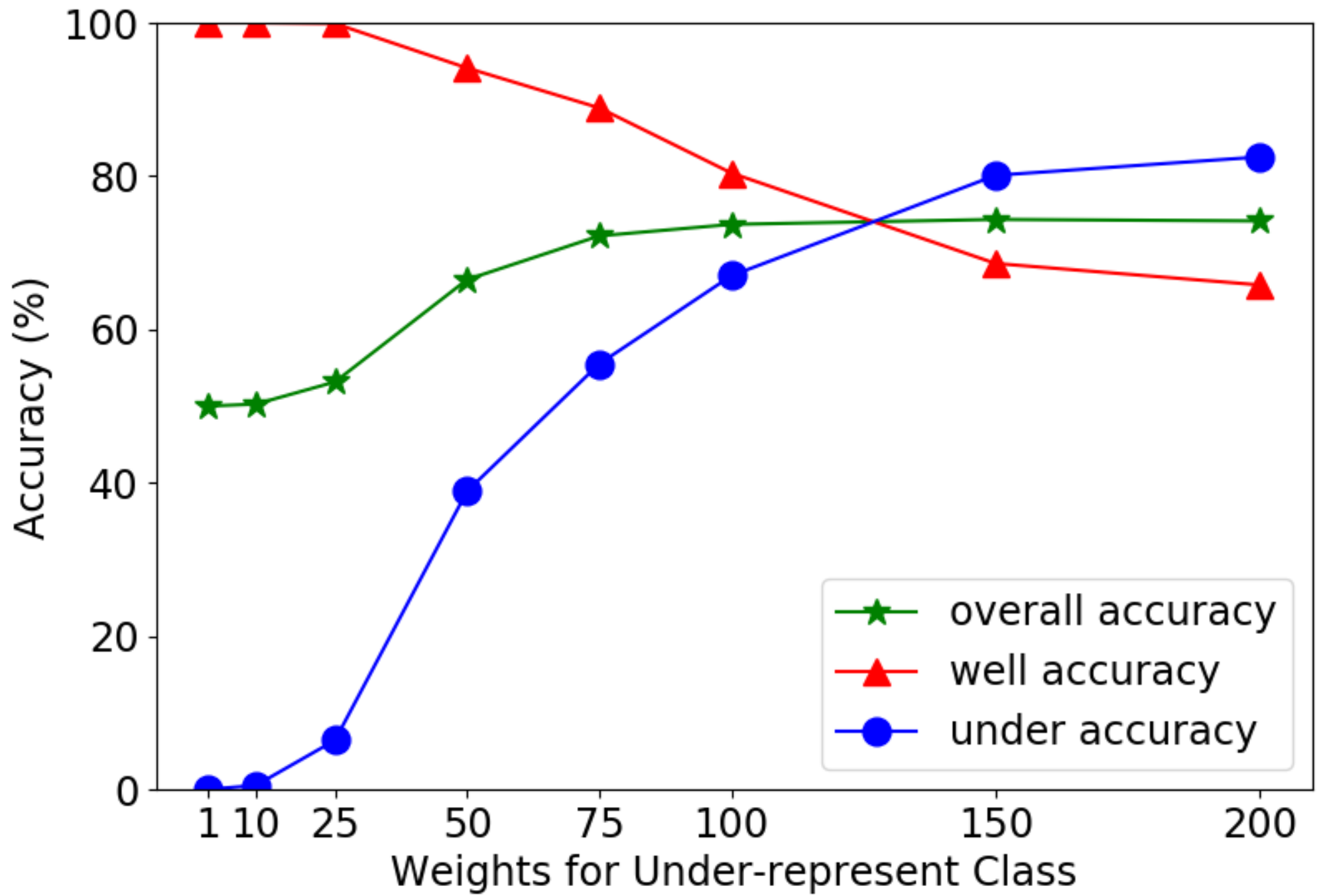}
\caption{Adv. Training Standard Acc.}
\label{fig:pre_binary_5_4_adv_standard}
\end{subfigure}
\hspace{0.12in}
\begin{subfigure}[b]{0.3\textwidth}
\centering
\includegraphics[width=1.72in]{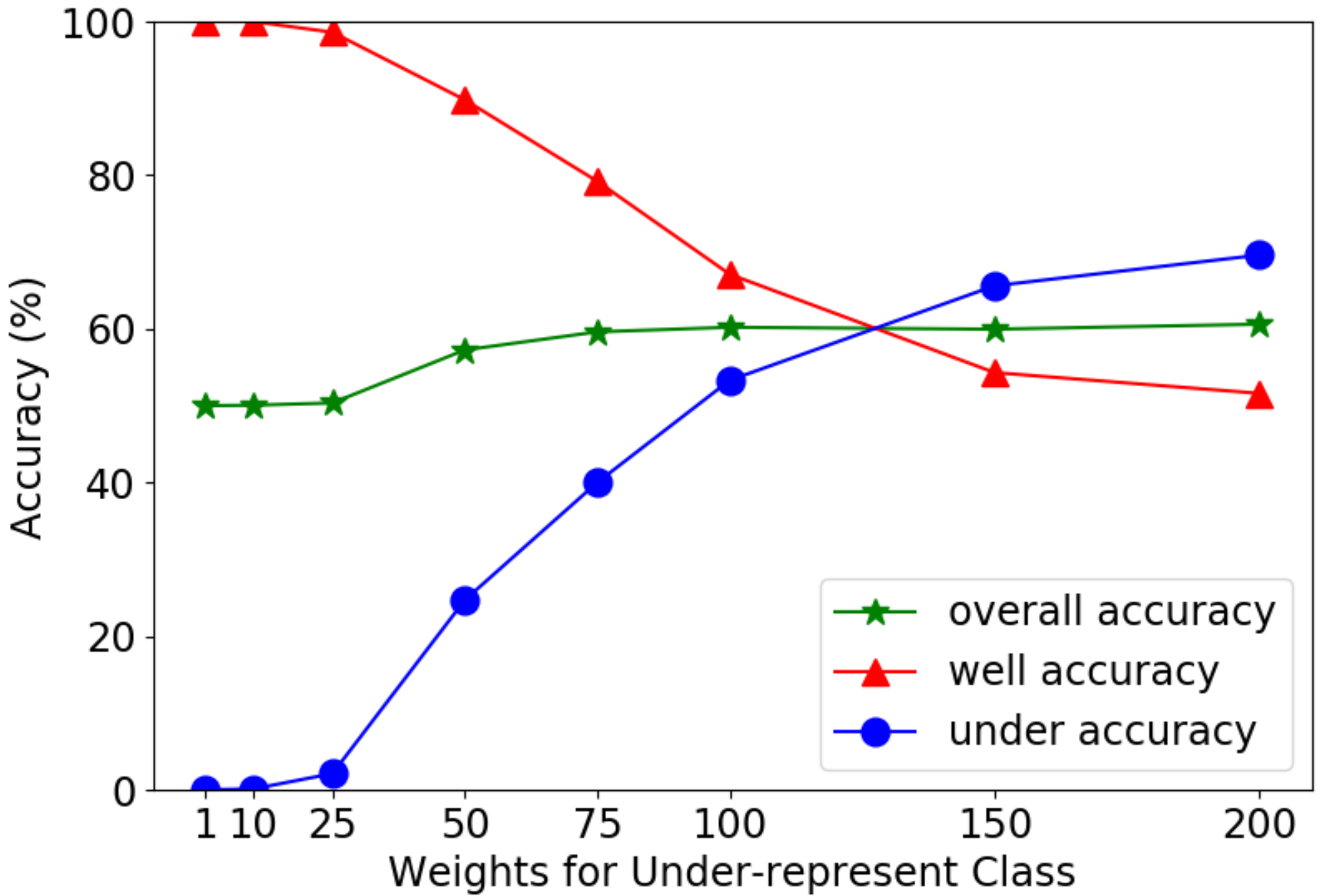}
\caption{Adv. Training Robust Acc.}
\label{fig:pre_binary_5_4_adv_robust}
\end{subfigure}
\caption{Class-wise performance of reweighted natural \& adversarial training in binary classification. (``dog'' as well-represented class and ``deer'' as under-represented class).}
\label{fig:pre_binary_5_4}
\end{figure}


\subsection{Proofs of the Theorems in Section~\ref{sec:theory}}

\subsubsection{Proof of Lemma~\ref{lemma1}}\label{app_sec:lemma1}

\lemm*
\begin{proof}[Proof of Lemma~\ref{lemma1}] We will first prove that the optimal model $f^*$ has parameters $w_1=w_2=\dots=w_d$ (or $w = \textbf{1}$) by contradiction. We define $G = \{1,2,\dots, d\}$ and make the following assumption: for the optimal $w$ and $b$, we assume if there exist $w_i < w_j$ for $i\neq j$ and $i,j \in G$. Then we obtain the following standard errors for two classes of this classifier $f$ with weight $w$:
\begin{align}
\small
\label{eq:weigh1}
\begin{split}
\text{Pr.}(f^*(x)\neq y | y = -1) &= \text{Pr.}\{\sum_{k\neq i, k\neq j} w_k \gN(-\eta, \sigma^2)  -b + w_i \gN(-\eta, \sigma^2) + w_j \gN(-\eta, \sigma^2) > 0 \}, \\
\text{Pr.}(f^*(x)\neq y | y = +1) &= \text{Pr.}\{\sum_{k\neq i, k\neq j} w_k \gN(+\eta, \sigma^2) -b+ w_i \gN(+\eta, \sigma^2) + w_j \gN(+\eta, \sigma^2) < 0 \}.
\end{split}
\end{align}
However, if we define a new classier $\tilde{f}$ whose weight $\tilde{w}$ uses $w_j$ to replace $w_i$, we obtain the errors for the new classifier:

\begin{align}
\small
\label{eq:weigh2}
\begin{split}
\text{Pr.}(\tilde{f}(x)\neq y | y = -1) &= \text{Pr.}\{\sum_{k\neq i, k\neq j} w_k \gN(-\eta, \sigma^2) -b + w_j \gN(-\eta, \sigma^2) + w_j \gN(-\eta, \sigma^2) > 0 \}, \\
\text{Pr.}(\tilde{f}(x)\neq y | y = +1)&= \text{Pr.}\{\sum_{k\neq i, k\neq j} w_k \gN(+\eta, \sigma^2) -b + w_j \gN(+\eta, \sigma^2) + w_j \gN(+\eta, \sigma^2) < 0 \}.    
\end{split}
\end{align}
By comparing the errors in Eq.~(\ref{eq:weigh1}) and Eq.~(\ref{eq:weigh2}), it can imply the classifier $\tilde{f}$ has smaller error in each class. Therefore, it contradicts with the assumption that $f$ is the optimal classifier with smallest error. Thus, we conclude for an optimal linear classifier in natural training, it must satisfies $w_1=w_2=\dots=w_d$ (or $w = \bf 1$) if we do not consider the scale of $w$. 

Next, we calculate the optimal bias term $b$ given $w = \textbf{1}$, where we find an optimal $b$ can minimize the (reweighted) empirical risk:
\begin{align*}
\small
\begin{split}
&\text{Error}_\text{train}(f^*)\\
    = & \; \text{Pr.}(f^*(x)\neq y | y = -1)\cdot \text{Pr.}(y=-1) \cdot \rho + \text{Pr.}(f^*(x)\neq y | y = +1)\cdot \text{Pr.}(y=+1) \\
    \propto & \; \text{Pr.}(f^*(x)\neq y | y = -1)\cdot \rho + \text{Pr.}(f^*(x)\neq y | y = +1)\cdot K \\
    = & \; \rho \cdot \text{Pr.}(\sum_{i=1}^{d} \gN(-\eta, \sigma^2)-b>0) + K \cdot\text{Pr.}(\sum_{i=1}^{d} \gN(\eta, \sigma^2)-b<0)\\
    = & \; \rho \cdot \text{Pr.}(\gN(0,1)< -\frac{b + d\eta}{d\sigma}) +  K \cdot \text{Pr.}(\gN(0,1)< \frac{b - d\eta}{d\sigma}),
\end{split}
\end{align*}
and we take the derivative with respect to $b$:
\begin{align*}
\small
    \frac{\partial \text{Error}_\text{train}}{\partial b} = \frac{\rho}{\sqrt{2\pi}} \cdot (-\frac{1}{d\sigma}) \exp(-\frac{1}{2} (-\frac{b+d\eta}{d\sigma})^2)+ \frac{K}{\sqrt{2\pi}} \cdot (\frac{1}{d\sigma}) \exp(-\frac{1}{2} (\frac{b-d\eta}{d\sigma})^2).
\end{align*}
When $\partial \text{Error}_\text{train}/\partial b = 0$, we can calculate the optimal $b$ which gives the minimum value of the empirical error, and we have:
\begin{align*}
\small
    b  = \frac{1}{2} \log (\frac{\rho}{K}) \frac{d\sigma^2}{\eta} = \frac{1}{2} \log (\frac{\rho}{K}) \frac{d}{S}.
\end{align*}
\end{proof}

\subsubsection{Proof of Theorem~\ref{theorem1}}\label{app_sec:theorem1}

\theoryA*

\begin{proof}[Proof of Theorem~\ref{theorem1}]
Without loss of generality, for distribution $\mathcal{D}_1$, $\mathcal{D}_2$ with different mean-variance pairs $(\pm \eta_1, \sigma^2_1)$ and $(\pm \eta_2, \sigma^2_2)$, we can only consider the case $\eta_1 = \eta_2$ and $\sigma_1^2 < \sigma_2^2$. Otherwise, we can simply rescale one of them to match the mean vector of the other and will not impact the results. Under this definition, the optimal classifier $f_1^*$ and $f_2^*$ has weight vector $w_1 = w_2 = \textbf{1}$ and bias term $b_1, b_2$, with the value as demonstrated in Lemma~\ref{lemma1}. Next, we will prove the Theorem~\ref{theorem1} by 2 steps.

\textit{Step 1.} For the error of class ``-1'', we have:
\begin{align*}\small
    \text{Pr.}(f_1^*(x^{(1)})\neq y^{(1)} | y^{(1)} = -1) &= \text{Pr.}(\sum_{i=1}^{d} \gN (-\eta, \sigma_1^2) - b_1>0)\\
    &< \text{Pr.}(\sum_{i=1}^{d} \gN (-\eta, \sigma_1^2) - b_2>0)~~~~(\text{because } S_1 > S_2)\\
    &< \text{Pr.}(\sum_{i=1}^{d} \gN (-\eta, \sigma_2^2) - b_2>0)~~~~(\text{because } \sigma^2_1 < \sigma^2_2)\\
    &=\text{Pr.}(f_2^*(x^{(2)})\neq y^{(2)} | y^{(2)} = -1).
\end{align*}

\textit{Step 2.} For the error of class ``+1'', we have:
\begin{align}
\label{eq:eq11}
\begin{split}
\small
    \text{Pr.}(f_1^*(x^{(1)})\neq y^{(1)} | y^{(1)} = +1) 
    &= \text{Pr.}(\sum_{i=1}^{d} \gN(\eta, \sigma^2_1) - b_1 < 0)\\
    &= \text{Pr.}(\gN(0,1) < \frac{b_1 - d\eta}{d\sigma_1})\\
    & = \text{Pr.}(\gN(0,1)< \frac{-\log(K)\cdot\sigma_1}{2\eta} - \frac{\eta}{\sigma_1}),
\end{split}
\end{align}
and similarly, 
\begin{align}\label{eq:eq12}
\small
    \text{Pr.}(f_2^*(x^{(2)})\neq y^{(2)} | y^{(2)} = +1) 
    = \text{Pr.}(\gN(0,1)< \frac{-\log(K) \cdot \sigma_2}{2\eta} - \frac{\eta}{\sigma_2}).
\end{align}
Note that when $K$ is large enough, i.e., $\log(K)>\frac{2\cdot \eta^2}{\sigma_1\cdot \sigma_2}$, we can get the Z-score in Eq.~(\ref{eq:eq11}) is larger than Eq.~(\ref{eq:eq12}). As a result, we have:
\begin{align}\label{eq:eq13}
\small
     \text{Pr.}(f_1^*(x^{(1)})\neq y^{(1)} | y^{(1)} = +1) > \text{Pr.}(f_2^*(x^{(2)})\neq y^{(2)} | y^{(2)} = +1).
\end{align}
By combining \textit{Step 1} and \textit{Step 2}, we can get the inequality in Theorem~\ref{theorem1}.

\end{proof}

\subsubsection{Proof of Theorem~\ref{theorem2}}\label{app_sec:theorem2}

\theoryB*

\begin{proof}[Proof of Theorem~\ref{theorem2}]
We first show that under both distribution $\mathcal{D}_1$ and $\mathcal{D}_2$, the optimal reweighting ratio $\rho$ is equal to the imbalance ratio $K$. Based on the results in Eq.~(\ref{eq:weigh1}) and calculated model parameters $w$ and $b$, we have the test error (given the model trained by reweight value $\rho$):
\begin{align*}
\small
\begin{split}
&\text{Error}_\text{test}(f^*)\\
    = & \; \text{Pr.}(f^*(x)\neq y | y = -1)\cdot \text{Pr.}(y=-1) + \text{Pr.}(f^*(x)\neq y | y = +1)\cdot \text{Pr.}(y=+1) \\
    \propto & \; \text{Pr.}(\gN(0,1)< -\frac{b + d\eta}{d\sigma}) +  \text{Pr.}(\gN(0,1)< \frac{b - d\eta}{d\sigma})\\
    = & \; \text{Pr.} (\gN(0,1)<-\frac{1}{2}\log(\frac{\rho}{K}) - \frac{\sigma}{\eta}) + \text{Pr.} (\gN(0,1)<\frac{1}{2}\log(\frac{\rho}{K}) - \frac{\sigma}{\eta}).
\end{split}
\end{align*}
The value of taking the minimum when its derivative with respect to $\rho$ is equal to $0$, where we can get $\rho = K$ and the bias term $b = 0$.
Note that the variance values have the relation: $\sigma_1^2 <\sigma_2^2$. Therefore,  it is easy to get that:
\vspace{-0.2cm}
\begin{align}\label{eq:eq17}
\small
\begin{split}
    \text{Pr.}({f'_1}^*(x^{(1)})\neq y^{(1)} | y^{(1)} = +1) &= \text{Pr.}(\sum_{i=1}^{d} \gN (\eta, \sigma_1^2)<0)  \\ <  \text{Pr.}(\sum_{i=1}^{d} \gN (\eta, \sigma_2^2)<0) & =   \text{Pr.}({f'_2}^*(x^{(2)})\neq y^{(2)} | y^{(2)} = +1).
\end{split}
\end{align}
\vspace{-0.2cm}
Combining the results in Eq.~(\ref{eq:eq13}) and~(\ref{eq:eq17}), we have proved the inequality in Theorem~\ref{theorem2}.

\end{proof}

\subsection{Algorithm of SRAT}\label{app_sec:algorithm}

The algorithm of our proposed SRAT framework is shown in Algorithm~\ref{alg:srat}. Specifically, in each training iteration, we first generate adversarial examples using PGD for examples in the current batch (Line 5). If the current training iteration does not reach a predefined starting reweighting epoch $T_d$, we will assign same weights, i.e., $w_i = 1$ for all adversarial examples $\mathbf{x}_i$ in the current batch (Line 6). Otherwise, the reweighting strategy will be adopted in the final loss function (Line 15), where a specific weight $w_i$ will be assigned for each adversarial example $\mathbf{x}_i$ if its corresponding clean example $\mathbf{x}_i$ comes from an under-represented class.

\begin{algorithm}[h]
\caption{Separable Reweighted Adversarial Training (SRAT).} 
\label{alg:srat}
\begin{algorithmic}[1]
\REQUIRE imbalanced training dataset $D = \{(\mathbf{x}_i, y_i)\}_{i=1}^{n}$, number of total training epochs $T$, starting reweighting epoch $T_d$, batch size $N$, number of batches $M$, learning rate $\gamma$
\ENSURE An adversarially robust model $f_\theta$
\STATE Initialize the model parameters $\theta$ randomly;
\FOR{epoch $=1, \dots, T_d - 1$} 
\FOR{mini-batch $=1, \dots, M$}
\STATE Sample a mini-batch $\mathcal{B}=\{(\mathbf{x}_i, y_i) \}_{i=1}^{N}$ from $D$;
\STATE Generate adversarial example $\mathbf{x}'_i$ for each $\mathbf{x}'_i \in \mathcal{B}$;
\STATE $\mathcal{L}(f_\theta) = \frac{1}{N} \sum_{i = 1}^N \max_{\Vert \mathbf{x}_i' - \mathbf{x}_i \Vert_p \leq \epsilon} \mathcal{L}(f_{\theta}(\mathbf{x}'_i), y_i) + \lambda \mathcal{L}_{sep}(\mathbf{x}'_i)$
\STATE $\theta \gets \theta - \gamma \nabla_{\mathbf{\theta}} \mathcal{L}(f_\theta)$
\ENDFOR
\STATE Optional: $\gamma \gets \gamma/\kappa$
\ENDFOR
\FOR{epoch $=T_d, \dots, T$}
\FOR{mini-batch $=1, \dots, M$}
\STATE Sample a mini-batch $\mathcal{B}=\{(\mathbf{x}_i, y_i) \}_{i=1}^{N}$ from $D$;
\STATE Generate adversarial example $\mathbf{x}'_i$ for each $\mathbf{x}'_i \in \mathcal{B}$;
\STATE $\mathcal{L}(f_\theta) = \frac{1}{N} \sum_{i = 1}^N \max_{\Vert \mathbf{x}_i' - \mathbf{x}_i \Vert_p \leq \epsilon} w_i \mathcal{L}(f_{\theta}(\mathbf{x}'_i), y_i) + \lambda \mathcal{L}_{sep}(\mathbf{x}'_i)$
\STATE $\theta \gets \theta - \gamma \nabla_{\mathbf{\theta}} \mathcal{L}(f_\theta)$
\ENDFOR
\STATE Optional: $\gamma \gets \gamma/\kappa$
\ENDFOR
\end{algorithmic}
\end{algorithm}

\subsection{Data Distribution of Imbalanced Training Datasets}\label{app_sec:dataset}

In our experiments, we construct multiple imbalanced training datasets to simulate various kinds of imbalanced scenarios by combining different imbalance types (i.e., Exp and Step) with different imbalanced ratios (i.e., $K=10$ and $K=100$). Figure~\ref{fig:data_distribution_cifar10} and Figure~\ref{fig:data_distribution_svhn} show the data distribution of all ten-classes imbalanced training datasets used in our preliminary studies and experiments based on CIFAR10~\cite{krizhevsky2009learning} and SVHN~\cite{netzer2011reading} datasets, respectively.

\begin{figure}[h]
\centering
\begin{subfigure}[b]{0.235\textwidth}
\centering
\includegraphics[width=1.3in]{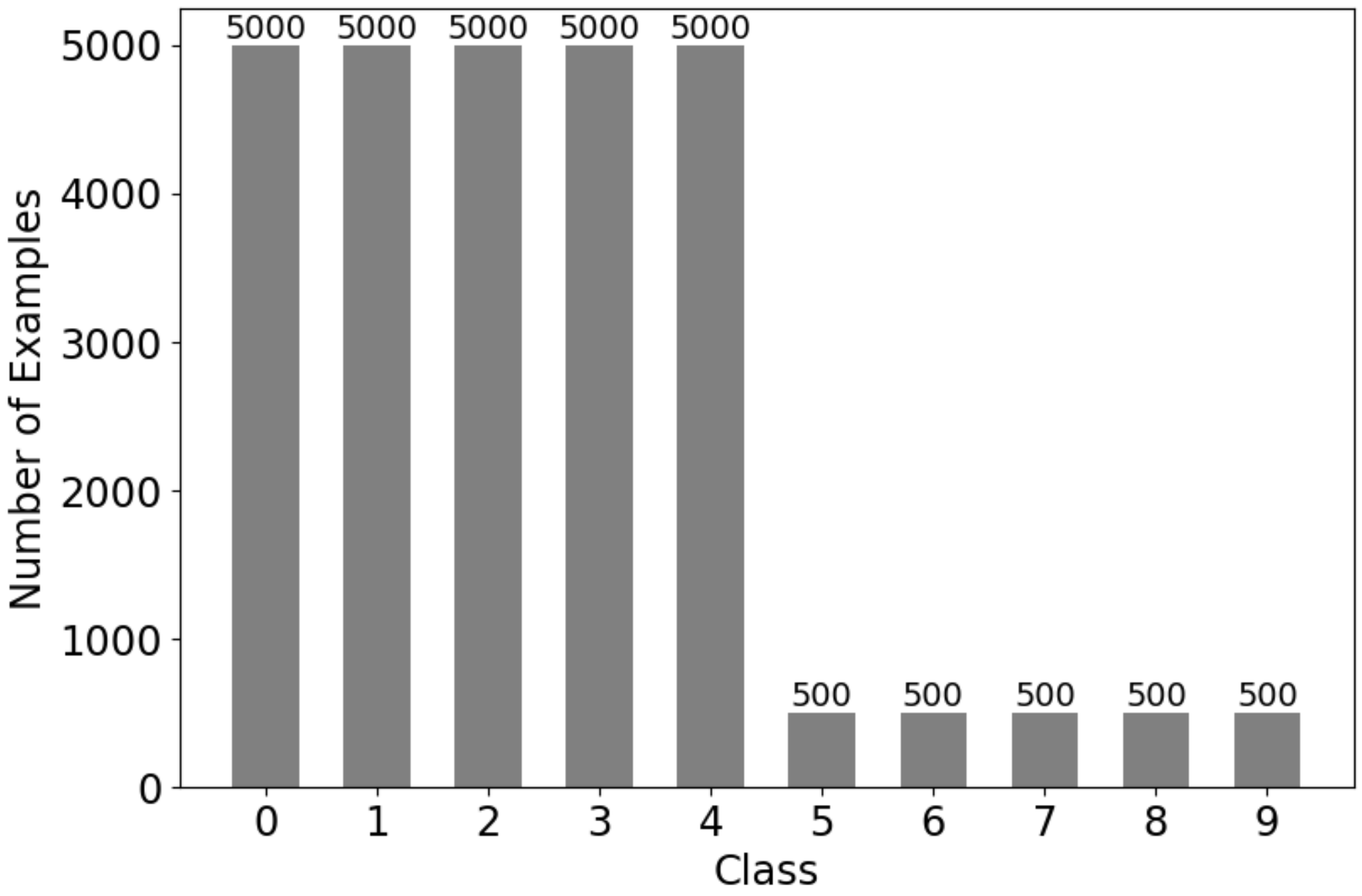}
\caption{Step-10}
\label{fig:data_distribution_cifar10_step_10}
\end{subfigure}
\hspace{0.03in}
\begin{subfigure}[b]{0.235\textwidth}
\centering
\includegraphics[width=1.3in]{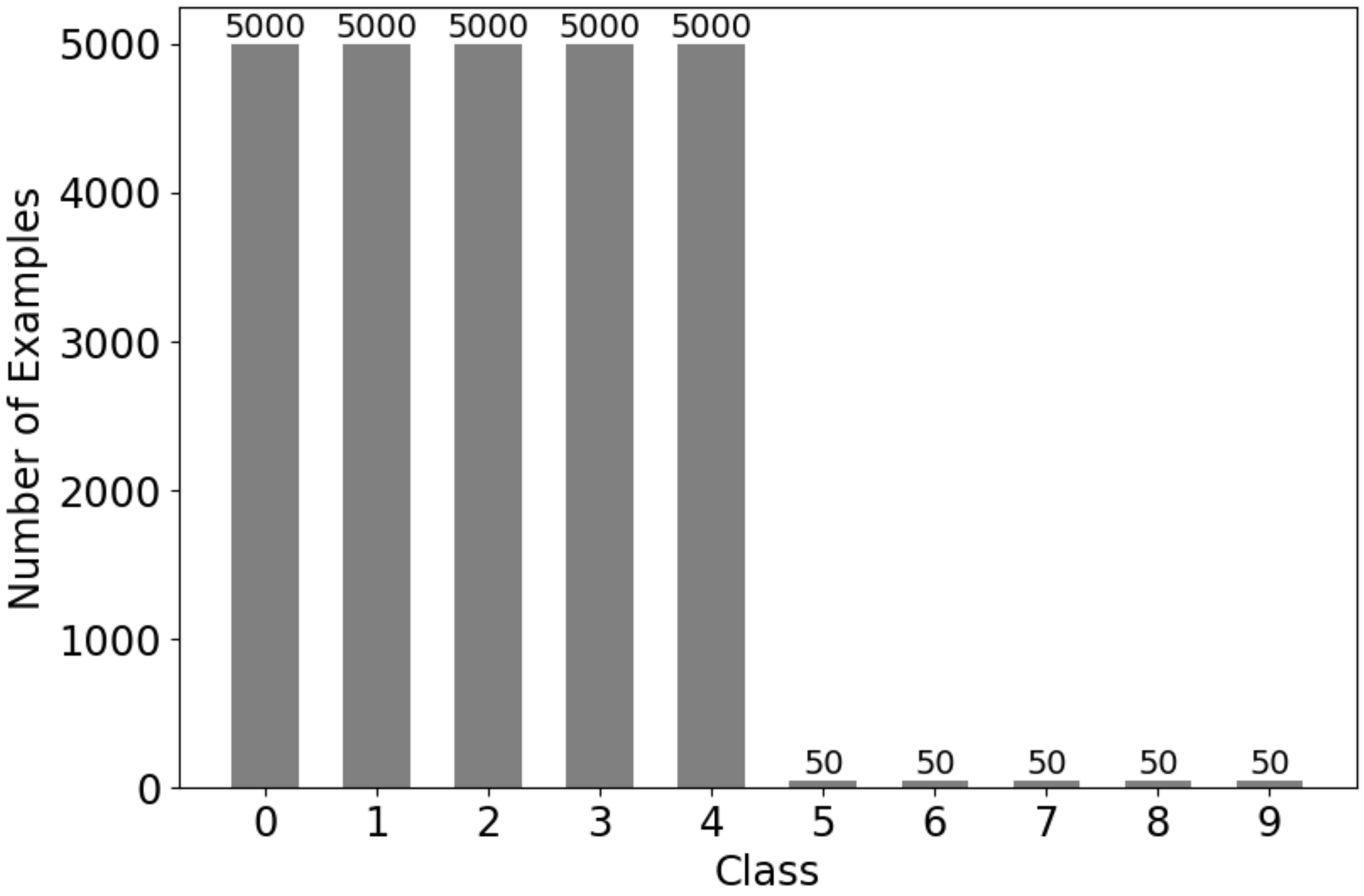}
\caption{Step-100}
\label{fig:data_distribution_cifar10_step_100}
\end{subfigure}
\hspace{0.03in}
\begin{subfigure}[b]{0.235\textwidth}
\centering
\includegraphics[width=1.3in]{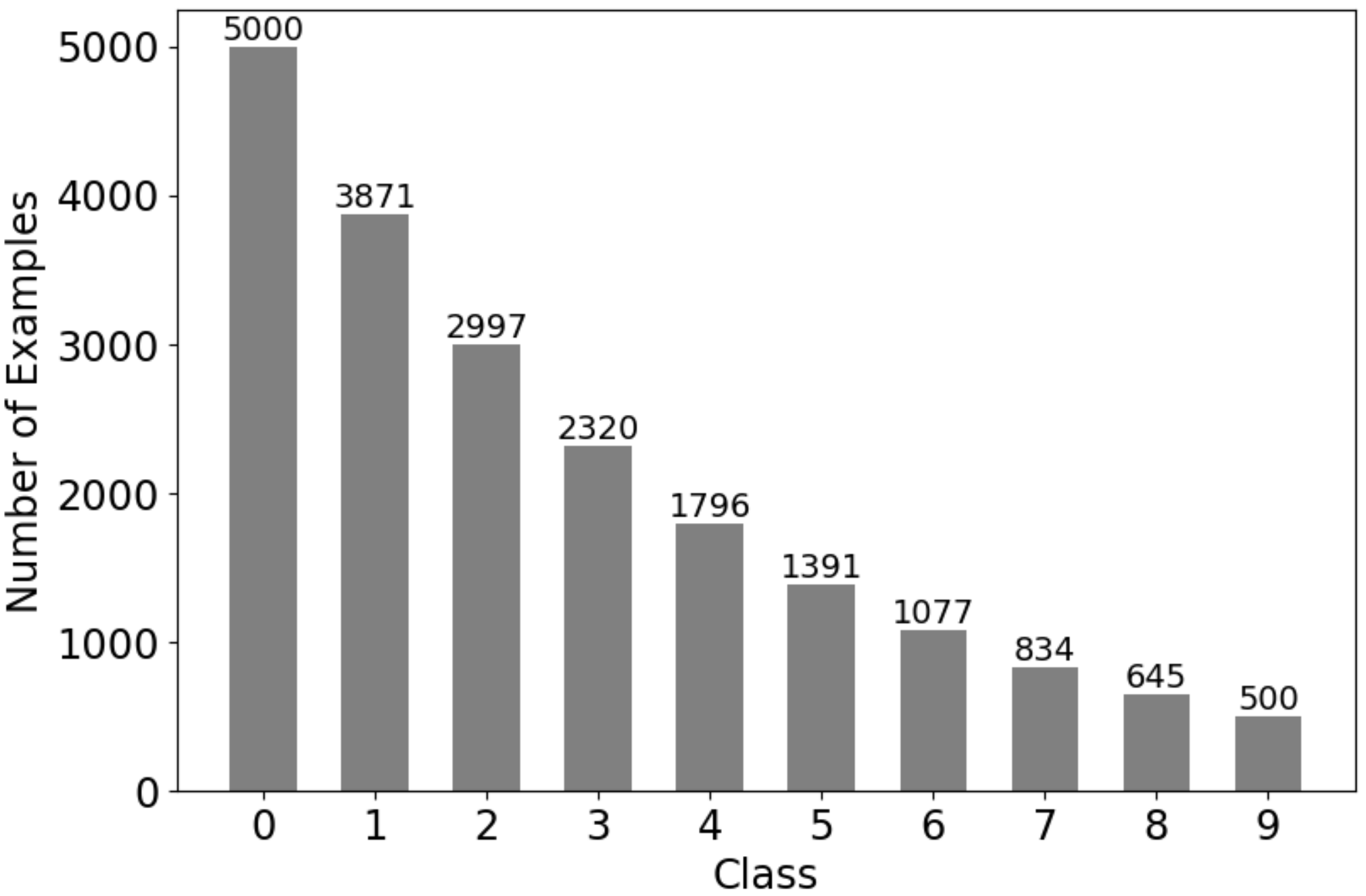}
\caption{Exp-10}
\label{fig:data_distribution_cifar10_exp_10}
\end{subfigure}
\hspace{0.03in}
\begin{subfigure}[b]{0.235\textwidth}
\centering
\includegraphics[width=1.3in]{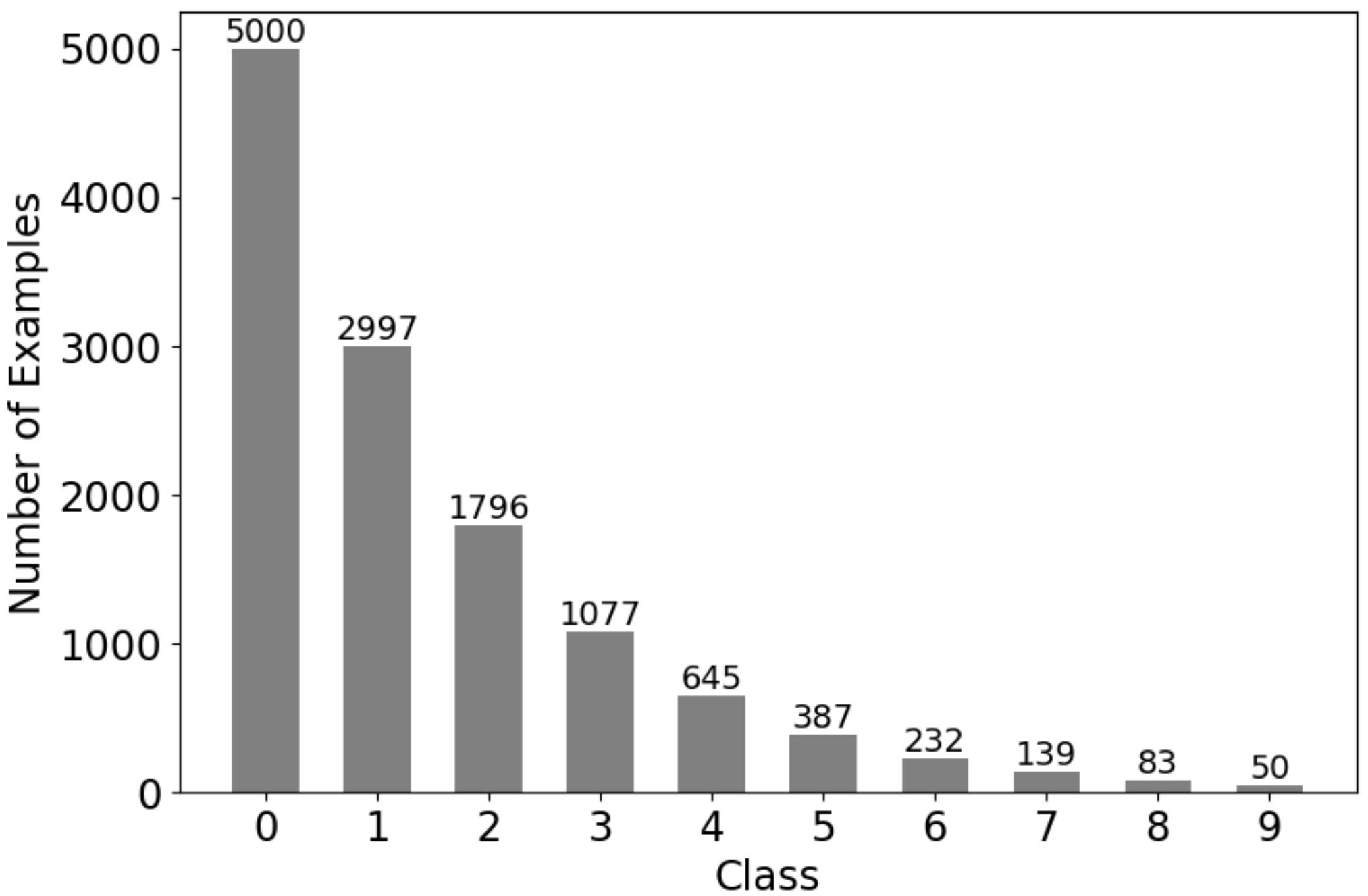}
\caption{Exp-100}
\label{fig:data_distribution_cifar10_exp_100}
\end{subfigure}
\caption{Data distribution of imbalanced training datasets constructed from CIFAR10 dataset.}
\label{fig:data_distribution_cifar10}
\end{figure}

\begin{figure}[h]
\centering
\begin{subfigure}[b]{0.235\textwidth}
\centering
\includegraphics[width=1.3in]{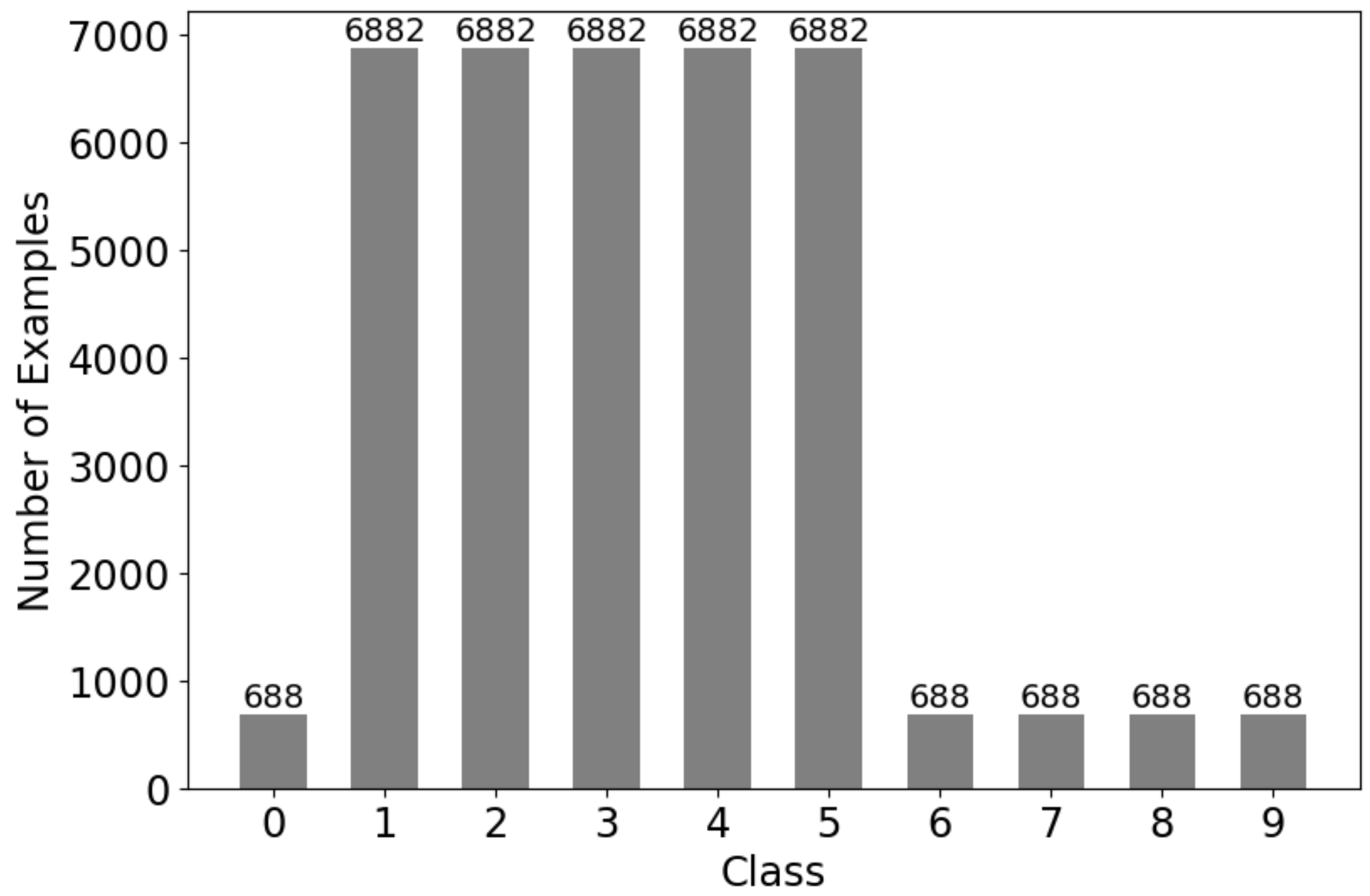}
\caption{Step-10}
\label{fig:data_distribution_svhn_step_10}
\end{subfigure}
\hspace{0.03in}
\begin{subfigure}[b]{0.235\textwidth}
\centering
\includegraphics[width=1.3in]{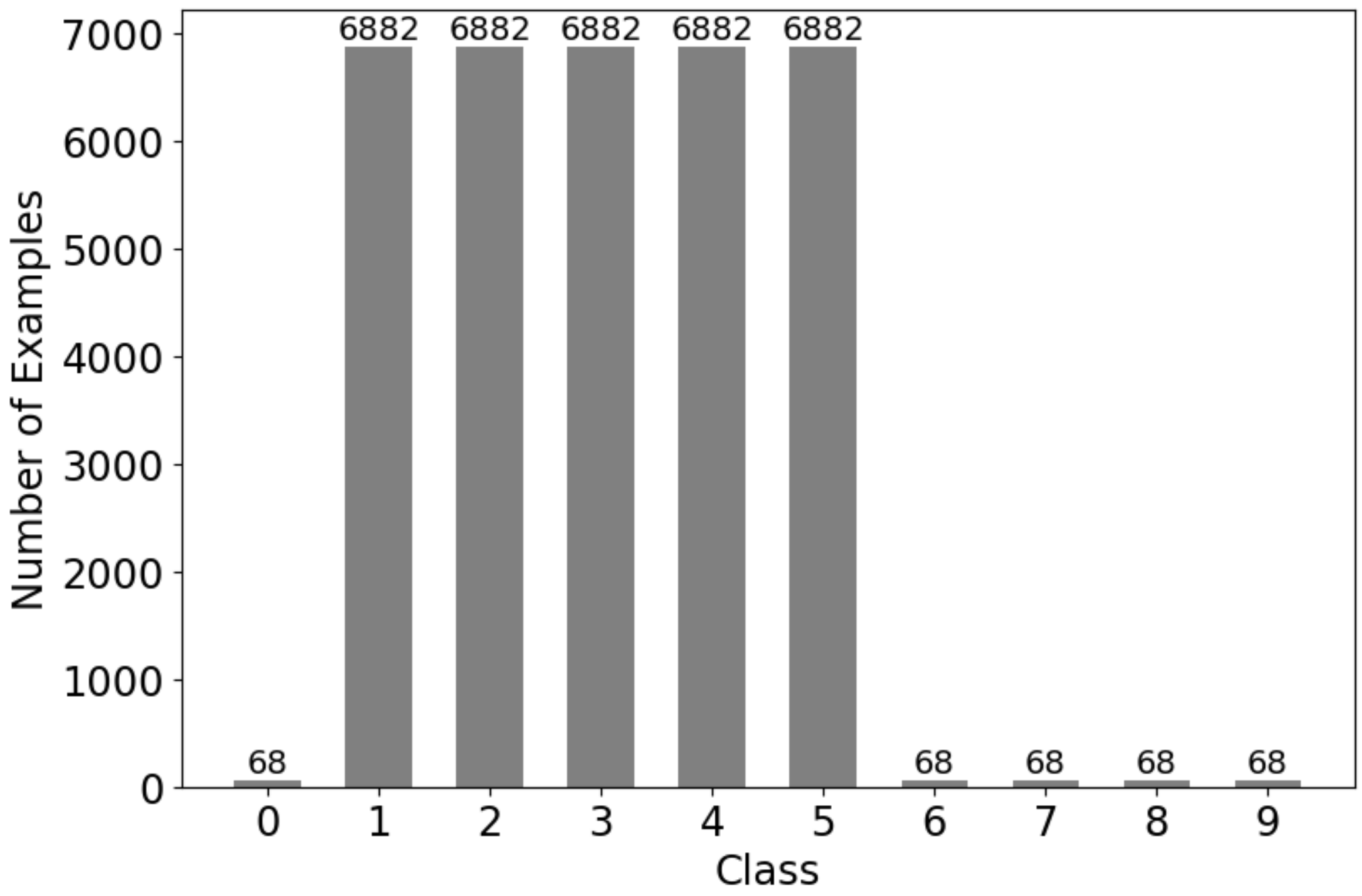}
\caption{Step-100}
\label{fig:data_distribution_svhn_step_100}
\end{subfigure}
\hspace{0.03in}
\begin{subfigure}[b]{0.235\textwidth}
\centering
\includegraphics[width=1.3in]{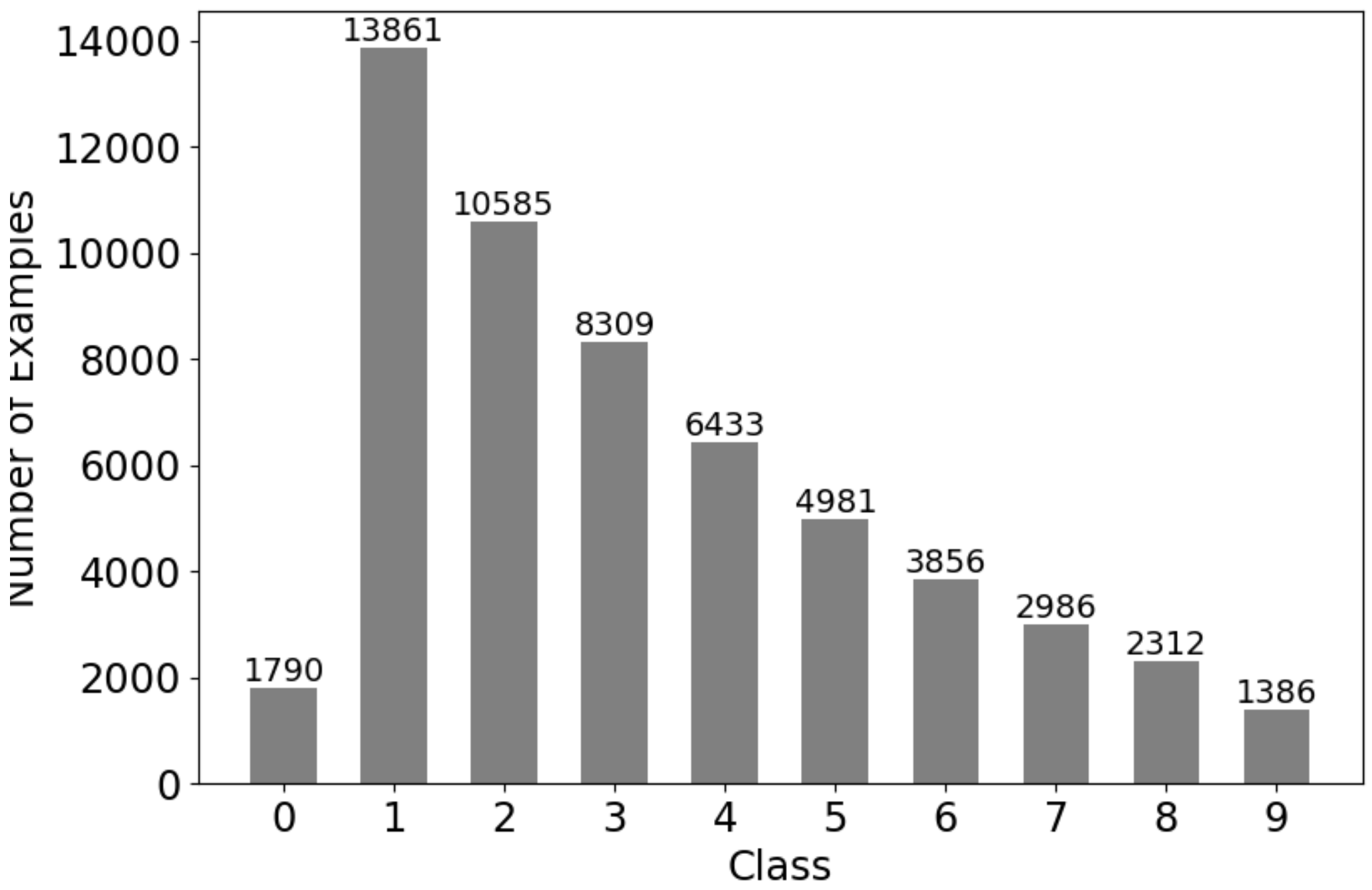}
\caption{Exp-10}
\label{fig:data_distribution_svhn_exp_10}
\end{subfigure}
\hspace{0.03in}
\begin{subfigure}[b]{0.235\textwidth}
\centering
\includegraphics[width=1.3in]{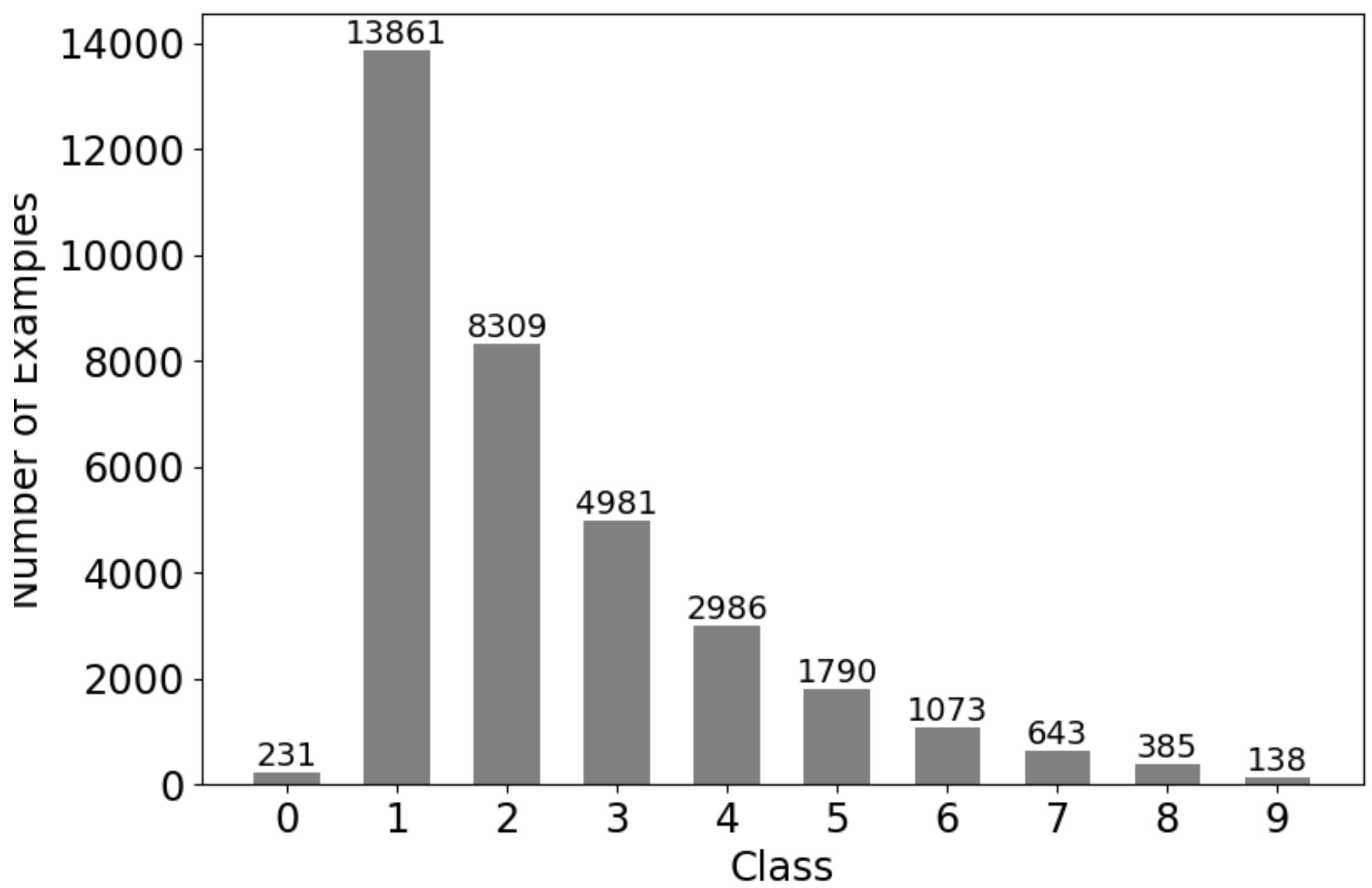}
\caption{Exp-100}
\label{fig:data_distribution_svhn_exp_100}
\end{subfigure}
\caption{Data distribution of imbalanced training datasets constructed from SVHN dataset.}
\label{fig:data_distribution_svhn}
\end{figure}

\subsection{Performance Comparison on Imbalanced SVHN Datasets}\label{app_sec:svhn_results}

Table~\ref{tab:step_main_result_svhn} and Table~\ref{tab:longtail_main_result_svhn} show the performance comparison on various imbalanced SVHN datasets with different imbalance types and imbalance ratios. We use bold values to denote the highest accuracy among all methods and use the underline values to indicate our SRAT variants which achieve the highest accuracy among their corresponding baseline methods utilizing the same loss function for making predictions.

From Table~\ref{tab:step_main_result_svhn} and Table~\ref{tab:longtail_main_result_svhn}, we get similar observation that, comparing with baseline methods, our proposed SRAT method can produce a robust model which can achieve improved overall performance when the training dataset is imbalanced. In addition, based on the experimental results in Table~\ref{tab:step_main_result} to Table~\ref{tab:longtail_main_result_svhn}, we find that, compared with the performance improvement between DRCB-LDAM and SRAT-LDAM, the improvement between DRCB-CE and SRAT-CE and the improvement between DRCB-Focal and SRAT-Focal are more obviously. The possible reason behind this phenomenon is, the LDAM loss can also implicitly produce a more separable feature space~\cite{cao2019learning} while CE loss and Focal loss do not conduct any specific operations on the latent feature space. Hence, the feature separation loss contained in SRAT-CE and SRAT-Focal could be more effective on learning separable feature space and facilitate the Focal loss on prediction. However, in SRAT-LDAM, the feature separation loss and LDAM loss may affect each other on learning feature representations and, hence, the effectiveness of the feature separation loss may be counteracted or weakened.

In conclusion, experiments conducted on multiple imbalanced datasets verify the effectiveness of our proposed SRAT method under various imbalanced scenarios.

\begin{table}[h]
\small
\centering
\caption{Performance Comparison on Imbalanced SVHN Datasets (Imbalanced Type: Step)}
\label{tab:step_main_result_svhn}
\begin{tabular}{c|c|c|c|c|c|c|c|c}
\hline
Imbalance Ratio & \multicolumn{4}{c}{10} & \multicolumn{4}{|c}{100} \\
\hline
Imbalance Ratio & \multicolumn{2}{c}{Standard Accuracy} & \multicolumn{2}{|c}{Robust Accuracy} & \multicolumn{2}{|c}{Standard Accuracy} & \multicolumn{2}{|c}{Robust Accuracy} \\
\hline
Method & Overall & Under & Overall & Under & Overall & Under & Overall & Under \\
\hline
CE & 79.88 & 67.04 & 37.62 & 22.08 & 59.61 & 26.19 & 29.57 & 5.03 \\
Focal & 79.96 & 67.03 & 37.83 & 22.47 & 60.58 & 28.17 & 30.27 & 5.83 \\
LDAM & 84.55 & 74.96 & 45.80 & 31.23 & 65.61 & 37.13 & 33.34 & 8.36 \\
CB-Reweight & 79.48 & 66.07 & 37.38 & 21.66 & 60.23 & 27.68 & 29.54 & 5.32 \\
CB-Focal & 80.29 & 67.56 & 38.10 & 23.00 & 60.73 & 28.37 & 30.09 & 5.75 \\
DRCB-CE & 80.62 & 68.74 & 37.25 & 22.79 & 60.67 & 28.36 & 30.02 & 5.59 \\
DRCB-Focal & 79.11 & 65.72 & 37.01 & 22.02 & 61.65 & 30.29 & 30.78 & 7.06 \\
DRCB-LDAM & \textbf{87.83} & \textbf{82.63} & \textbf{46.45} & \textbf{35.15} & 63.78 & 33.99 & 33.60 & 7.28 \\
\hline
\hline
SRAT-CE & \underline{82.89} & \underline{72.79} & \underline{38.23} & \underline{24.70} & \underline{63.39} & \underline{33.85} & 29.64 & \underline{6.11} \\
SRAT-Focal & \underline{85.05} & \underline{77.10} & \underline{39.51} & \underline{28.06} & \underline{70.12} & \underline{47.44} & \underline{32.18} & \underline{11.08} \\
SRAT-LDAM & 87.65 & 82.62 & 46.03 & 34.75 & \underline{\textbf{71.56}} & \underline{\textbf{50.33}} & \underline{\textbf{33.54}} & \underline{\textbf{11.63}} \\
\hline
\end{tabular}
\end{table}

\begin{table}[h]
\small
\centering
\caption{Performance Comparison on Imbalanced SVHN Datasets (Imbalanced Type: Exp)}
\label{tab:longtail_main_result_svhn}
\begin{tabular}{c|c|c|c|c|c|c|c|c}
\hline
Imbalance Ratio & \multicolumn{4}{c}{10} & \multicolumn{4}{|c}{100} \\
\hline
Metric & \multicolumn{2}{c}{Standard Accuracy} & \multicolumn{2}{|c}{Robust Accuracy} & \multicolumn{2}{|c}{Standard Accuracy} & \multicolumn{2}{|c}{Robust Accuracy} \\
\hline
Method & Overall & Under & Overall & Under & Overall & Under & Overall & Under \\
\hline
CE & 87.54 & 82.67 & 44.12 & 35.33 & 72.51 & 56.30 & 33.34 & 16.93 \\
Focal & 87.82 & 83.01 & 44.88 & 35.97 & 72.61 & 56.48 & 34.09 & 17.62 \\
LDAM & 90.06 & 86.69 & 51.84 & 43.73 & 79.11 & 66.86 & \textbf{40.42} & \textbf{25.18} \\
CB-Reweight & 87.66 & 82.79 & 44.39 & 35.53 & 72.25 & 55.97 & 33.36 & 17.16 \\
CB-Focal & 87.86 & 82.96 & 44.61 & 35.55 & 73.23 & 57.34 & 34.25 & 17.90 \\
DRCB-CE & 88.49 & 84.51 & 43.82 & 36.28 & 73.74 & 58.03 & 33.52 & 17.68 \\
DRCB-Focal & 87.47 & 82.78 & 42.52 & 34.31 & 71.95 & 55.11 & 33.43 & 17.63 \\
DRCB-LDAM & 91.24 & \textbf{89.65} & \textbf{52.39} & \textbf{46.71} & 80.29 & 69.23 & 40.16 & 24.64 \\
\hline
\hline
SRAT-CE & \underline{88.70} & \underline{84.94} & \underline{44.54} & \underline{36.59} & \underline{77.11} & \underline{64.47} & \underline{34.48} & \underline{19.91} \\
SRAT-Focal & \underline{89.51} & \underline{85.42} & \underline{45.37} & \underline{37.20} & \underline{80.04} & \underline{69.54} & \underline{35.25} & \underline{23.04} \\
SRAT-LDAM & \textbf{91.27} & 89.55 & 52.10 & 46.13 & \underline{\textbf{80.71}} & \underline{\textbf{70.49}} & 40.33 & 25.11 \\
\hline
\end{tabular}
\end{table}

\end{document}